\documentclass{article}
\usepackage{microtype}
\usepackage{graphicx}
\usepackage{subcaption}
\usepackage{booktabs}
\usepackage{hyperref}

\usepackage[preprint]{icml2026}
\usepackage[utf8]{inputenc} 
\usepackage[T1]{fontenc}    
\usepackage{url}            
\usepackage{booktabs}       
\usepackage{amsfonts}       
\usepackage{nicefrac}       
\usepackage{microtype}      
\usepackage{soul}
\usepackage{graphicx}
\usepackage{amsmath}
\usepackage{outlines}
\usepackage{xurl}

\usepackage{amsmath,amsfonts,bm}

\def\eqref#1{equation~\ref{#1}}

\def\1{\bm{1}}

\DeclareMathAlphabet{\mathsfit}{\encodingdefault}{\sfdefault}{m}{sl}
\SetMathAlphabet{\mathsfit}{bold}{\encodingdefault}{\sfdefault}{bx}{n}

\usepackage{multirow}
\usepackage{paralist}

\usepackage{multicol}
\usepackage{diagbox}

\usepackage{here}

\usepackage{amsmath,amssymb,amsfonts,amsbsy,amsfonts,latexsym}
\usepackage{makecell}
\usepackage{xcolor}
\usepackage{colortbl}

\usepackage{tabularx,colortbl,xcolor}
\usepackage[normalem]{ulem}
\useunder{\uline}{\ul}{}

\usepackage{enumitem}

\usepackage{xparse}

\usepackage{longtable}

\NewDocumentCommand{\var}{O{s} m O{}}{%
  \ensuremath{#1_{#2}^{#3}}
}
\usepackage{siunitx}

\newcommand{\commentout}[1]{}

\definecolor{light-gray}{gray}{0.80}

\usepackage{amsthm}

\usepackage{amsthm}

\newtheorem{mytheorem}{Theorem}[section]
\newtheorem{theorem}[mytheorem]{Theorem}
\newtheorem{corollary}{Corollary}[section]
\newtheorem{assumption}{Assumption}
\newtheorem{lemma}{Lemma}[section]
\theoremstyle{definition}

\newtheorem{proposition}{Proposition}[section]

\theoremstyle{plain}
\newtheorem{remk}{Remark}[section]

\newcommand{\bigO}{\mathcal{O}}
\definecolor{myblue}{rgb}{0,0.2,0.8}
\usepackage{xcolor}
\usepackage{listings}

\definecolor{upforestgreen}{rgb}{0.6, 0.8, 0.2}

\usepackage{chngcntr}
\usepackage{adjustbox}
\usepackage{wrapfig}
\usepackage{arydshln}
\usepackage{tcolorbox}
\usepackage{amsthm,amsmath,amssymb}
\usepackage{amsfonts,dsfont}
\usepackage{mathtools}
\usepackage{float}

\definecolor{mygreen}{HTML}{009901}
\definecolor{myred}{HTML}{A52A2A}

\icmltitlerunning{Scaling Laws for Linear RNNs}

\begin{document}

\twocolumn[
  \icmltitle{Learning from Transients: Linear RNN Scaling Laws in Stationary Prediction}

  \begin{icmlauthorlist}
    \icmlauthor{Ziyan Chen}{sydney}
    \icmlauthor{Linhe Zheng}{unsw}
    \icmlauthor{Peilin Liu}{psu}
    \icmlauthor{Zhongzhu Zhou}{together,sydney}
    \icmlauthor{Dingxuan Zhou}{sydney}
  \end{icmlauthorlist}

  \icmlaffiliation{sydney}{The University of Sydney, Sydney, Australia}
  \icmlaffiliation{unsw}{The University of New South Wales, Sydney, Australia}
  \icmlaffiliation{psu}{The Pennsylvania State University, State College, PA, USA}
  \icmlaffiliation{together}{Together AI}

  \icmlcorrespondingauthor{Ziyan Chen}{ziyan.chen@sydney.edu.au}
  \icmlkeywords{Linear recurrent neural networks, scaling laws, nonstationary training}

  \vskip 0.3in
]
\printAffiliationsAndNotice{}

\begin{abstract}
Does the organization of a fixed number of transitions into independent trajectories affect a fixed stationary prediction task? We study a Gaussian-sketched linear recurrent teacher--student model, where independently initialized training trajectories create transients that can excite weak spectral directions. For full-batch gradient descent with stable deterministic learning rates, we establish matching stationary approximation, bias, variance, and total-risk scaling laws from dependent trajectory data. When initialization decays at least as fast spectrally as the innovations, trajectory count and length enter the leading variance rate through their product. With a heavier initialization tail, transient excitation can improve the trajectory-length-dependent finite-time bias and, at fixed transition count, variance, while the stationary oracle shift remains controlled by representation approximation. Under token and compute caps, constant-length trajectories attain the optimal risk scaling order within the theorem-controlled range. Synthetic experiments provide finite-scale evidence in this stylized model.
\end{abstract}


\section{Introduction}
\label{sec:introduction}

Suppose a fixed training token budget can be organized
into either a few long trajectories or many short ones. 
How does this choice affect the quality of learning the dynamics between tokens?

Scaling laws provide the answer, which evaluates how model quality depends on model size, data, optimization steps 
and compute \citep{hestness2017deep,kaplan2020scaling,
henighan2020scaling,rae2021scaling,hoffmann2022training,rosenfeld2020constructive}.
Their theoretical origins have been studied in regression, kernel,
random-feature, and data-reuse models \citep{sharma2022manifold,
bahri2021explaining,bordelon2024dynamical,lin2024scaling,lin2025reuse}.
Linear recurrent work studies learning dynamics and representations
\citep{proca2025learning}, empirical scaling in linear-complexity language
models \citep{shen2024linearcomplexity}, and autoregressive identification
\citep{yuksel2025longcontext}. Yet the decomposition of stationary prediction
error into representation, optimization, and sampling contributions under
independently initialized training trajectories remains less understood.

We study a Gaussian-sketched linear recurrent teacher--student model. The
student of representation dimension \(M\) is trained on complete dependent
trajectories by safeguarded full-batch gradient descent, with stable
deterministic learning rates. Evaluation uses the same stationary
objective. To the best of our knowledge, our analysis is the first scaling law 
for dynamic learning in linear RNNs to jointly track how model width, trajectory count \& length, optimization time,
influences the stationary prediction risk. Our main contributions are as follows.

\begin{enumerate}[leftmargin=*]
    \item \textbf{Trajectory organization matters.}
    Beyond the total number of transitions, trajectory length can affect
    learning. When initial-state effects dominate, shorter trajectories improve
    the Bias scaling and, at a fixed transition count, the Variance scaling.
    This does not imply that shorter trajectories always reduce total risk.
    \item \textbf{Learning from transients.}
    Initial states can help the model learn weak directions. The resulting
    difference between the training and evaluation optima contributes no more
    than the representation-error order.
    \item \textbf{Resource allocation.}
    Under our token and compute constraints, constant-length trajectories attain
    the optimal risk scaling within the theorem's range. Synthetic experiments
    provide finite-scale evidence for our findings.
\end{enumerate}

These findings are supported by an exact stationary risk decomposition,
matching componentwise rates for empirical GD, and an optimization-scale
regularized covariance analysis common to both regimes. Noncommutative filter
transfer covers general stable deterministic learning rates, and a
population-filter surrogate controls cross terms to obtain matching total-risk
bounds without Bias dominance. The cumulative learning-rate mass governs the
rates up to constants, rather than requiring a particular WSD schedule.

\paragraph{Notation.}
We collect the important notations used across the paper. For two positive-valued functions \(f\) and \(g\), we write \(f\lesssim g\) (\(f=\bigO(g)\)) and \(f\gtrsim g\) when the inequality holds up to an absolute constant, and \(f\asymp g\) ((\(f=\Theta(g)\))) when both comparisons hold. For vectors, \(\langle u,v\rangle=u^\top v\). For matrices, \(\|\cdot\|_F\) and \(\|\cdot\|_{\mathrm{op}}\) denote the Frobenius and operator norms, \(\operatorname{Tr}(\cdot)\) denotes the trace, and \(\preceq\) denotes the Loewner order on symmetric matrices. For any positive semidefinite matrix \(G\), we use the \(G\)-weighted matrix inner product \(\langle D,E\rangle_G:=\operatorname{Tr}(DGE^\top)\) and norm \(\|D\|_G^2:=\langle D,D\rangle_G=\|DG^{1/2}\|_F^2\).
For a symmetric matrix \(A\), \(\mu_j(A)\) denotes its \(j\)-th eigenvalue in nonincreasing order whenever the ordering matters.


\section{Preliminary}
\label{sec:preliminaries}

\subsection{Problem Setup}

\paragraph{Latent linear RNN teacher.}
For each independent trajectory \(i=1,\dots,N\), the latent state
\(x_{i,p}\in\mathbb R^d\) follows the linear Gaussian autoregressive model
\[
    x_{i,p+1}=A_*x_{i,p}+\xi_{i,p+1},
    \qquad p=0,\dots,P-1.
\]
Here \(p\) indexes sequence position, and \(P\ge2\) is the trajectory
length. The transition matrix \(A_*\) describes
dynmaics, while \(\xi_{i,p}\sim\mathcal N(0,\Sigma_\xi)\)
and \(x_{i,0}\sim\mathcal N(0,\Sigma_0)\) are independent innovations and
initial states, respectively.
\paragraph{Sketched observations.}
The learner observes only the sketched states in \(M\) dimensional space with \(M\ll d\):
\[
    u_{i,p}:=Sx_{i,p}\in\mathbb R^M,
    \qquad
    y_{i,p}:=u_{i,p+1},
\]
where \(S\in\mathbb R^{M\times d}\) is a Gaussian sketch with independent entries \(S_{\ell j}\sim \mathcal N(0,1/M)\). Throughout the paper, the ambient and sketch dimensions satisfy \(d\ge C_{\mathrm{amb}}M,\) where \(C_{\mathrm{amb}}>1\) is a sufficiently large fixed constant
independent of \(M,N,P,L\). 
The Gaussian sketch is a tractable representation bottleneck motivated by
constraints in student's
representation width. Because the trajectories are identically distributed,
we define covariance at index \(p\) and its average
\[
    \Sigma_p:=\mathbb E[x_{i,p}x_{i,p}^\top],
    \qquad
    H:=\frac1P\sum_{p=0}^{P-1}\Sigma_p,
\]
and we denote the sketched covariance and cross-covariance
\[
    G_p:=\mathbb E[u_{i,p}u_{i,p}^\top]=S\Sigma_pS^\top,
    \qquad
    C_p:=\mathbb E[y_{i,p}u_{i,p}^\top].
\]


\paragraph{Population and empirical regression.}
The observable population and empirical objectives are
\[
\begin{aligned}
    \mathcal L_M(B)
    &:=
    \frac1{2P}
    \sum_{p=0}^{P-1}
    \mathbb E\|y_{i,p}-Bu_{i,p}\|_2^2,\\
    \widehat{\mathcal L}_M(B)
    &:=
    \frac1{2NP}
    \sum_{i=1}^N\sum_{p=0}^{P-1}
    \|y_{i,p}-Bu_{i,p}\|_2^2,
\end{aligned}
\]
where $B\in\mathbb R^{M\times M}$ is the trainable parameter.

Define the population covariance and cross-covariance as
\[
    G:=SHS^\top
    =
    \frac1P\sum_{p=0}^{P-1}G_p,
    \qquad
    C:=\frac1P\sum_{p=0}^{P-1}C_p.
\]
with their empirical counterparts
\[
    \widehat G
    :=
    \frac1{NP}
    \sum_{i=1}^N\sum_{p=0}^{P-1}
    u_{i,p}u_{i,p}^\top,
    \qquad
    \widehat C
    :=
    \frac1{NP}
    \sum_{i=1}^N\sum_{p=0}^{P-1}
    y_{i,p}u_{i,p}^\top.
\]
The the objectives and the population minimizer have
\begin{align*}
    \nabla\mathcal L_M(B)=BG-C,
    \qquad
    \nabla\widehat{\mathcal L}_M(B)=B\widehat G-\widehat C,\\
    B^*:=CG^{-1}
    =SA_*HS^\top(SHS^\top)^{-1}.
\end{align*}
Thus \(B^*\) is the minimizer of the training population objective, and its quadratic excess
risk has the exact form
\[
    \mathcal L_M(B)-\mathcal L_M(B^*)
    =
    \frac12\|B-B^*\|_G^2.
\]

\paragraph{Stationary evaluation distribution.}
Since \(\|A_*\|_{\mathrm{op}}\le\rho<1\), the latent autoregressive process
admits a unique stationary Gaussian distribution with covariance
\[
    \Sigma_\infty
    :=
    \sum_{q=0}^{\infty}
    A_*^q\Sigma_\xi(A_*^\top)^q.
\]
Under Assumptions~\ref{assump:stable-rnn}--\ref{assump:power-law-source},
\[
    \Sigma_\infty
    =
    \sum_{j=1}^d
    \frac{\sigma_{\xi,j}}{1-a_j^2}e_je_j^\top,
    \qquad
    \sigma_{\infty,j}:=\frac{\sigma_{\xi,j}}{1-a_j^2}
    \asymp j^{-\alpha}.
\]
For evaluation, draw a stationary state
\(x_\infty\sim\mathcal N(0,\Sigma_\infty)\) and an independent innovation
\(\xi\sim\mathcal N(0,\Sigma_\xi)\), and set
\[
    u_\infty:=Sx_\infty,
    \qquad
    y_\infty:=S(A_*x_\infty+\xi).
\]
The stationary one-step prediction objective is
\[
    \boxed{
    \mathcal L_{M,\infty}(B)
    :=
    \frac12\mathbb E\|y_\infty-Bu_\infty\|_2^2.
    }
\]
Equivalently,
\[
    \mathcal L_{M,\infty}(B)
    =
    \frac12\operatorname{Tr}(S\Sigma_\xi S^\top)
    +
    \frac12\|(BS-SA_*)\Sigma_\infty^{1/2}\|_F^2.
\]
Define the stationary covariance and cross-covariance
\[
    G_\infty:=S\Sigma_\infty S^\top,
    \qquad
    C_\infty:=SA_*\Sigma_\infty S^\top,
\]
and the stationary population minimizer
\[
    \boxed{
    B_\infty^*:=C_\infty G_\infty^{-1}.
    }
\]
The stationary objective has the exact excess-risk identity
\[
    \mathcal L_{M,\infty}(B)-\mathcal L_{M,\infty}(B_\infty^*)
    =
    \frac12\|B-B_\infty^*\|_{G_\infty}^2.
\]
In general, \(B^*\neq B_\infty^*\): \(B^*\) minimizes the training
objective, whereas \(B_\infty^*\) minimizes the stationary evaluation. The training trajectories start from \(\Sigma_0\), and
gradient descent is trained toward \(B^*\); but its evaluation distribution is stationary.

\paragraph{Safeguarded gradient descent.}
Let \(\eta_1,\ldots,\eta_L>0\) be a deterministic sequence, fixed
independently of the sketch and training trajectories. Define the cumulative
learning-rate mass \(\mathsf H_t:=\sum_{s=1}^t\eta_s\), with
\(\mathsf H_0=0\), \(\mathsf H:=\mathsf H_L\), and write
\(R:=\mathsf H\) for the effective horizon throughout the scaling formulas.
The distinct font avoids confusion with the training covariance \(H\).
Let \(\eta_{\max}:=\max_t\eta_t\) and set
\[
    a_{\mathrm{safe}}:=\min\left\{1,
    \frac{1}{2\eta_{\max}\|\widehat G\|_{\mathrm{op}}}\right\},
    \qquad \bar\eta_t:=a_{\mathrm{safe}}\eta_t,
\]
where \(a_{\mathrm{safe}}=1\) when \(\widehat G=0\).
Starting from \(B_0=0\), full-batch gradient descent updates
\[
    B_t
    =
    B_{t-1}
    -
    \bar\eta_t
    (B_{t-1}\widehat G-\widehat C),
    \qquad t=1,\ldots,L.
\]
The safeguard rescales all steps uniformly and ensures
\(0\preceq I-\bar\eta_t\widehat G\preceq I\) for every sample. On the
covariance event used in the proofs it is inactive, so
\(\bar\eta_t=\eta_t\); its only purpose is to control the rare complement
event when taking an unconditional expectation over the training trajectories.

Figure~\ref{fig:setup-demonstration} provides a schematic demonstration of the
setup. Initial latent states generate \(N\) trajectories through the shared
transition \(A_*\), the Gaussian sketch produces the observed sequences
\(\{u_{i,p}\}\), and safeguarded GD on the empirical next-token objective
returns the iterate \(B_L\); the displayed words are illustrative rather than
an additional modeling assumption.

\begin{figure*}[t]
    \centering
    \includegraphics[width=\textwidth]{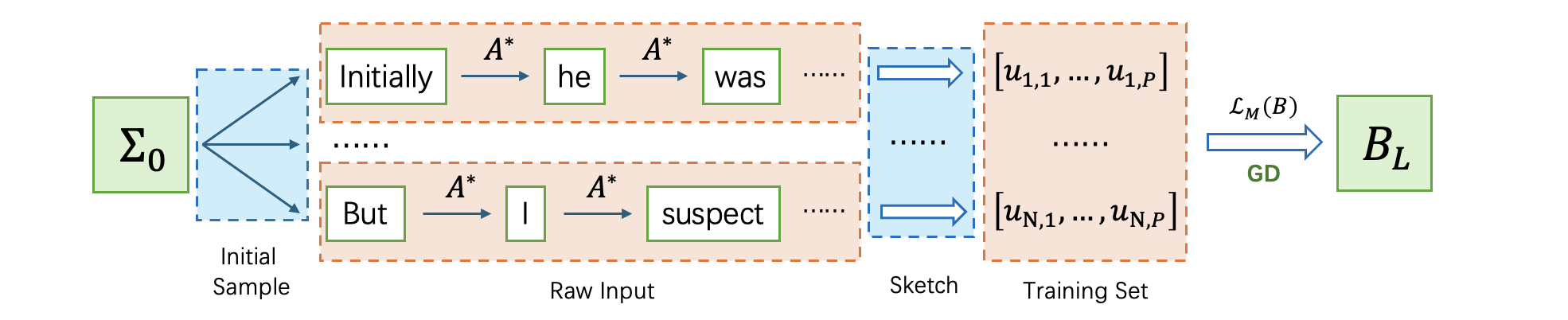}
    \caption{Schematic demonstration of the learning setup, from random
    initialization and latent autoregressive trajectories to sketched training
    sequences and the empirical GD iterate \(B_L\).}
    \label{fig:setup-demonstration}
\end{figure*}

\subsection{Standing assumptions}
\label{sec:standing-assumptions}

\begin{assumption}[Positive stable teacher and source decay]
\label{assump:stable-rnn}
The teacher transition matrix \(A_*\) is symmetric positive definite with orthonormal eigenbasis \(\{e_j\}_{j=1}^d\) and respective eigenvalues \(a_j\in\mathbb R\) which are uniformly stable and follow a power-law decay:
\[
    A_*e_j=a_je_j, \quad 0<a_j\le \rho<1, \quad a_j\asymp j^{-r}.
\]
where \(1\le j\le d\) and \(0<2r<1\).
\end{assumption}

This assumption gives stable dynamics with source spectrum:
leading eigendirections carry more predictable transition energy than
higher-index modes. This is a stylized task-regularity condition, not a claim
that the eigendirections correspond to particular linguistic features. Such
power-law source or target spectra are standard in theoretical scaling-law
analyses \citep{bahri2021explaining,lin2024scaling,lin2025reuse}.
\begin{assumption}[Primitive covariance power laws]
\label{assump:power-law-source}
In the basis of Assumption~\ref{assump:stable-rnn}:
\[
    \Sigma_\xi=
    \sum_{j=1}^d\sigma_{\xi,j}e_je_j^\top,
    \qquad
    \Sigma_0=
    \sum_{j=1}^d\sigma_{0,j}e_je_j^\top.
\]
For exponents \(\alpha>1\) and \(\theta>1\), their eigenvalues satisfy
\[
    \sigma_{\xi,j}\asymp j^{-\alpha},
    \qquad
    \sigma_{0,j}\asymp j^{-\theta},
    \qquad 1\le j\le d.
\]
Together with the source exponent \(r\), assume \(\theta\ge\alpha-2r\) and define \(\beta_\alpha:=\alpha+2r\), \(\beta_\theta:=\theta+2r\).
\end{assumption}

In particular, \(\Sigma_\xi\succ0\) for every finite \(d\). The restriction
\(\theta\ge\alpha-2r\) is the one-step smoothing condition used in the
nonstationary variance lower bound. Coordinatewise,
\[
    a_j^2\sigma_{0,j}
    \asymp j^{-(\theta+2r)}
    \lesssim j^{-\alpha}
    \asymp \sigma_{\xi,j},
\]
so the whitened target's covariance has \(A_*\Sigma_0A_*^\top\preceq C_{\mathrm{sm}}\Sigma_\xi\).

When \(\theta\ge\alpha\), the innovation covariance determines the leading
tail, where when
\(\alpha-2r\le\theta<\alpha\), initialization has the heavier tail. We assume the
teacher, initialization covariance, and innovation covariance are aligned in
a common eigenbasis; this relative-alignment restriction, without loss of generality, is common in Gaussian models, since the Gaussian sketch is invariant under a common orthogonal change of coordinates
\citep{halko2011finding,woodruff2014sketching}.

\begin{assumption}[Stable deterministic learning rates]
\label{assump:stepsize}
The deterministic sequence, independent of the sketch and training data,
satisfies \(0<\eta_t\le\eta_{\max}\le\eta_0\) for every \(t\) and
\(R=\mathsf H_L=\sum_{t=1}^L\eta_t\gtrsim1\). The fixed constant
\(\eta_0>0\) is sufficiently small and depends only on the fixed primitive
data, not on \(M,N,P,L\) or the shape of the schedule. The schedule may depend
on the dimensions and \(L\); no monotonicity or per-step lower bound is required.
\end{assumption}

The constant \(\eta_0\) is chosen small enough that the safeguard is inactive
on the covariance and sketch events. The assumption also excludes the vacuous
regime in which the nominal optimization horizon is too small for the bias and
variance filters to enter their asymptotic regime.
Stable constant, decreasing, and positive cosine schedules are included, as is
the WSD schedule \citep{wen2025wsd} used in our experiments, provided their
cumulative mass satisfies the stated horizon condition.


\section{Main result}
\label{sec:main-result}

We now decompose the stationary next-token prediction risk 
of the iterate
trained on the sketched teacher trajectories, and state
component-wise scaling laws under the power-law source condition. Let \(B_L\) be the \(L\)-step safeguarded empirical GD iterate initialized at
\(B_0=0\). Define its empirical training residual and GD filters by

\begin{align*}
    \widehat{\mathcal B}_L(B^*)
    :=
    B^*\prod_{t=1}^L(I-\bar\eta_t\widehat G),
    \\
    \widehat{\mathcal V}_L(\widehat E)
    :=
    \sum_{t=1}^L
    \bar\eta_t\widehat E
    \prod_{s=t+1}^L(I-\bar\eta_s\widehat G),
\end{align*}
where an empty product is the identity and
\(\widehat E:=\widehat C-B^*\widehat G\). The exact GD identity is
\[
    B_L-B^*
    =
    -\widehat{\mathcal B}_L(B^*)
    +\widehat{\mathcal V}_L(\widehat E).
\]
Note the expectation
inside \(\mathcal L_{M,\infty}\) is over an independent stationary
evaluation state and innovation; the outer expectation
\(\mathbb E_{\mathcal D}[\cdot\mid S]\) averages over the original training
trajectories conditional on the Gaussian sketch. This convention also applies
to all four risk components and both cross terms below.

Define the sapproximation error, population oracle mismatch,
stationary GD bias, and stationary GD variance by
\begin{align*}
    \mathrm{Approx}
    &:={}
    \frac12\inf_{B\in\mathbb R^{M\times M}}
    \|(BS-SA_*)\Sigma_\infty^{1/2}\|_F^2,\\
    \mathrm{Shift}
    &:={}
    \frac12\|B^*-B_\infty^*\|_{G_\infty}^2,\\
    \mathrm{Bias}
    &:={}
    \frac12\mathbb E_{\mathcal D}
    \left[
    \|\widehat{\mathcal B}_L(B^*)\|_{G_\infty}^2
    \mid S
    \right],\\
    \mathrm{Var}
    &:={}
    \frac12\mathbb E_{\mathcal D}
    \left[
    \|\widehat{\mathcal V}_L(\widehat E)\|_{G_\infty}^2
    \mid S
    \right].
\end{align*}
Define the bias--variance and shift interaction terms by
\begin{align*}
    & \mathrm{Cross}_{\mathrm{BV}}
    :=-\mathbb E_{\mathcal D}
    \left[
    \left\langle
    \widehat{\mathcal B}_L(B^*),
    \widehat{\mathcal V}_L(\widehat E)
    \right\rangle_{G_\infty}
    \mid S
    \right],\\
    & \mathrm{Cross}_{\mathrm{S}}
    :=\mathbb E_{\mathcal D}
    \Big[
    \big\langle B^*-B_\infty^*,
    -\widehat{\mathcal B}_L(B^*)
    +\widehat{\mathcal V}_L(\widehat E)
    \big\rangle_{G_\infty}
    \mid S
    \Big].
\end{align*}

Intuitively, \(\mathrm{Approx}\) measures the representation loss of the best stationary
predictor using the sketch. \(\mathrm{Shift}\) measures the stationary loss
from the mismatch between the training and evaluation population oracles.
\(\mathrm{Bias}\) measures the residual optimization error relative to the
training oracle, while \(\mathrm{Var}\) measures the score-driven error from
finitely sampled trajectories. \(\mathrm{Cross}_{\mathrm{BV}}\) captures the
interaction between these optimization and sampling errors;
\(\mathrm{Cross}_{\mathrm{S}}\) captures their interaction with the oracle
mismatch. Both cross terms can be negative, unlike the four squared-error
components.

The irreducible innovation loss is
\(\frac12\operatorname{Tr}(S\Sigma_\xi S^\top)\). Define the total stationary
excess risk above this floor by
\[
    \mathcal E:=
    \mathbb E_{\mathcal D}[\mathcal L_{M,\infty}(B_L)\mid S]
    -\tfrac12\operatorname{Tr}(S\Sigma_\xi S^\top).
\]

\begin{proposition}[Stationary-evaluation risk decomposition and comparison]
\label{prop:main-decomposition}
For the iterate \(B_L\), it has:
\begin{enumerate}[label=(\roman*),leftmargin=*]
\item \textbf{Exact decomposition.}
The evaluation risk satisfies
\[
\begin{aligned}
    \mathbb E_{\mathcal D}[\mathcal L_{M,\infty}(B_L)\mid S]
    ={}
    \frac12\operatorname{Tr}(S\Sigma_\xi S^\top)
    +\mathrm{Approx} \\
    +\mathrm{Shift}
    +\mathrm{Bias}+\mathrm{Var}
    +\mathrm{Cross}_{\mathrm{BV}}
    +\mathrm{Cross}_{\mathrm{S}}.
\end{aligned}
\]
\item \textbf{Total stationary excess risk.}
Under the hypotheses of Theorem~\ref{thm:main-scaling-law}, with probability
at least \(1-\exp(-\Omega(M))\) over the Gaussian sketch,
\[
    \mathcal E\asymp\mathrm{Approx}+\mathrm{Bias}+\mathrm{Var}.
\]
This is a two-sided comparison up to fixed multiplicative constants after
controlling both cross terms.
\end{enumerate}
\end{proposition}

The following component-wise stationary rates are established in the
appendices; the training covariance and empirical GD recursion remain as in
the setup above, while evaluation uses the stationary covariance.

\begin{theorem}[Stationary scaling law for empirical GD]
\label{thm:main-scaling-law}
Suppose Assumptions~\ref{assump:stable-rnn},
\ref{assump:power-law-source}, and \ref{assump:stepsize} hold. In both
regimes, put \(\delta=(\alpha-\theta)_+\), and define the effective cutoff scale
\[
    k:=\mathsf H^{1/\alpha}+(\mathsf H/P)^{1/\theta}.
\]
Assume sufficiently strict unsaturation,
\[
    \mathsf H(M^{-\alpha}+P^{-1}M^{-\theta})\le c_{\mathrm{unsat}},
\]
and the common sufficient sample condition
\begin{equation}
\label{eq:unified-sample-condition}
\begin{aligned}
    NP\gtrsim{}&
    \frac{k+\log(M+\mathsf H+2)}{(1-\rho)^2}
    \left(\frac Mk\right)^{1-2r}\\
    &\quad\times
    \left(1+\frac{k^\delta}{P}\right)
    \left(1+\min\{P,M^\delta\}\right)\log^2(e+\mathsf H).
\end{aligned}
\end{equation}
Here \(c_{\mathrm{unsat}}>0\) is sufficiently small and the implicit
sample-size constant is sufficiently large. Then the following rates hold
with probability at least \(1-\exp(-\Omega(M))\) over the Gaussian sketch.
Also, constants depend only on the fixed primitive data, including
\(\rho<1\) and the fixed stability bound \(\eta_0\), but not on
\(M,N,P,\mathsf H,L\) or the schedule shape. The rates depend on the deterministic
schedule through its cumulative horizon \(\mathsf H=\mathsf H_L\), up to constants.

\paragraph{Regime I: \(\theta\ge\alpha\).}
We have
\[
    \boxed{\mathrm{Approx}\asymp M^{1-\beta_\alpha}},\qquad
    \mathrm{Shift}\lesssim P^{-2}M^{1-\beta_\alpha}.
\]
\[
    \boxed{\mathrm{Bias}\asymp \mathsf H^{(1-\beta_\alpha)/\alpha}},
    \qquad
    \boxed{\mathrm{Var}\asymp\frac{\mathsf H^{1/\alpha}}{NP}}.
\]

\paragraph{Regime II: \(\alpha-2r\le\theta<\alpha\).}
We have
\[
    \boxed{\mathrm{Approx}\asymp M^{1-\beta_\alpha}}.
\]
The population shift satisfies
\[
    \mathrm{Shift}\lesssim M^{1-\beta_\alpha}
    \left(\frac{M^{\alpha-\theta}}{P+M^{\alpha-\theta}}\right)^2.
\]
\[
    \boxed{\mathrm{Bias}\asymp
    \begin{cases}
        \mathsf H^{(1-\beta_\alpha)/\alpha},
        &\mathsf H\lesssim P^{\alpha/(\alpha-\theta)},\\[1mm]
        (\mathsf H/P)^{(1-\beta_\alpha)/\theta},
        &\mathsf H\gtrsim P^{\alpha/(\alpha-\theta)},
    \end{cases}}
\]
\[
    \boxed{\mathrm{Var}\asymp
    \begin{cases}
        \dfrac{\mathsf H^{1/\alpha}}{NP},
        &\mathsf H\lesssim P^{\alpha/(\alpha-\theta)},\\[3mm]
        \dfrac1N(\mathsf H/P)^{(1-\alpha+\theta)/\theta},
        &\mathsf H\gtrsim P^{\alpha/(\alpha-\theta)}.
    \end{cases}}
\]
\end{theorem}

\paragraph{Proof sketch.}
Gaussian-sketch source estimates give the stationary Approximation rate.
The averaged training spectrum is comparable to
\(j^{-\alpha}+P^{-1}j^{-\theta}\); balancing it against \(\mathsf H^{-1}\)
identifies the effective cutoff \(k\). Population GD filter bounds and
stationary covariance weights then give
\[
    \mathrm{Bias}^{\mathrm{pop}}\asymp k^{1-\beta_\alpha},\qquad
    \mathrm{Var}^{\mathrm{pop}}\asymp
    \frac{k}{NP(1+k^\delta/P)}.
\]
Substituting the dominant term in \(k\) yields the two regimes and the
Regime-II crossover. The common sample condition ensures concentration of
the empirical covariance at regularization scale \(\mathsf H^{-1}\). Ordered
noncommutative filter perturbation bounds transfer the population rates to
empirical GD; safeguard and moment bounds control the complementary event.
Appendix~\ref{app:proof-main-scaling-law} gives the component-wise proof.
The separate cross-term and Shift controls needed for
Proposition~\ref{prop:main-decomposition}(ii) are proved in
Appendices~\ref{app:proof-main-decomposition} and
\ref{app:total-risk-comparison}.

\subsection{Interpretations}
\label{subsec:interpretations}

\paragraph{Regime I: when segmentation is neutral for Variance.}
When \(\theta\ge\alpha\), the innovation spectrum determines the leading
training tail, and the training and stationary evaluation spectra are both
comparable to \(j^{-\alpha}\). Approximation and Bias have their one-scale
rates, while Variance scales as \(\mathsf H^{1/\alpha}/(NP)\). Thus trajectory count and length are interchangeable
at leading order for Variance, subject to the theorem's sample-size and
stability conditions.
This limit is related to sketched data-reuse scaling
\citep{lin2025reuse} and to analyses of learning from multiple trajectories
\citep{tu2024learning,xinchiu2022autonomous}.

\paragraph{Regime II: spectral excitation from restarts.}
When \(\alpha-2r\le\theta<\alpha\), shorter trajectories retain more of the
heavy initialization contribution in the training spectrum
\(j^{-\alpha}+P^{-1}j^{-\theta}\), while the stationary evaluation spectrum
remains comparable to \(j^{-\alpha}\). This can accelerate learning in weak
directions: in the initialization-dominated range, Bias scales as
\((\mathsf H/P)^{(1-\beta_\alpha)/\theta}\). At fixed \(\mathsf H\) and within this range,
Variance scales as a negative power of \(P\) at fixed \(N\), but at fixed
\(T=NP\) it scales as
\(T^{-1}\mathsf H^{(1-\alpha+\theta)/\theta}P^{(\alpha-1)/\theta}\), a positive
power of \(P\). These are componentwise scaling comparisons within the
initialization-dominated theorem range, not a universal ordering of finite-sample
total risk.

\paragraph{Excitation under fixed stationary evaluation.}
Changing \(P\) changes the training spectrum but not the stationary evaluation
objective. Stationary Approximation is independent of \(P\), and the shift
between training and stationary oracles is bounded by Approximation, although
it can be of the same order in Regime II
(Appendix~\ref{app:population-distribution-shift}). Thus extra training
excitation can change finite-time Bias and Variance without adding a larger
leading oracle-shift order. The population-filter Bias in
Section~\ref{sec:experiments} is a related diagnostic, not a direct estimate
of empirical Bias.

\subsection{Compute-optimal allocation}
\label{subsec:compute-optimal-allocation}

Initial-state excitation changes optimization efficiency as well as the sample
requirements for reliable empirical learning. We therefore consider the joint
allocation of trajectory count and length, width, and optimization horizon
under token and compute caps.

\begin{corollary}[Stationary compute-optimal allocation under resource caps]
\label{cor:compute-optimal-allocation}
Impose the resource caps
\[
    NP\le\mathcal T,\qquad NPM^2L\le\mathcal C,\qquad \mathcal K=\eta_0\mathcal C,
\]
and optimize over stable deterministic schedules and the other parameters,
within the common sample and strict unsaturation conditions of 
Theorem~\ref{thm:main-scaling-law}. Put
\(\zeta=\min\{\alpha,\theta\}\) and \(b=\alpha+2r-1\).
For sufficiently large budgets and ambient dimension, the optimal risk order on the theorem's sketch events is
\[
    \boxed{
    k_{\mathrm{opt}}\asymp
    \min\left\{\mathcal T^{1/(\zeta+2r)},
                 \mathcal K^{1/(2\zeta+2r+2)}\right\},
    \qquad
    \mathcal E_{\mathrm{opt}}\asymp k_{\mathrm{opt}}^{-b}.
    }
\]
An attaining allocation in either regime is
\[
    P=2,\qquad M\asymp k_{\mathrm{opt}},\qquad
    \mathsf H\asymp k_{\mathrm{opt}}^\zeta,\qquad
    NP\asymp k_{\mathrm{opt}}^{\zeta+2r},
\]
where the width constant is sufficiently large for strict unsaturation and a
fixed stable constant step size gives \(L\asymp \mathsf H/\eta_0\).
The critical token scale is
\[
    \mathcal T_{\mathrm{crit}}
    \asymp\mathcal K^{(\zeta+2r)/(2\zeta+2r+2)}.
\]
Thus the token-rich risk is
\(\mathcal K^{-b/(2\zeta+2r+2)}\), whereas the token-limited risk is
\(\mathcal T^{-b/(\zeta+2r)}\). The latter allocation may leave compute
unused. Optimality is only within the theorem-controlled range under resource
upper bounds, not under exact compute equality or over all learning methods.
For a prescribed schedule, physical compute still depends on its actual
number of steps \(L\), which may be much larger than \(\mathsf H/\eta_0\).
\end{corollary}

The constant-length construction shows that long trajectories are not needed
to attain this risk order in our model; it need not be the unique optimal
allocation and does not imply strict superiority over every longer trajectory.

\paragraph{Compute feasibility versus risk optimization.}
For \(M=\Theta(k)\) with a sufficiently large width constant and fixed \(P\),
the common sufficient sample condition becomes
\[
    NP\gtrsim k^{1+\delta}\log^2(e+k),
    \qquad \mathsf H\asymp k^\zeta,\qquad
    \delta=(\alpha-\theta)_+.
\]
With a fixed stable constant learning rate, choosing this sample order gives
a sufficient normalized compute budget of order
\(k^{\alpha+3}\log^2(e+k)\). This is a feasibility construction, not an
information-theoretic lower bound or a risk-optimal allocation. The corollary
instead balances Bias and Variance using
\(NP\asymp k^{\zeta+2r}\); this satisfies the common condition with
polynomial slack because \(\zeta+2r>1+\delta\).

\section{Experiments}
\label{sec:experiments}

\begin{figure*}[t]
    \centering
    \begin{minipage}[t]{0.32\linewidth}
        \centering
        \includegraphics[width=\linewidth]{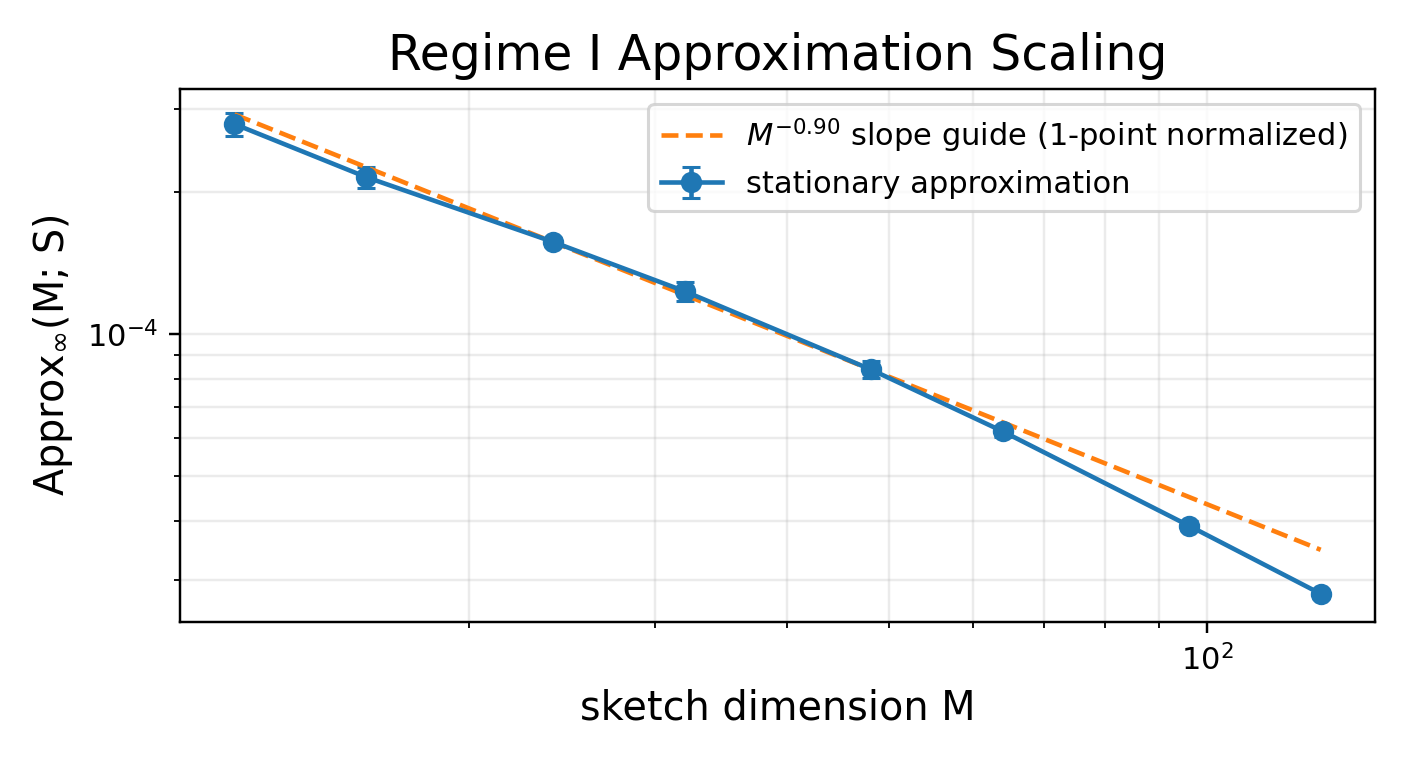}
        {\scriptsize (a) Regime I: Approximation vs. \(M\).}
    \end{minipage}
    \hfill
    \begin{minipage}[t]{0.32\linewidth}
        \centering
        \includegraphics[width=\linewidth]{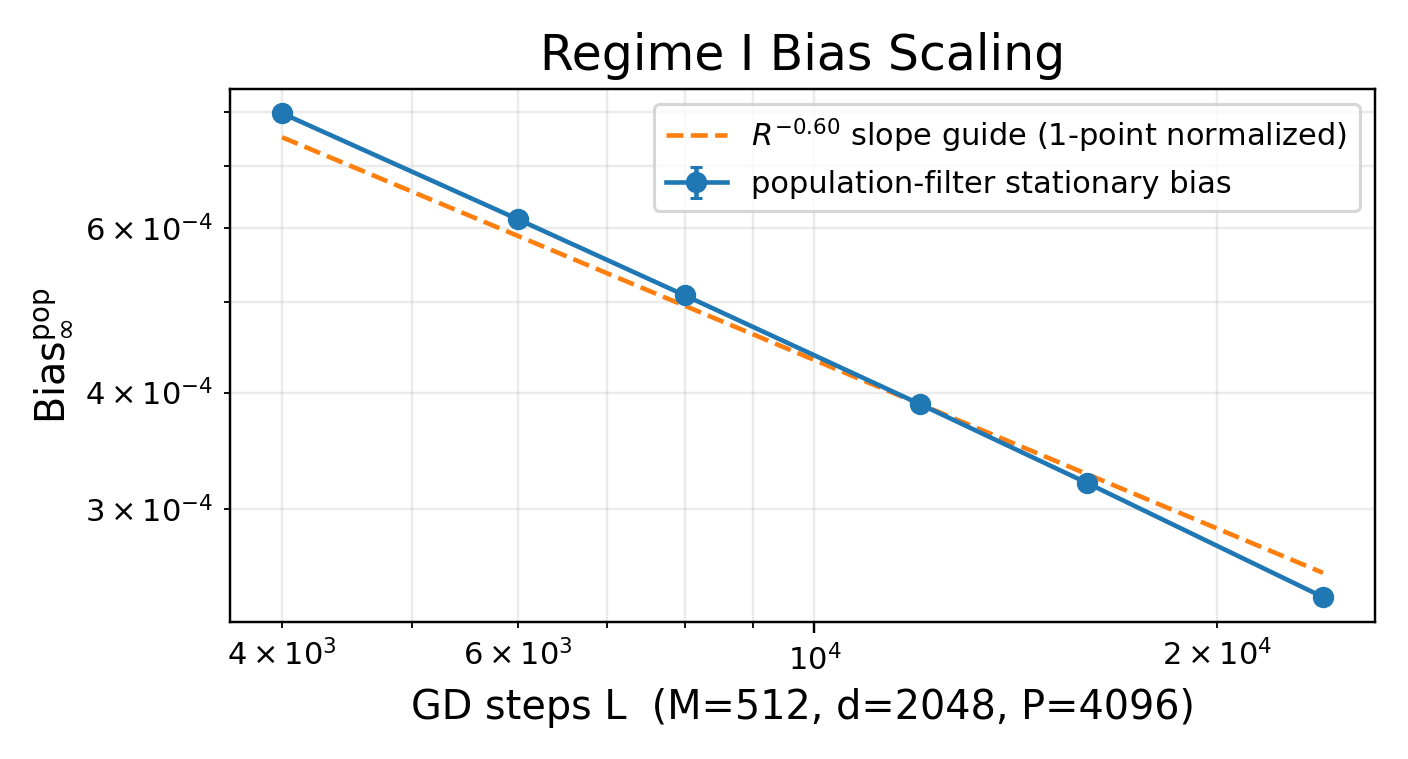}
        {\scriptsize (b) Regime I: population Bias vs. \(R\).}
    \end{minipage}
    \hfill
    \begin{minipage}[t]{0.32\linewidth}
        \centering
        \includegraphics[width=\linewidth]{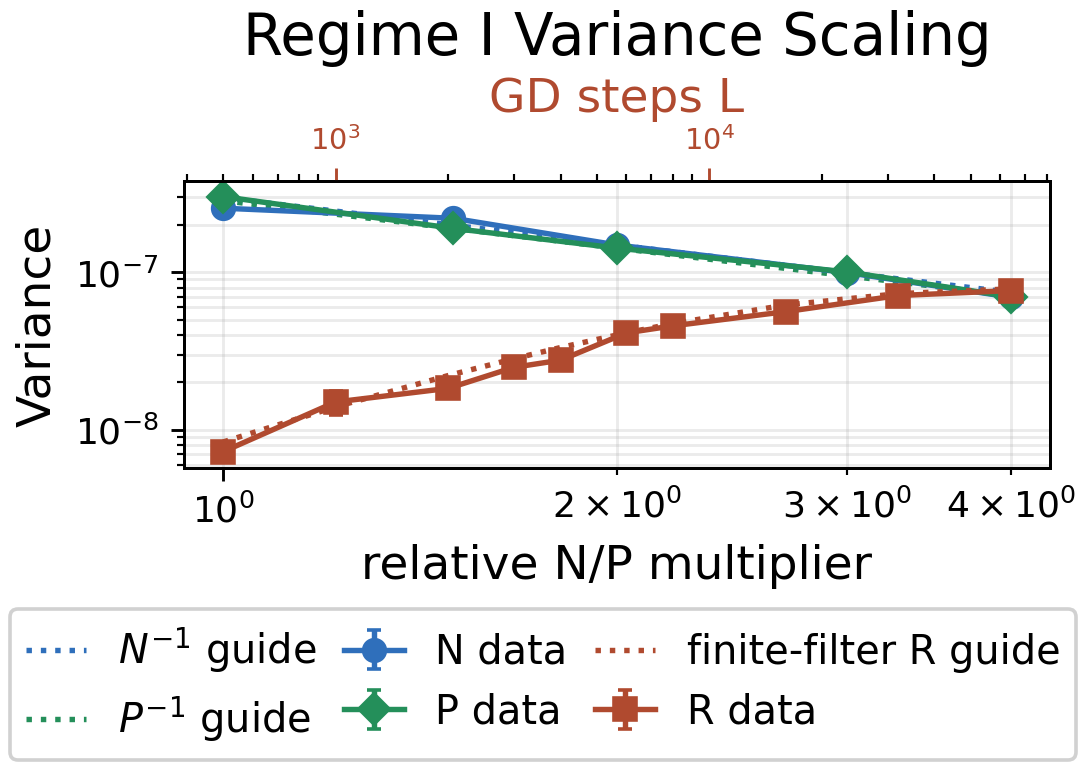}
        {\scriptsize (c) Regime I: empirical Variance.}
    \end{minipage}

    \vspace{0.8em}
    \begin{minipage}[t]{0.32\linewidth}
        \centering
        \includegraphics[width=\linewidth]{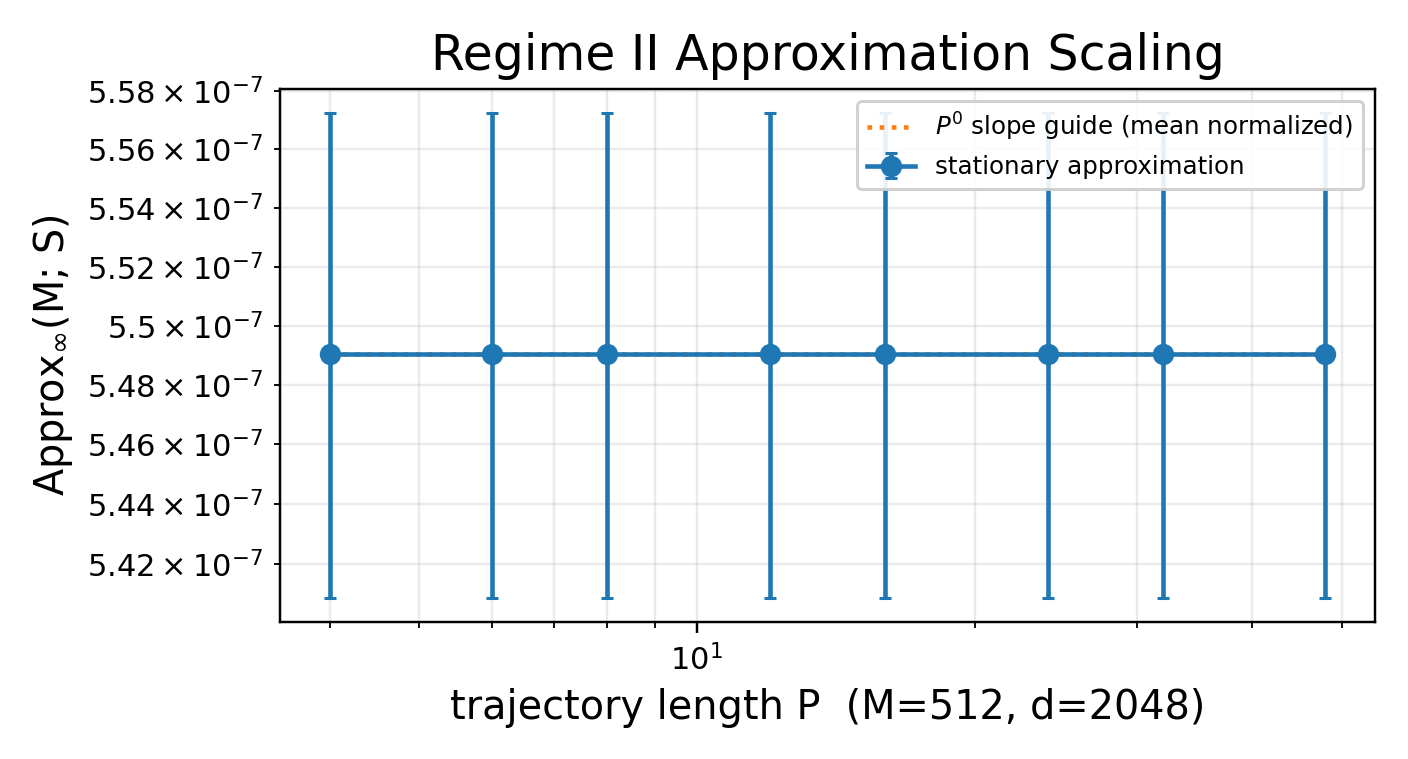}
        {\scriptsize (d) Regime II: Approximation vs. \(P\).}
    \end{minipage}
    \hfill
    \begin{minipage}[t]{0.32\linewidth}
        \centering
        \includegraphics[width=\linewidth]{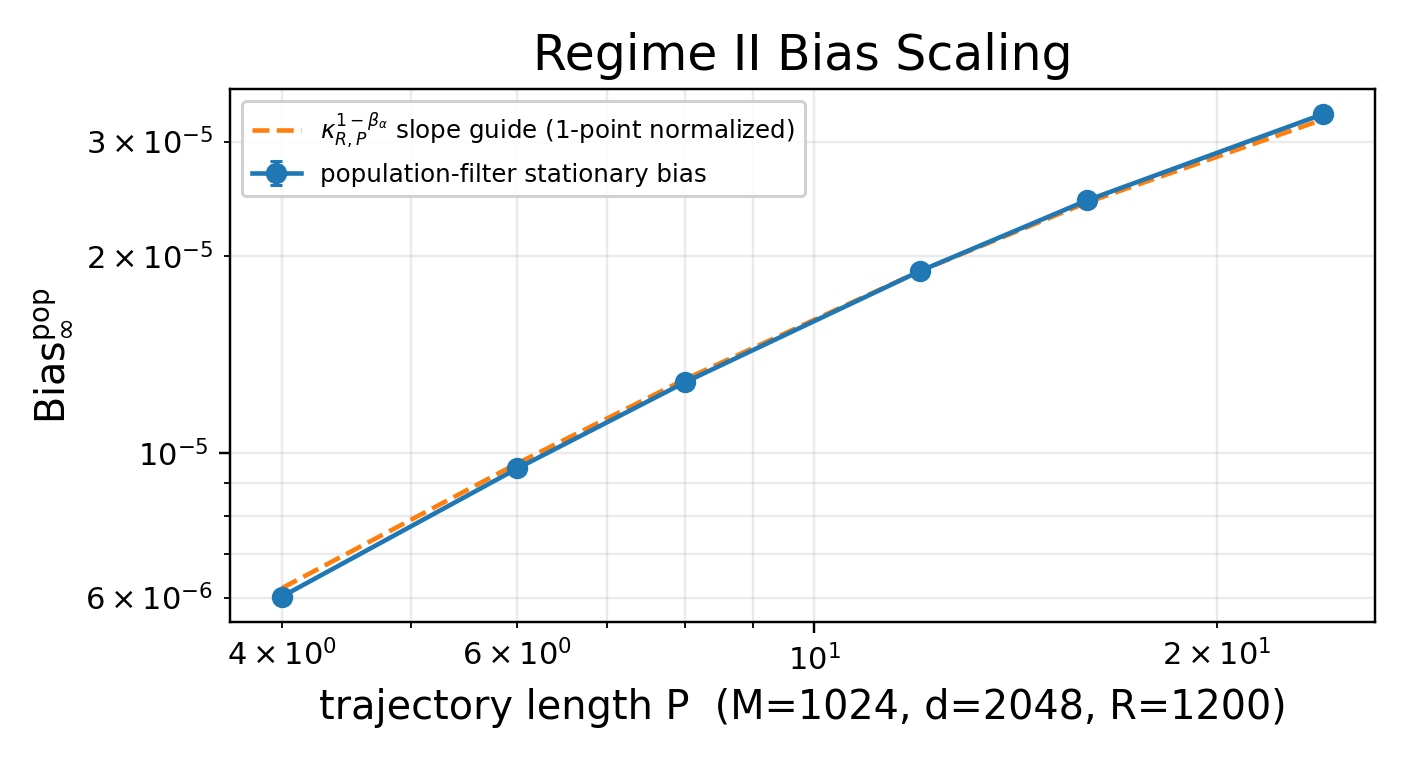}
        {\scriptsize (e) Regime II: population Bias vs. \(P\).}
    \end{minipage}
    \hfill
    \begin{minipage}[t]{0.32\linewidth}
        \centering
        \includegraphics[width=\linewidth]{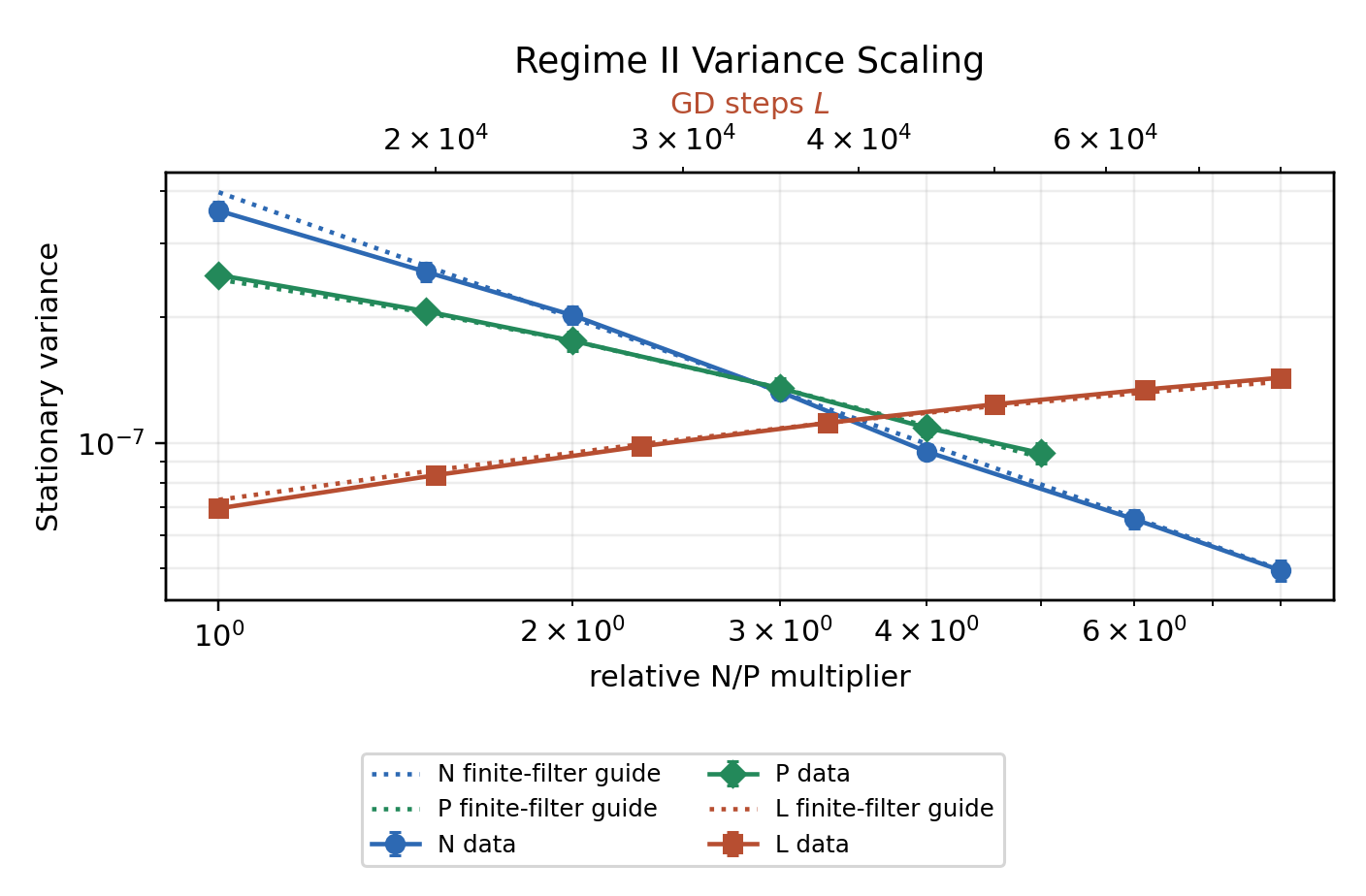}
        {\scriptsize (f) Regime II: empirical Variance.}
    \end{minipage}
    \caption{Stationary-evaluation component experiments. Bias panels use
    population filters; Variance is estimated from sampled trajectories. The
    variance panels use separate horizontal axes for \(N,P\) and \(L\).
    Dashed and dotted guides are normalized at one empirical checkpoint.}
    \label{fig:synthetic-exp}
\end{figure*}

We train on nonstationary Gaussian trajectories of the diagonal teacher
\(x_{i,p+1}=A_*x_{i,p}+\xi_{i,p+1}\), using independent Gaussian sketches,
and evaluate prediction-risk components under the stationary distribution.
The teacher, innovation, and initialization spectra decay as \(j^{-r}\),
\(j^{-\alpha}\), and \(j^{-\theta}\), respectively. We use
\((\alpha,\theta,r)=(1.5,2.5,0.2)\) in Regime I and
\((1.8,1.25,0.35)\) in Regime II, with WSD gradient descent, peak step size
\(\gamma=0.1\). In this section and the figures, \(R=\gamma L\) denotes the
experimental nominal horizon; the theorem's cumulative horizon
\(\mathsf H_L=\sum_t\eta_t\) is comparable to it for this WSD schedule.
The experiments use this schedule as an illustration; the theorem's
unspecified sufficient upper bound on the step size is not separately verified.
The samples are complete
trajectories, including their within-trajectory dependence. Full parameters,
grids, and diagnostics are in Appendix~\ref{app:additional-experiments}.

\begin{figure*}[t]
    \centering
    \includegraphics[width=0.9\linewidth]{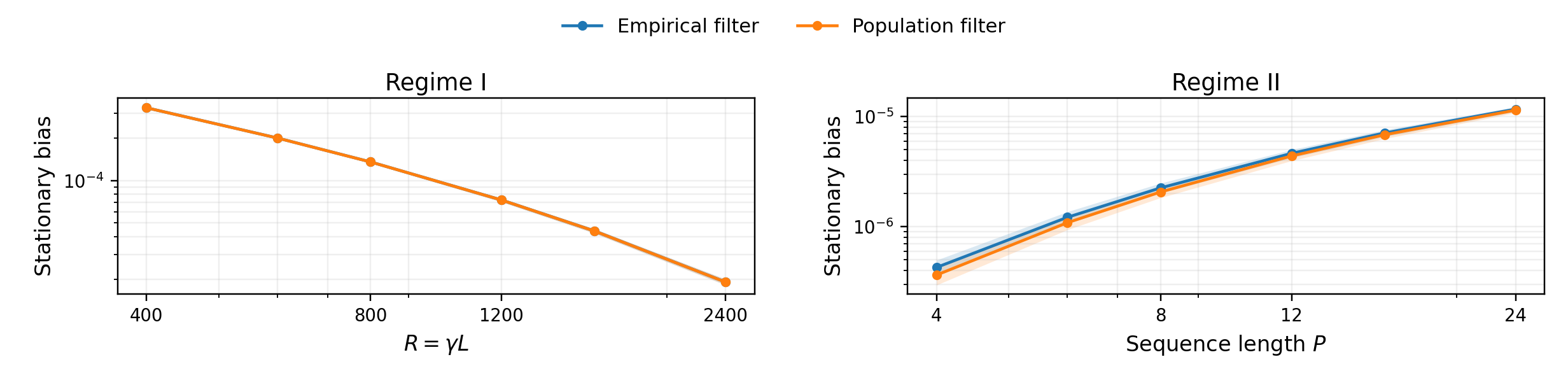}
    \caption{Empirical and population WSD Bias on the same sketches, with
    \((M,d,N)=(24,96,1024)\). Regime I varies \(R\) at \(P=64\); Regime II
    varies \(P\) at \(R=1200\). Shading denotes percentile 95\% bootstrap
    intervals over six sketch clusters, each with three trajectory repeats.}
    \label{fig:bias-validation}
\end{figure*}

\begin{figure*}[t]
    \centering
    \includegraphics[width=0.9\linewidth]{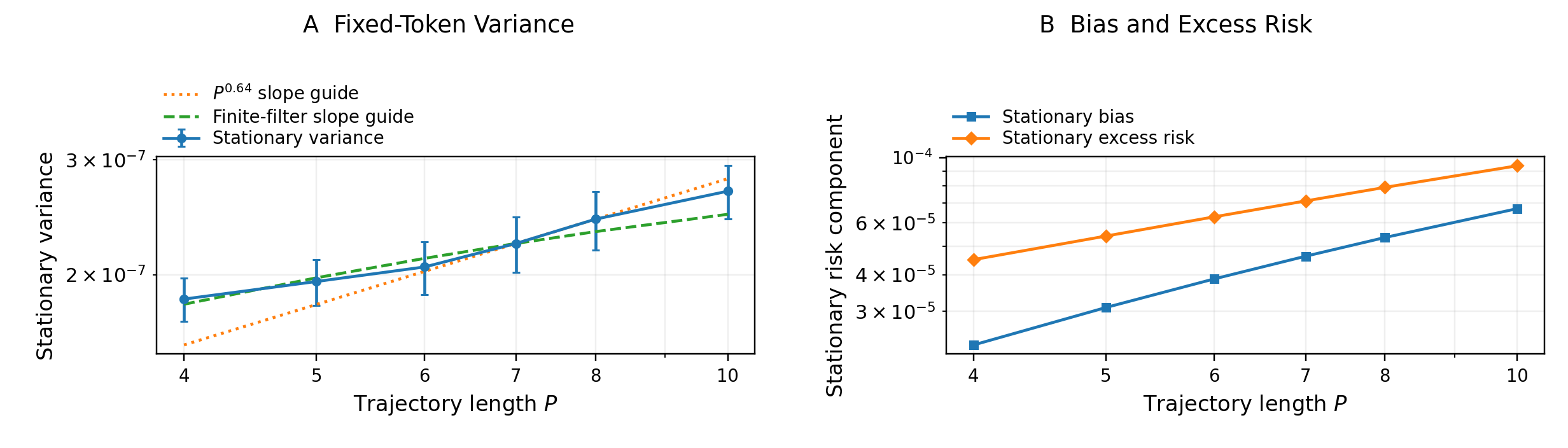}
    \caption{Regime-II fixed-token allocation with \(NP\approx262144\),
    \((M,d,R)=(48,192,240)\), eight sketches, and four trajectory repeats per
    sketch. The variance slope has a sketch-cluster bootstrap interval
    \([+0.335,+0.499]\). Both slope guides are normalized at one checkpoint;
    error bars show standard errors across sketches.}
    \label{fig:fixed-token-exp}
\end{figure*}

\paragraph{Approximation.}
These experiments check whether the stationary representation error depends on
the training trajectory length.
Figure~\ref{fig:synthetic-exp}(a) varies the sketch size \(M\) in Regime I;
the measured slope is \(-0.962\), compared with the \(-0.900\) scaling guide.
In Regime II, Figure~\ref{fig:synthetic-exp}(d) varies the trajectory length
\(P\) using the same sketches and confirms that stationary Approximation is
independent of \(P\).

\paragraph{Bias.}
The Bias sweeps probe how training spectral excitation changes finite-time
learning under stationary evaluation.
The original bias sweeps evaluate the population WSD filter, not the empirical
filter in the risk decomposition. The Regime-I slope against \(R\) is
\(-0.666\), compared with \(-0.600\), and the Regime-II slope against \(P\)
is \(+0.954\), compared with \(+0.927\) for the finite mixed-spectrum guide
(Figure~\ref{fig:synthetic-exp}(b,e)). These are finite-grid comparisons;
the latter does not establish the asymptotic \(+1.2\) exponent.

\paragraph{Empirical versus population bias.}
Figure~\ref{fig:bias-validation} compares the two filters on the same sketches,
using sampled trajectories to form the empirical covariance. At
\((M,d,N)=(24,96,1024)\), they are close in Regime I. In Regime II the
empirical mean is about \(17\%\) higher at \(P=4\), with the relative gap
decreasing to about \(2\%\) at \(P=24\). This smaller-dimensional experiment
provides a direct finite-sample comparison, not a validation of the original
larger-dimensional bias sweeps or their asymptotic exponents.

\paragraph{Variance.}
The Variance sweeps distinguish the effects of trajectory count, length, and
optimization horizon on sampling error.
Figures~\ref{fig:synthetic-exp}(c,f) vary \(N\), \(P\), and \(R\) using
empirical trajectory-based variance. In Regime I the slopes against \(N\)
and \(P\) are \(-0.972\) and \(-1.027\), while the \(R\)-slope is \(+0.482\)
versus the unsaturated \(+0.667\) guide. In Regime II the slopes are
\(-0.970\), \(-0.620\), and \(+0.413\), close to the respective finite-filter
guide slopes \(-1.000\), \(-0.614\), and \(+0.374\). These observations do not
establish the sharp asymptotic exponents; in particular the fixed-\(N\)
\(P\)-slope differs from the initialization-dominated prediction \(-0.36\).

\paragraph{Fixed-token allocation.}
This comparison isolates trajectory organization at approximately fixed
transition count.
At approximately fixed \(NP\), Figure~\ref{fig:fixed-token-exp} tests the
effect of changing \(P\) in Regime II. Variance increases over the tested grid,
with fitted slope \(+0.427\) and eight-sketch bootstrap interval
\([+0.335,+0.499]\); Bias and total excess risk also increase. The observed
variance slope is closer to the finite-filter guide \(+0.347\) than to the
asymptotic value \(+0.640\). Because \(M\) and \(R\) are fixed, this is not a
test of the joint compute-optimal allocation.

\paragraph{Scope.}
An additional matched-horizon comparison of constant, WSD, and cosine learning
rates gives nearly identical empirical Bias and Variance in both regimes
(Appendix~\ref{app:general-lr-experiment}). The steps are small in this
experiment, so it is a limited schedule comparison rather than a test of all
admissible learning-rate sequences.
The regularized covariance diagnostic in Appendix~\ref{app:additional-experiments}
decreases with \(N\) in both regimes, but was measured on a separate,
smaller-dimensional configuration. Together with possible finite-size
saturation and the limited number of sketches, it means these results provide
finite-sample evidence rather than verification of every theorem condition.
Plotted guides are normalized at one checkpoint and do not test prefactors.


\section{Related work}
\label{sec:related-work}

\paragraph{Solvable scaling-law models.}
Our closest methodological inspiration is
\citet{lin2024scaling,lin2025reuse}, who derive sketched linear-regression
scaling laws in model size, data, and optimization time. Related multi-epoch
analysis shows that useful data reuse can itself grow with dataset size
\citep{yan2026largerdatasets}. For sequence models,
\citet{lyu2025solvable} characterize training-, model-, data-, and
context-length scaling in a solvable self-attention model. We instead study a
recurrent matrix predictor whose nonstationary training covariance depends on
trajectory length, and explicitly distinguish that covariance from the fixed
stationary evaluation distribution.
Learning-rate dynamics, including the warmup-stable-decay schedule, have also
been studied in neural-network optimization \citep{wen2025wsd}. Our analysis
instead covers a class of stable deterministic schedules through their
cumulative learning-rate mass, subject to the stated assumptions.

\paragraph{Linear recurrent models.}
\citet{proca2025learning} analyze how task dynamics influence learning
dynamics, stability, and representations in linear RNNs. At the empirical
level, \citet{shen2024linearcomplexity} study scaling in several
linear-complexity language models, including a linear RNN. Our work instead
derives stationary prediction-risk scaling for a Gaussian-sketched linear
recurrent teacher--student model as representation width, trajectory
organization, and optimization horizon vary jointly.

\paragraph{Learning from multiple trajectories.}
Learning system dynamics from multiple trajectories is an established topic.
\citet{xinchiu2022autonomous} study learning autonomous linear dynamics from
multiple trajectories, and \citet{tu2024learning} analyze statistical learning
from dependent trajectory data, including linear-system identification.
These works make clear that the usefulness of multiple trajectories is not
itself new. Our focus is different: Gaussian-sketched representation error,
nonstationary initialization versus stationary evaluation, finite-time
empirical GD with stable deterministic learning rates, stationary Bias and Variance, and the resulting
sequence-length-dependent spectral cutoff. Related work on long-context linear
autoregressive systems obtains near-i.i.d. parametric sample complexity under
its own assumptions \citep{yuksel2025longcontext}.

\paragraph{Sequence organization in language-model pretraining.}
Empirical studies report effects of document truncation and sequence-length
distributions on language-model performance
\citep{ding2024truncations,pouransari2024datasetdecomposition,
gao2025longcontext}. These results motivate treating sequence organization as
consequential, but they do not directly test the spectral assumptions or the
stationary-prediction mechanism studied here. Theoretical work on next-token
prediction and autoregressive pretraining studies other mechanisms
\citep{li2024mechanics,li2025ntpgeneralization,gong2025autoregressive,
riechers2025nexttoken,kunstner2025bigram,cagnetta2025hierarchical}. Our
contribution is a solvable recurrent model that jointly tracks \(M,N,P,R\),
stationary risk components, and resource-constrained allocation; it is not a
claim about optimal context length in current large language models.


\section{Conclusion}
\label{sec:conclusion}

The statistical value of sequential data is not determined by transition count alone: initial-state transients can change training covariance geometry and the finite-time learning of weak directions, even when the evaluation task remains stationary. In our Gaussian-sketched linear recurrent model, this mechanism yields stationary Bias and Variance scaling laws and resource-allocation consequences; constant-length trajectories attain the optimal risk order within the theorem-controlled range. Our analysis assumes stable linear-Gaussian dynamics, aligned power-law spectra, independent trajectory initializations, a fixed stationary evaluation task, and full-batch gradient descent. Extending these insights to learned representations, nonlinear recurrent architectures, stochastic optimization, and language-model training remains future work.

\newpage
\bibliography{ref}
\bibliographystyle{icml2026}

\clearpage
\onecolumn
\appendix
\counterwithin{mytheorem}{section}
\begin{center}
	\textbf{Scaling Laws for Linear RNNs\\ \vspace{8pt} \textit{Supplementary Material}}
\end{center}

\counterwithin{figure}{section}
\counterwithin{table}{section}


\section{Preliminaries: sketched autoregressive prediction}
\label{app:preliminaries}

\subsection{Spectral comparisons and sketched regression identities}
\label{app:preliminaries-identities}

\paragraph{Decreasing comparison envelopes.}
\label{app:comparison-envelopes}
The power laws are two-sided coordinate bounds in the common teacher basis;
the actual sequences need not be monotone. For \(q>0\), define
\[
    F_q:=\operatorname{diag}(1^{-q},\ldots,d^{-q}),
    \qquad
    f_P(j):=j^{-\alpha}+P^{-1}j^{-\theta},
    \qquad
    F_P:=F_\alpha+P^{-1}F_\theta.
\]
Lemma~\ref{lem:primitive-positionwise-regularity} gives
\[
    cF_P\preceq H\preceq CF_P,
    \qquad
    cF_\alpha\preceq\Sigma_\infty\preceq CF_\alpha,
\]
\[
    c(F_{\alpha+r}+P^{-1}F_{\theta+r})
    \preceq A_*H
    \preceq C(F_{\alpha+r}+P^{-1}F_{\theta+r}).
\]
These comparisons preserve the original coordinate order and hold uniformly
in \(P\ge2\). Their constants may depend on \(\rho<1\) and the fixed
primitive power-law bounds. Throughout the scaling statements, these
primitive constants, including \(\rho\), are fixed; no uniformity as
\(\rho\uparrow1\) is asserted unless a bound explicitly specifies it.

All coordinate sums are finite; for example,
\[
    \sum_{j>m}b_j:=\sum_{j=m+1}^{d}b_j.
\]
Upper bounds may compare such a sum with the corresponding infinite
power-law tail. Matching lower bounds use fixed annuli of order \(M\), which
exist when \(d\ge C_{\mathrm{amb}}M\).

\paragraph{Sketched autoregressive regression.}
For each trajectory \(i=1,\dots,N\), the learner observes
\[
    u_{i,p}:=Sx_{i,p}\in\mathbb R^M,
    \qquad
    y_{i,p}:=u_{i,p+1},
    \qquad p=0,\dots,P-1.
\]
The population covariance and cross-covariance are, for any fixed trajectory
index because the trajectories are identically distributed,
\[
    G:=\frac1P\sum_{p=0}^{P-1}\mathbb E[u_{i,p}u_{i,p}^\top],
    \qquad
    C:=\frac1P\sum_{p=0}^{P-1}\mathbb E[y_{i,p}u_{i,p}^\top],
\]
and their empirical versions are
\[
    \widehat G
    :=
    \frac1{NP}\sum_{i=1}^N\sum_{p=0}^{P-1}u_{i,p}u_{i,p}^\top,
    \qquad
    \widehat C
    :=
    \frac1{NP}\sum_{i=1}^N\sum_{p=0}^{P-1}y_{i,p}u_{i,p}^\top.
\]
The original population objective is defined before imposing the sketch. For any
ambient transition \(A\in\mathbb R^{d\times d}\), define
\[
    \mathcal L(A)
    :=
    \frac1{2P}\sum_{p=0}^{P-1}
    \mathbb E\|x_{i,p+1}-Ax_{i,p}\|_2^2.
\]
Using \(x_{i,p+1}=A_*x_{i,p}+\xi_{i,p+1}\) and independence of the mean-zero
innovation from \(x_{i,p}\),
\[
    \mathcal L(A)
    =
    \mathcal L(A_*)
    +
    \frac12
    \left\|
    (A-A_*)H^{1/2}
    \right\|_F^2,
    \qquad
    \mathcal L(A_*)=\frac12\operatorname{Tr}(\Sigma_\xi).
\]
Thus \(A_*\) is the original population minimizer and \(\mathcal L(A_*)\) is the
full-state irreducible innovation loss.

The trained sketched model uses only the observed variables
\(u_{i,p}=Sx_{i,p}\) and \(y_{i,p}=Sx_{i,p+1}\). Its observable population objective is
\[
    \mathcal L_M(B)
    :=
    \frac1{2P}\sum_{p=0}^{P-1}
    \mathbb E\|y_{i,p}-Bu_{i,p}\|_2^2,
    \qquad B\in\mathbb R^{M\times M}.
\]
Equivalently,
\[
    \mathcal L_M(B)
    =
    \frac12\operatorname{Tr}(S\Sigma_\xi S^\top)
    +
    \frac12
    \left\|
    (BS-SA_*)H^{1/2}
    \right\|_F^2.
\]
This is not the full-state risk of any lifted ambient matrix; using such a
lifted risk would change the approximation, bias, and variance terms analyzed
below.
The empirical sketched objective is
\[
    \widehat{\mathcal L}_M(B)
    :=
    \frac1{2NP}\sum_{i=1}^N\sum_{p=0}^{P-1}
    \|y_{i,p}-Bu_{i,p}\|_2^2.
\]
The gradients are
\[
    \nabla\widehat{\mathcal L}_M(B)=B\widehat G-\widehat C,
    \qquad
    \nabla\mathcal L_M(B)=BG-C.
\]
Thus
\[
    B_*:=CG^{-1}
\]
is the sketched population minimizer. Using the \(G\)-weighted norm from the
notation paragraph, we have
\[
    \mathcal L_M(B)-\mathcal L_M(B_*)
    =
    \frac12\|B-B_*\|_G^2.
\]
The population approximation term is
\[
    \operatorname{Approx}_M
    :=
    \mathcal L_M(B_*)-
    \frac12\operatorname{Tr}(S\Sigma_\xi S^\top)
    =
    \inf_{B\in\mathbb R^{M\times M}}
    \frac12
    \left\|
    (BS-SA_*)H^{1/2}
    \right\|_F^2.
\]

\paragraph{Empirical gradient descent.}
Full-batch GD on \(\widehat{\mathcal L}_M\) is
\[
    B_t=B_{t-1}-\bar\eta_t(B_{t-1}\widehat G-\widehat C),
    \qquad t=1,\dots,L,
    \qquad B_0=0.
\]
Let
\[
    \Delta_t:=B_t-B_*,
    \qquad
    \widehat E:=\widehat C-B_*\widehat G.
\]
Then
\[
    \Delta_t
    =
    \Delta_{t-1}(I-\bar\eta_t\widehat G)+\bar\eta_t\widehat E.
\]
Iterating the recursion gives
\[
    B_L-B_*
    =
    -B_*\prod_{t=1}^L(I-\bar\eta_t\widehat G)
    +
    \sum_{t=1}^L
    \bar\eta_t\widehat E
    \prod_{s=t+1}^L(I-\bar\eta_s\widehat G).
\]
Accordingly define
\[
    \widehat{\mathcal B}_L(B_*)
    :=
    B_*\prod_{t=1}^L(I-\bar\eta_t\widehat G)
\]
and
\[
    \widehat{\mathcal V}_L(\widehat E)
    :=
    \sum_{t=1}^L
    \bar\eta_t\widehat E
    \prod_{s=t+1}^L(I-\bar\eta_s\widehat G).
\]
Then
\[
    B_L-B_*
    =
    -\widehat{\mathcal B}_L(B_*)+
    \widehat{\mathcal V}_L(\widehat E).
\]
Consequently, taking expectation over the sampled training trajectories
(and hence over \(\widehat G,\widehat C\) and \(B_L\), conditional on the
sketch),
\[
    \mathbb E_{\mathcal D}
    [\mathcal L_M(B_L)-\mathcal L_M(B_*)\mid S]
    =
    \operatorname{Bias}_L+
    \operatorname{Var}_L+
    \operatorname{Cross}_L,
\]
where
\[
    \operatorname{Bias}_L
    :=
    \frac12\mathbb E_{\mathcal D}
    \left[
    \|\widehat{\mathcal B}_L(B_*)\|_G^2
    \mid S
    \right],
\]
\[
    \operatorname{Var}_L
    :=
    \frac12\mathbb E_{\mathcal D}
    \left[
    \|\widehat{\mathcal V}_L(\widehat E)\|_G^2
    \mid S
    \right],
\]
and
\[
    \operatorname{Cross}_L
    :=
    -\mathbb E_{\mathcal D}\left[
    \left\langle
    \widehat{\mathcal B}_L(B_*),
    \widehat{\mathcal V}_L(\widehat E)
    \right\rangle_G
    \mid S
    \right].
\]
By Cauchy--Schwarz and \(2ab\le a^2+b^2\),
\[
    |\operatorname{Cross}_L|
    \le
    2\sqrt{\operatorname{Bias}_L\operatorname{Var}_L}
    \le
    \operatorname{Bias}_L+\operatorname{Var}_L.
\]

\paragraph{Covariance events used in proofs.}
Some auxiliary comparisons use the original full-spectrum event
\[
    \mathcal E_{\mathrm{cov}}(R)
    :=
    \left\{
    \left\|
    G^{-1/2}\widehat G G^{-1/2}-I
    \right\|
    \le c_{\mathrm{rel}}R^{-1/2}
    \right\},
\]
where \(c_{\mathrm{rel}}>0\) is a sufficiently small absolute constant. This
event is used only inside the proofs; Lemma~\ref{lem:rnn-covariance-concentration}
turns the effective-sample condition into this event with high probability.
On the Gaussian-sketch event, let \(\|G\|_{\mathrm{op}}\le C_G\). If
\(\eta_0\le(4C_G)^{-1}\), then on \(\mathcal E_{\mathrm{cov}}(R)\),
\[
    \widehat G\preceq2G,
    \qquad
    \eta_{\max}\le\frac{1}{2\|\widehat G\|_{\mathrm{op}}},
    \qquad
    a_{\mathrm{safe}}=1.
\]
Thus all good-event filter calculations may use the nominal schedule
\(\eta_t\), while the safeguarded schedule is stable for every sample.
For the main bias and variance bounds in both regimes, set
\(\lambda=R^{-1}\) and use instead
\[
    \mathcal E_\eta(\lambda)
    :=\left\{
    \left\|(G+\lambda I)^{-1/2}
      (\widehat G-G)(G+\lambda I)^{-1/2}\right\|_{\mathrm{op}}
    \le\eta
    \right\}.
\]
Lemma~\ref{lem:rnn-regularized-covariance-concentration} controls this event
at the optimization scale without requiring accurate relative covariance
estimation below \(R^{-1}\). Lemma~\ref{lem:safeguarded-stability-moments}
shows that the safeguard is inactive on either good event and remains globally
contractive on their complements.

\subsection{Proof of Proposition~\ref{prop:main-decomposition}}
\label{app:proof-main-decomposition}

\begin{proof}
\textbf{(i) Exact decomposition.}
The independent stationary evaluation state and innovation give, for every
\(B\in\mathbb R^{M\times M}\),
\[
    \mathcal L_{M,\infty}(B)
    =
    \frac12\operatorname{Tr}(S\Sigma_\xi S^\top)
    +
    \frac12\|(BS-SA_*)\Sigma_\infty^{1/2}\|_F^2.
\]
The stationary normal equation
\[
    (B_\infty^*S-SA_*)\Sigma_\infty S^\top=0
\]
implies the exact quadratic expansion
\[
    \mathcal L_{M,\infty}(B)
    =
    \frac12\operatorname{Tr}(S\Sigma_\xi S^\top)
    +\mathrm{Approx}
    +\frac12\|B-B_\infty^*\|_{G_\infty}^2,
\]
where
\[
    \mathrm{Approx}
    =\frac12\|(B_\infty^*S-SA_*)\Sigma_\infty^{1/2}\|_F^2.
\]
This is the stationary-evaluation expansion; the training oracle \(B^*\)
need not equal \(B_\infty^*\).

The empirical GD recursion remains trained on \(\widehat G,\widehat C\).
Subtracting the training oracle \(B^*\) and iterating from \(B_0=0\)
gives the unchanged identity
\[
    B_L-B^*
    =
    -\widehat{\mathcal B}_L(B^*)
    +\widehat{\mathcal V}_L(\widehat E).
\]
Consequently,
\[
    B_L-B_\infty^*
    =
    (B^*-B_\infty^*)
    -\widehat{\mathcal B}_L(B^*)
    +\widehat{\mathcal V}_L(\widehat E).
\]
Expanding the squared \(G_\infty\)-norm, taking the conditional
expectation over the training data, and using the definitions of
\(\mathrm{Shift}\), \(\mathrm{Bias}\), \(\mathrm{Var}\),
\(\mathrm{Cross}_{\mathrm{BV}}\), and \(\mathrm{Cross}_{\mathrm{S}}\)
produces exactly Proposition~\ref{prop:main-decomposition}(i).

\noindent\textbf{(ii) Total stationary excess risk.}
Assume the hypotheses of Theorem~\ref{thm:main-scaling-law} and fix a sketch
in the joint event used by its component lemmas; this event has probability
at least \(1-\exp(-\Omega(M))\). All expectations remain over the full
training-data distribution conditional on this sketch.
Appendix~\ref{app:population-distribution-shift}, together with
Theorems~\ref{thm:approx-specific-regime-one} and
\ref{thm:approx-specific}, gives
\[
    0\le\mathrm{Shift}\lesssim\mathrm{Approx}.
\]
For the two possibly negative cross terms, Cauchy--Schwarz gives
\[
\begin{aligned}
    |\mathrm{Cross}_{\mathrm{BV}}|
    &\le2\sqrt{\mathrm{Bias}\,\mathrm{Var}},\\
    |\mathrm{Cross}_{\mathrm{S}}|
    &\le2\sqrt{\mathrm{Shift}}
       \left(\sqrt{\mathrm{Bias}}+\sqrt{\mathrm{Var}}\right).
\end{aligned}
\]
Part (i) and \(2ab\le a^2+b^2\) therefore imply
\[
    \mathcal E\lesssim
    \mathrm{Approx}+\mathrm{Shift}+\mathrm{Bias}+\mathrm{Var}
    \lesssim\mathrm{Approx}+\mathrm{Bias}+\mathrm{Var}.
\]
These inequalities alone do not give a lower bound, since the cross terms
can cancel positive components.

For the lower bound, Appendix~\ref{app:total-risk-comparison} proves the
population-filter surrogate comparison
\eqref{eq:population-surrogate-risk} and the full-data perturbation
bound~\eqref{eq:full-surrogate-transfer}, using only the component and
filter-transfer lemmas. In particular, with \(D=G_\infty\),
\[
\begin{aligned}
    \mathbb E\|\widetilde B_L-B_\infty^*\|_D^2
    &\ge c_0(\mathrm{Bias}^{\mathrm{pop}}+\mathrm{Var}^{\mathrm{pop}}),\\
    \mathbb E\|B_L-\widetilde B_L\|_D^2
    &\le\varepsilon_0
      (\mathrm{Bias}^{\mathrm{pop}}+\mathrm{Var}^{\mathrm{pop}}),
\end{aligned}
\]
where \(\varepsilon_0<c_0/4\). The reverse triangle inequality in the
square-integrable \(D\)-norm and the exact stationary expansion from (i) yield
\[
    \mathcal E
    \ge\mathrm{Approx}
      +\tfrac12(\sqrt{c_0}-\sqrt{\varepsilon_0})^2
       (\mathrm{Bias}^{\mathrm{pop}}+\mathrm{Var}^{\mathrm{pop}}).
\]
The component lemmas establish
\(\mathrm{Bias}\asymp\mathrm{Bias}^{\mathrm{pop}}\) and
\(\mathrm{Var}\asymp\mathrm{Var}^{\mathrm{pop}}\), completing the matching
lower bound. Thus neither cross term is assumed nonnegative or discarded,
and no bias-dominance assumption is used.
\end{proof}

\subsection{Proof of Theorem~\ref{thm:main-scaling-law}}
\label{app:proof-main-scaling-law}

For both regimes use the explicit scale
\(k=R^{1/\alpha}+(R/P)^{1/\theta}\) from the main theorem.
The implicit cutoff \(\kappa_{R,P}\) used in the component lemmas solves
\[
    R(\kappa_{R,P}^{-\alpha}+P^{-1}\kappa_{R,P}^{-\theta})=1.
\]
To compare them, put \(m=\max\{R^{1/\alpha},(R/P)^{1/\theta}\}\).
Monotonicity of the left-hand side gives
\[
    m\le\kappa_{R,P}\le 2^{1/\min\{\alpha,\theta\}}m,
    \qquad m\le k\le2m.
\]
Thus \(k\asymp\kappa_{R,P}\) and
\(R(k^{-\alpha}+P^{-1}k^{-\theta})\asymp1\).
All sample conditions and rate comparisons below are unchanged up to fixed
constants by this replacement; an exact cutoff equality is used only for
\(\kappa_{R,P}\). Put
\[
    \delta=(\alpha-\theta)_+,\qquad
    \zeta=\alpha-\delta,\qquad b=\alpha+2r-1.
\]
The following are proof abbreviations:
\[
    s_{k,P}=1+\frac{k^\delta}{P},\quad
    \chi_{M,k}=\left(\frac Mk\right)^{1-2r},\quad
    \omega_{0,M,P}=\min\{1,M^\delta/P\},\quad
    \widetilde C_\rho=(1-\rho)^{-2}.
\]
For \(R\gtrsim1\), \(k\gtrsim1\) and
\(R\asymp k^\alpha/s_{k,P}\). In Regime I,
\(k\asymp R^{1/\alpha}\), \(s_{k,P}\asymp1\), and
\(\omega_{0,M,P}=1/P\). Let
\(\tau=C_{\mathrm{conf}}\log(M+R+2)\), with a sufficiently large fixed
\(C_{\mathrm{conf}}>4\alpha+4\).

\begin{proof}
Fix a sketch in the intersection of the events from
Lemmas~\ref{lem:lin-sketch-transfer} and
\ref{lem:safeguarded-stability-moments}, the joint sketch events used in
the bias and variance theorems, and the appropriate stationary approximation
theorem, Theorem~\ref{thm:approx-specific-regime-one} or
Theorem~\ref{thm:approx-specific}. This intersection has probability at
least \(1-\exp(-\Omega(M))\). All expectations below are over the full
training-data distribution conditional on this sketch.

\paragraph{Common sample condition and component rates.}
The primitive unsaturation condition implies \(k\le cM\) for a prescribed
sufficiently small fixed \(c>0\), on decreasing \(c_{\mathrm{unsat}}\).
Since \(P\omega_{0,M,P}=\min\{P,M^\delta\}\), the sufficient sample
condition~\eqref{eq:unified-sample-condition} is equivalent, up to its
fixed constants, to
\begin{equation}
\label{eq:unified-regularized-smallness}
    q_\tau s_{k,P}\chi_{M,k}\log^2(e+R)\le c_\star,
    \qquad
    q_\tau:=\widetilde C_\rho(k+\tau)
       \left(\frac1{NP}+\frac{\omega_{0,M,P}}N\right).
\end{equation}
Choose \(c_\star>0\) sufficiently small. With
\(\eta=C_{\mathrm{sp}}(\sqrt{q_\tau}+q_\tau)\), Lemma
\ref{lem:rnn-regularized-covariance-concentration} gives
\[
    \mathbb P(\mathcal E_\eta(R^{-1})^c\mid S)\le2e^{-\tau},
    \qquad
    \eta^2s_{k,P}\chi_{M,k}\log^2(e+R)\le Cc_\star.
\]
The constant \(C_{\mathrm{sp}}\) depends only on fixed primitive data.
In particular \(\eta\le1/2\), and
Lemma~\ref{lem:safeguarded-stability-moments} makes the safeguard inactive
on this event. In Regime I \(q_\tau\asymp
\widetilde C_\rho(k+\tau)/(NP)\), which gives exactly the simplified
sample condition in the theorem.

Theorems~\ref{thm:approx-specific-regime-one} and
\ref{thm:approx-specific} give \(\mathrm{Approx}\asymp M^{-b}\).
Theorems~\ref{thm:bias-specific-regime-one},
\ref{thm:bias-specific}, \ref{thm:variance-specific-regime-one}, and
\ref{thm:variance-specific} give the full-data component comparisons
\[
    \mathrm{Bias}\asymp
    \mathrm{Bias}^{\mathrm{pop}}\asymp k^{-b},
    \qquad
    \mathrm{Var}\asymp
    \mathrm{Var}^{\mathrm{pop}}\asymp\frac{k}{NPs_{k,P}}.
\]
These do not retain an unabsorbed covariance-event remainder in either regime.
In Regime II, substituting \(k\asymp R^{1/\alpha}\) below the crossover
\(P^{\alpha/(\alpha-\theta)}\) and \(k\asymp(R/P)^{1/\theta}\)
above it gives the piecewise formulas. In the second branch,
\(k/(NPs_{k,P})\asymp k^{1-\delta}/N\), not
\(k^{1-\delta}/(NP)\).

\end{proof}

\subsection{Cross-term control for Proposition~\ref{prop:main-decomposition}(ii)}
\label{app:total-risk-comparison}

\begin{proof}[Surrogate comparison]
Retain the notation, joint sketch event, regularized covariance event, and
component comparisons established in Appendix~\ref{app:proof-main-scaling-law}.
Those comparisons follow directly from the component lemmas and do not use
the total-risk claim of Proposition~\ref{prop:main-decomposition}(ii).
All expectations below are over the full training-data distribution
conditional on the fixed sketch.

\paragraph{Shift control.}
Appendix~\ref{app:population-distribution-shift} gives
\(\mathrm{Shift}\lesssim\mathrm{Approx}\asymp M^{-b}\).
Together with \(\mathrm{Bias}^{\mathrm{pop}}\asymp k^{-b}\) and
sufficiently strict unsaturation, this bounds the deterministic shift
interaction as follows.

\paragraph{Population-filter surrogate.}
Write \(D=G_\infty\), \(J=B^*-B_\infty^*\), and
\[
    Q=\widehat C-B^*\widehat G,\qquad
    U=B^*\psi_L(G),\qquad Z=Qg_L(G),\qquad
    \widetilde B_L=B^*-U+Z.
\]
The population filters are deterministic conditional on \(S\).
The averaged normal equation gives
\[
    \mathbb E[Q\mid S]=C-B^*G=0,
    \qquad \mathbb E[Z\mid S]=0.
\]
Consequently, without conditioning on a covariance event,
\begin{equation}
\label{eq:population-surrogate-risk}
    \mathbb E\|\widetilde B_L-B_\infty^*\|_D^2
    =\|J-U\|_D^2+\mathbb E\|Z\|_D^2.
\end{equation}
Since
\[
    \frac{\|J\|_D^2}{\|U\|_D^2}
    =\frac{\mathrm{Shift}}{\mathrm{Bias}^{\mathrm{pop}}}
    \lesssim (k/M)^b,
\]
strict unsaturation can be chosen so that
\(\|J\|_D\le\|U\|_D/4\). Thus, for fixed constants \(c_0,C_0>0\),
\[
    c_0(\mathrm{Bias}^{\mathrm{pop}}+\mathrm{Var}^{\mathrm{pop}})
    \le \mathbb E\|\widetilde B_L-B_\infty^*\|_D^2
    \le C_0(\mathrm{Bias}^{\mathrm{pop}}+\mathrm{Var}^{\mathrm{pop}}).
\]
No assumption on the ratio of Variance to Bias is used here.

\paragraph{Full-data actual--surrogate distance.}
Let
\[
    \widehat\psi_L=\prod_{t=1}^L(I-\bar\eta_t\widehat G),
    \qquad
    \widehat g_L=\sum_{t=1}^L\bar\eta_t
            \prod_{s=t+1}^L(I-\bar\eta_s\widehat G)
\]
be the actual safeguarded filters. For every sample, including one on
which clipping occurs, telescoping gives
\(\widehat G\widehat g_L=I-\widehat\psi_L\), hence
\[
    B_L-\widetilde B_L
    =-B^*(\widehat\psi_L-\psi_L(G))
       +Q(\widehat g_L-g_L(G)).
\]
Only on \(\mathcal E_\eta\) do we replace these safeguarded filters by
the nominal filters evaluated at \(\widehat G\).
Lemmas~\ref{lem:regularized-bias-filter-transfer} and
\ref{lem:regularized-response-filter-transfer}, together with
\eqref{eq:stationary-source-frobenius} and the score covariance bound,
give
\[
\begin{aligned}
    \mathbb E\!\left[
      \|B_L-\widetilde B_L\|_D^2\mathbf1_{\mathcal E_\eta}\mid S\right]
    &\lesssim
      \eta^2\log^2(e+R)
      \left(\frac{M^{1-2r}}R+\frac{k}{NP}\right)\\
    &\lesssim
      \eta^2s_{k,P}\log^2(e+R)
      \left(\chi_{M,k}\mathrm{Bias}^{\mathrm{pop}}
                      +\mathrm{Var}^{\mathrm{pop}}\right).
\end{aligned}
\]
Here the response-filter lemma is pathwise for the sample-dependent \(Q\);
it does not assume independence between \(Q\) and \(\widehat G\).
In taking expectation we used
\[
    \mathbb E\|Q(G+R^{-1}I)^{-1/2}\|_F^2
    \lesssim \frac1{NP}
       \operatorname{Tr}[G(G+R^{-1}I)^{-1}]
    \asymp \frac{k}{NP}.
\]

On \(\mathcal E_\eta^c\), both the safeguarded product
\(\widehat\psi_L\) and the deterministic population product are contractions,
and both response filters have operator norm at most \(R\). Moreover,
\(\|D\|_{\mathrm{op}}\lesssim1\), \(\operatorname{Tr}G\lesssim1\),
and the score bound gives \(\mathbb E\|Q\|_F^2\lesssim(NP)^{-1}\).
Lemma~\ref{lem:score-chaos-fourth-moment} and Cauchy--Schwarz therefore yield
\[
\begin{aligned}
    \mathbb E\!\left[
      \|B_L-\widetilde B_L\|_D^2\mathbf1_{\mathcal E_\eta^c}\mid S\right]
    &\lesssim
      M^{1-2r}e^{-\tau}
        +R^2(\mathbb E\|Q\|_F^4)^{1/2}e^{-\tau/2}\\
    &\lesssim M^{1-2r}e^{-\tau}
        +\frac{R^2}{NP}e^{-\tau/2}.
\end{aligned}
\]
This bound uses the actual safeguarded filters on the complement, not a
nominal empirical-filter identity there. The ratios of these two terms to
population bias and population variance are bounded respectively by
\[
    C M^\alpha(M+R+2)^{-C_{\mathrm{conf}}},
    \qquad
    C M^{2\alpha}(M+R+2)^{-C_{\mathrm{conf}}/2}.
\]
Indeed \(k\lesssim M\), \(R\lesssim M^\alpha\), and
\(s_{k,P}/k\lesssim1\), since \(0\le\delta\le2r<1\).
Increasing the fixed \(C_{\mathrm{conf}}>4\alpha+4\) makes both ratios
arbitrarily small uniformly over the stated range.
Finally, \(\chi_{M,k}\ge1\) and
\eqref{eq:unified-regularized-smallness} control the good-event terms.
Choosing the sample-size constant after \(C_{\mathrm{conf}}\) gives
\begin{equation}
\label{eq:full-surrogate-transfer}
    \mathbb E\|B_L-\widetilde B_L\|_D^2
    \le\varepsilon_0
       (\mathrm{Bias}^{\mathrm{pop}}+\mathrm{Var}^{\mathrm{pop}})
\end{equation}
with any sufficiently small fixed \(\varepsilon_0>0\).

\paragraph{Matching total risk.}
Apply the triangle and reverse triangle inequalities in the Hilbert space
with norm \((\mathbb E\|\,\cdot\,\|_D^2)^{1/2}\).
For \(\varepsilon_0<c_0/4\),
\eqref{eq:population-surrogate-risk}--\eqref{eq:full-surrogate-transfer}
imply
\[
    \mathbb E\|B_L-B_\infty^*\|_D^2
    \asymp \mathrm{Bias}^{\mathrm{pop}}+\mathrm{Var}^{\mathrm{pop}}.
\]
The exact stationary expansion from
Proposition~\ref{prop:main-decomposition}(i) and the component comparisons give
\[
\begin{aligned}
    \mathcal E
    &=\mathrm{Approx}+\tfrac12\mathbb E\|B_L-B_\infty^*\|_D^2\\
    &\asymp\mathrm{Approx}+\mathrm{Bias}+\mathrm{Var}\\
    &\asymp M^{-b}+k^{-b}+\frac{k}{NPs_{k,P}}.
\end{aligned}
\]
Thus the interactions are controlled throughout the unified range, including
allocations for which Bias and Variance have comparable order.
\end{proof}

\subsection{Proof of Corollary~\ref{cor:compute-optimal-allocation}}
\label{app:proof-compute-optimal-allocation}

\begin{proof}
Write \(\mathcal D=NP\), \(\mathcal K=\eta_0\mathcal C\).
Since \(R=\sum_t\eta_t\le\eta_0L\), every admissible schedule obeys
\(\mathcal D\le\mathcal T\) and \(\mathcal D M^2R\le\mathcal K\).
All comparisons are on the theorem's sketch events, with constants uniform
over admissible allocations and fixed primitive data.

\paragraph{Lower bound over all admissible allocations.}
Strict unsaturation gives \(M\gtrsim k\). Since \(P\ge2\),
\(s_{k,P}\lesssim k^\delta\), so
\(R\asymp k^\alpha/s_{k,P}\gtrsim k^\zeta\).
Proposition~\ref{prop:main-decomposition}(ii), combined with the
component rates of Theorem~\ref{thm:main-scaling-law}, therefore implies
\[
\begin{aligned}
    \mathcal E
    &\gtrsim k^{-b}+\frac{k^{1-\delta}}{\mathcal D}\\
    &\gtrsim k^{-b}
       +\max\left\{\frac{k^{1-\delta}}{\mathcal T},
                    \frac{k^{\alpha+3-2\delta}}{\mathcal K}\right\}.
\end{aligned}
\]
The first variance term balances \(k^{-b}\) at
\(k_{\mathcal T}=\mathcal T^{1/(\zeta+2r)}\), and the second at
\(k_{\mathcal K}=\mathcal K^{1/(2\zeta+2r+2)}\).
Set \(k_0=\min\{k_{\mathcal T},k_{\mathcal K}\}\).
For \(k\le k_0\), Bias alone is at least a constant times \(k_0^{-b}\).
For \(k>k_0\), the variance term corresponding to the binding resource
is at least that order: both its exponents are positive because
\(\delta<1\) and \(\alpha+3-2\delta>0\).
This proves the lower bound uniformly over all \(P\ge2\) in the
theorem-controlled range, not only fixed-\(P\) allocations.

\paragraph{Achievability and sample-condition slack.}
Fix \(P=2\). For large target \(k\), choose
\[
    M\asymp\Lambda k,\qquad
    R\asymp k^\zeta,\qquad
    \mathcal D\asymp k^{\zeta+2r},
\]
where \(\Lambda\) is a sufficiently large fixed constant for strict
unsaturation. The cutoff equation holds up to fixed constants in both regimes,
and \(s_{k,P}\asymp k^\delta\). Thus
\[
    M^{-b}\lesssim k^{-b},\qquad
    \frac{k}{\mathcal D s_{k,P}}\asymp k^{-b},\qquad
    \mathcal D M^2R\asymp k^{2\zeta+2r+2}.
\]
The unified sample condition reduces to
\(\mathcal D\gtrsim k^{1+\delta}\log^2(e+k)\). Its polynomial slack is
strictly positive:
\[
    \zeta+2r-(1+\delta)
    =
    \begin{cases}
        \alpha+2r-1>0,&\theta\ge\alpha,\\
        (\theta-1)+(2r-\delta)>0,&\theta<\alpha.
    \end{cases}
\]
The second inequality uses \(\theta>1\) and \(\delta\le2r\).
Consequently the construction satisfies the required concentration and
transfer constants for all sufficiently large \(k\), with no logarithmic loss
in the resulting optimal rate. Choose \(k\) to be a sufficiently small fixed
multiple of \(k_0\), after fixing \(\Lambda\) and the allocation constants,
to respect both resource caps. Rounding \(M,N,L\) to integers changes only
constants; the ambient dimension must permit \(d\ge C_{\mathrm{amb}}M\).
Choose a constant stable learning rate \(\eta_t=\eta_*\), where
\(\eta_*/\eta_0>0\) is fixed, so \(L\asymp R/\eta_0\). Rounding \(L\)
and adjusting constants preserves the resource caps. This schedule attains
\(\mathcal E\lesssim k_0^{-b}\); a fixed schedule with substantially smaller
average step size need not attain this compute order.

Equating \(k_{\mathcal T}\) and \(k_{\mathcal K}\) gives
\(\mathcal T_{\mathrm{crit}}\asymp
\mathcal K^{(\zeta+2r)/(2\zeta+2r+2)}\).
Above this scale the construction uses only the needed tokens; below it,
it uses order \(\mathcal T\) tokens and may leave compute unused.
This proves the stated allocation and risk orders.

\paragraph{Separate component-wise feasibility frontier.}
For comparison, take \(M=\Theta(k)\) with sufficiently large width constant,
but do not fix \(P\). The common sufficient sample threshold is, up to fixed
constants,
\[
    \mathcal D_{\mathrm{suff}}
    \asymp k(1+k^\delta/P)(1+\min\{P,k^\delta\})\log^2(e+k).
\]
The budget needed by this construction is also subject to \(\mathcal D\ge P\),
since \(N\ge1\); choose \(\mathcal D\) of order
\(\max\{P,\mathcal D_{\mathrm{suff}}\}\).
As \(R\asymp k^\alpha/(1+k^\delta/P)\), a sufficient compute order is
\[
    \max\left\{
       \frac{Pk^{\alpha+2}}{1+k^\delta/P},
       k^{\alpha+3}(1+\min\{P,k^\delta\})\log^2(e+k)
    \right\}.
\]
For fixed \(P\), this gives the main-text sufficient order
\(k^{\alpha+3}\log^2(e+k)\). This construction only establishes
feasibility under the sufficient sample condition; it is not the
Bias--Variance-balanced risk optimum proved above, nor a necessary
resource bound outside the theorem-controlled range.
\end{proof}


\section{Approximation error}
\label{app:approximation}

Retain the training covariance
\[
    H:=\frac1P\sum_{p=0}^{P-1}\Sigma_p.
\]
For stationary evaluation, define
\[
    D:=\Sigma_\infty
    =\sum_{j=1}^d d_je_je_j^\top,
    \qquad
    d_j:=\frac{\sigma_{\xi,j}}{1-a_j^2}\asymp j^{-\alpha}.
\]
For a generic positive diagonal covariance
\(T=\sum_{j=1}^dt_je_je_j^\top\), define
\[
    \mathsf{App}(T)
    :=
    \frac12\inf_{B\in\mathbb R^{M\times M}}
    \|(BS-SA_*)T^{1/2}\|_F^2.
\]
The stationary approximation component and the auxiliary training-weighted
approximation quantity are, respectively,
\[
    \boxed{\mathrm{Approx}:=\mathsf{App}(D)},
    \qquad
    \boxed{\mathrm{Approx}_{\mathrm{tr}}:=\mathsf{App}(H)}.
\]
Only \(\mathrm{Approx}\) is the stationary approximation component in the
main theorem; \(\mathrm{Approx}_{\mathrm{tr}}\) is used below only as an
auxiliary quantity in the distribution-shift proof.

For a generic \(T\), write
\[
    K_T:=A_*T^{1/2},
    \qquad
    \widetilde K_T:=SK_T,
    \qquad
    \tau_{T,j}:=\|K_Te_j\|_2^2=a_j^2t_j.
\]
In particular, for \(T=D\), \(\tau_{D,j}\asymp j^{-\beta_\alpha}\); for
\(T=H\), the training covariance, \(t_j=h_{P,j}\asymp
j^{-\alpha}+P^{-1}j^{-\theta}\) and
\(\tau_{H,j}\asymp j^{-\beta_\alpha}+P^{-1}j^{-\beta_\theta}\).
For an index set \(I\), let \(\Pi_I\) be the coordinate projection onto
\(\operatorname{span}\{e_j:j\in I\}\), and write
\(\Pi_{\le m}:=\Pi_{\{1,\ldots,m\}}\) and
\(\Pi_{>m}:=I-\Pi_{\le m}\).

\subsection{General approximation upper bound}
\label{app:approx-general-upper}

\begin{lemma}[General approximation upper bound]
\label{lem:approx-general-upper}
Suppose \(T=\sum_jt_je_je_j^\top\) is positive diagonal, with eigenvalues
nonincreasing, adjacent-comparable and doubling on the relevant index range,
and satisfying
\(\sum_{j>m}t_j\lesssim mt_{m+1}\) and
\(\sum_{j>m}t_j^2\lesssim mt_{m+1}^2\). These conditions hold for the
decreasing comparison covariances \(F_\alpha\) and \(F_P\) defined in
Appendix~\ref{app:comparison-envelopes}; exact monotonicity of \(D\) or \(H\)
is not required. There is a sufficiently small absolute
constant \(c_0>0\) such that, for every integer \(m\le c_0M\), with
probability at least \(1-\exp(-cM)\) over the Gaussian sketch,
\[
    \mathsf{App}(T)
    \lesssim
    \sum_{j=m+1}^{d}\tau_{T,j}
    +
    t_{m+1}
    \sum_{j\le m}\frac{\tau_{T,j}}{t_j}.
\]
\end{lemma}

\begin{proof}
Let
\[
    G_T:=STS^\top,
    \qquad
    U:=G_T^{-1/2}ST^{1/2},
    \qquad
    P_S:=U^\top U,
    \qquad
    Q_S:=I-P_S.
\]
Because \(UU^\top=I_M\), \(P_S\) is the orthogonal projection onto the row
space of \(ST^{1/2}\). Projecting the rows of \(SA_*T^{1/2}\) onto this row
space gives
\[
    \mathsf{App}(T)
    =
    \frac12\|SK_TQ_S\|_F^2.
\]
Splitting at \(m\) and using that \(Q_S\) is a contraction,
\[
    \|SK_TQ_S\|_F^2
    \le
    2\|SK_T\Pi_{\le m}Q_S\|_F^2
    +
    2\|SK_T\Pi_{>m}\|_F^2.
\]
Gaussian quadratic-form concentration gives
\[
    \|SK_T\Pi_{>m}\|_F^2
    \lesssim
    \sum_{j=m+1}^{d}\tau_{T,j}
\]
with probability at least \(1-\exp(-cM)\).

For the head term, let \(T_0:=\Pi_{\le m}T\Pi_{\le m}\). Then
\[
\begin{aligned}
    \|SK_T\Pi_{\le m}Q_S\|_F^2
    \le{}&
    \|SK_T\Pi_{\le m}T_0^{-1/2}\|_F^2
    \|T_0^{1/2}\Pi_{\le m}Q_S\|^2, \\
    \|SK_T\Pi_{\le m}T_0^{-1/2}\|_F^2
    \lesssim{}&
    \sum_{j\le m}\frac{\tau_{T,j}}{t_j}.
\end{aligned}
\]
It remains to control the projection leakage. Write
\(S=(S_0,S_1)\), \(T=T_0\oplus T_1\), and
\(\Sigma_1:=S_1T_1S_1^\top\). When \(m\le c_0M\), standard Gaussian
singular-value concentration gives \(S_0^\top S_0\succeq cI_m\). Moreover,
the eigenvalue sequence satisfies
\[
    \sum_{j>m}t_j\lesssim mt_{m+1},
    \qquad
    t_m\asymp t_{m+1},
\]
by the assumed regularity and integral comparison. A weighted chi-square net argument
therefore yields
\(\|\Sigma_1\|\lesssim t_{m+1}\). The Woodbury identity now gives
\[
\begin{aligned}
    \|T_0^{1/2}\Pi_{\le m}Q_S\|^2
    &=
    \left\|
    \left(T_0^{-1}+S_0^\top\Sigma_1^{-1}S_0\right)^{-1}
    \right\| \\
    &\lesssim
    t_{m+1}.
\end{aligned}
\]
Combining the head and tail estimates proves the result.
\end{proof}

\paragraph{Transfer from envelopes to the actual covariances.}
For every fixed \(B\), the squared residual is linear and monotone in its
covariance:
\[
    2\mathsf{App}(T)
    =\inf_B\operatorname{Tr}[(BS-SA_*)T(BS-SA_*)^\top].
\]
Consequently, \(\mathsf{App}\) is positively homogeneous and
\[
    cF\preceq T\preceq CF
    \quad\Longrightarrow\quad
    c\,\mathsf{App}(F)\le\mathsf{App}(T)
       \le C\,\mathsf{App}(F).
\]
We apply the upper lemma to \(F_\alpha\) or \(F_P\), and use this
comparison for \(D\) or \(H\). For \(q>1\), integral comparison gives
\[
    \sup_{j>m}j^{-q}\le(m+1)^{-q},\qquad
    \sum_{j>m}j^{-q}\lesssim m(m+1)^{-q},\qquad
    \sum_{j>m}j^{-2q}\lesssim m(m+1)^{-2q}.
\]
The same bounds hold for \(f_P\), uniformly in \(P\), by adding its two
components and using \((x+y)^2\le2x^2+2y^2\). Thus all tail, adjacent-index,
and fixed-scale comparisons used above are valid for these envelopes.
The Gaussian events controlling their tails and weighted traces can be
chosen from the two primitive components, independently of \(P\).

\subsection{General approximation lower bound}
\label{app:approx-general-lower}

\begin{lemma}[General approximation lower bound]
\label{lem:approx-general-lower}
Suppose \(T=\sum_jt_je_je_j^\top\) is positive diagonal, with eigenvalues
comparable on fixed-scale blocks, and Assumptions~\ref{assump:stable-rnn} and
\ref{assump:power-law-source} hold. With probability at least
\(1-\exp(-cM)\) over the Gaussian sketch, for every
\(m\) satisfying \(c_1M\le m\le c_2M\),
\[
    \mathsf{App}(T)
    \gtrsim
    \sum_{j=m+1}^{2m}\tau_{T,j}.
\]
\end{lemma}

\begin{proof}
Write \(s_j:=Se_j\). Then
\[
    2\mathsf{App}(T)
    =
    \inf_B
    \sum_{j=1}^{d}t_j\|(B-a_jI)s_j\|_2^2.
\]
Fix \(m\asymp M\), let \(n=\lceil2M\rceil\), and choose a sufficiently
large fixed \(K\). Consider the blocks
\[
    I_1:=\{m+1,\ldots,m+n\},
    \qquad
    I_2:=\{\lceil Km\rceil+1,\ldots,\lceil Km\rceil+n\}.
\]
The fixed constant \(C_{\mathrm{amb}}\) is chosen larger than the endpoint
multiplier of \(I_2\), so both blocks lie in \(\{1,\ldots,d\}\).
Because \(a_j^2\asymp j^{-2r}\) with \(r>0\), \(K\) can be chosen so that
the teacher eigenvalues on \(I_2\) are a sufficiently small fixed fraction
of those on \(I_1\). The eigenvalue sequence has fixed-scale regularity, so
all covariance weights on \(I_1\cup I_2\) are bounded below by a constant
multiple of \(t_m\).

On the Gaussian event
\[
    cI_M
    \preceq
    \sum_{j\in I_\ell}s_js_j^\top
    \preceq
    CI_M,
    \qquad \ell=1,2,
\]
the first block forces either a squared fitting error of order
\(Mm^{-2r}\) or a matrix norm \(\|B\|_F^2\) of that order. In the latter
case, the second block forces the same squared error because its teacher
eigenvalues are smaller by the chosen fixed factor. Consequently,
\[
    \inf_B
    \sum_{j\in I_1\cup I_2}\|(B-a_jI)s_j\|_2^2
    \gtrsim
    Mm^{-2r}.
\]
Thus
\[
    \mathsf{App}(T)
    \gtrsim
    m t_m m^{-2r}.
\]
Since \(\tau_{T,j}=t_ja_j^2\) and both factors are comparable on
\((m,2m]\), the last expression is comparable to
\(\sum_{j=m+1}^{2m}\tau_{T,j}\). The Gaussian block event holds uniformly
over admissible \(m\) with probability at least \(1-\exp(-cM)\).
\end{proof}

\subsection{Specific approximation bound for
\texorpdfstring{\(\theta\ge\alpha\)}{theta >= alpha}}
\label{app:approx-specific-regime-one}

\begin{theorem}[Approximation bound for \(\theta\ge\alpha\)]
\label{thm:approx-specific-regime-one}
If \(\theta\ge\alpha\), then, with probability at least
\(1-\exp(-cM)\) over the Gaussian sketch,
\[
    \boxed{\mathrm{Approx}\asymp M^{1-\beta_\alpha}}.
\]
The boundary case \(\theta=\alpha\) is included. This is the approximation
error under the stationary evaluation covariance \(D\), not under the
initialization covariance.
\end{theorem}

\begin{proof}
The envelope comparison gives
\(\mathsf{App}(D)\asymp\mathsf{App}(F_\alpha)\).
Apply the general lemmas to \(T=F_\alpha\), whose source weights
\(\tau_{F_\alpha,j}=a_j^2j^{-\alpha}\asymp j^{-\beta_\alpha}\).
Choose \(m\asymp M\),
with a sufficiently small fixed multiplier for the upper bound. The finite
tail has the same order as the infinite power-law tail because it contains
the fixed annulus \((m,2m]\). The general upper bound gives
\[
    \sum_{j>m}\tau_{F_\alpha,j}
    \asymp m^{1-\beta_\alpha},
    \qquad
    (m+1)^{-\alpha}\sum_{j\le m}a_j^2
    \asymp m^{-\alpha}\sum_{j\le m}a_j^2
    \asymp m^{-\alpha}m^{1-2r}
    =m^{1-\beta_\alpha}.
\]
For the lower bound, the general lower bound on the annulus gives
\[
    \mathrm{Approx}
    \gtrsim
    \sum_{m<j\le2m}\tau_{F_\alpha,j}
    \asymp m^{1-\beta_\alpha}.
\]
Since \(m\asymp M\), the two bounds prove the result.
\end{proof}

\subsection{Specific approximation bound for
\texorpdfstring{\(\alpha-2r\le\theta<\alpha\)}{alpha - 2r <= theta < alpha}}
\label{app:approx-specific-regime-two}

\begin{theorem}[Stationary approximation bound for heavier initialization]
\label{thm:approx-specific}
If \(\alpha-2r\le\theta<\alpha\), then, with probability at least
\(1-\exp(-cM)\) over the Gaussian sketch,
\[
    \boxed{\mathrm{Approx}\asymp M^{1-\beta_\alpha}}.
\]
The stationary approximation depends on \(D\), not on the initialization-
dependent training covariance \(H\). The former two-scale expression remains
the training-weighted upper bound for the auxiliary quantity
\(\mathrm{Approx}_{\mathrm{tr}}\), and will be used in B.5.
\end{theorem}

\begin{proof}
Apply Lemmas~\ref{lem:approx-general-upper} and
\ref{lem:approx-general-lower} to \(F_\alpha\), and transfer the bounds
through \(cF_\alpha\preceq D\preceq CF_\alpha\), exactly as in the proof of
Theorem~\ref{thm:approx-specific-regime-one}.
Since \(D\) is independent of \(P\) and of \(\Sigma_0\), the same head and
tail estimates give \(\mathrm{Approx}\asymp M^{1-\beta_\alpha}\). For
comparison, apply the upper lemma to \(F_P\). Its head and tail terms are
bounded by
\[
    \sum_{j>m}a_j^2f_P(j)
    +f_P(m+1)\sum_{j\le m}a_j^2
    \lesssim m^{1-\beta_\alpha}+P^{-1}m^{1-\beta_\theta}.
\]
The comparison \(H\preceq CF_P\), with \(m\asymp M\), gives the auxiliary
training-weighted estimate
\[
    \boxed{
    \mathrm{Approx}_{\mathrm{tr}}
    \lesssim
    M^{1-\beta_\alpha}+P^{-1}M^{1-\beta_\theta}.
    }
\]
This training-weighted quantity is distinct from the stationary component
proved in the theorem.
\end{proof}

\subsection{Population Distribution Shift}
\label{app:population-distribution-shift}

Define the stationary and training covariances and their population
minimizers by
\[
    G_\infty:=SDS^\top,
    \qquad
    G:=SHS^\top,
    \qquad
    B_\infty^*:=SA_*DS^\top G_\infty^{-1},
    \qquad
    B^*:=SA_*HS^\top G^{-1}.
\]
The population distribution-shift component is
\[
    \boxed{
    \mathrm{Shift}:=\frac12\|B^*-B_\infty^*\|_{G_\infty}^2.
    }
\]
Let
\[
    \Delta:=H-D,
    \qquad
    E_\infty:=SA_*-B_\infty^*S.
\]
The stationary normal equation gives
\[
    E_\infty DS^\top=0.
\]
Consequently, the difference of the training and stationary population
minimizers has the exact representation
\[
\begin{aligned}
    B^*-B_\infty^*
    &=(SA_*-B_\infty^*S)HS^\top G^{-1}\\
    &=E_\infty\Delta S^\top G^{-1},
\end{aligned}
\]
and therefore
\[
    \boxed{
    2\mathrm{Shift}
    =\|E_\infty\Delta S^\top G^{-1}G_\infty^{1/2}\|_F^2.
    }
\]
This exact oracle-difference identity is the starting point for both regimes.
Also,
\[
    \|E_\infty D^{1/2}\|_F^2=2\mathrm{Approx}.
\]

\paragraph{Uniform Gaussian-sketch event.}
Fix \(m=\lfloor c_0M\rfloor\), with a sufficiently small fixed \(c_0>0\),
and write \(s_j:=Se_j\), \(S_0=[s_1,\ldots,s_m]\). The Gaussian
head--tail concentration arguments in Lemmas~\ref{lem:approx-general-upper}
and~\ref{lem:approx-general-lower} give a common sketch event of probability
at least \(1-\exp(-cM)\) on which
\[
\begin{aligned}
    S_0^\top S_0\succeq cI_m,
    &\qquad
    \sum_{j\le m}a_j^2\|s_j\|_2^2
    \lesssim
    \sum_{j\le m}a_j^2,\\
    \|S_{>m}D_{>m}S_{>m}^\top\|_{\mathrm{op}}\lesssim M^{-\alpha},
    &\qquad
    \|S_{>m}\Sigma_{0,>m}S_{>m}^\top\|_{\mathrm{op}}\lesssim M^{-\theta}.
\end{aligned}
\]
and
\[
    cM^{-\alpha}I
    \preceq
    G_{\infty,1}:=S_{>m}D_{>m}S_{>m}^\top
    \preceq
    CM^{-\alpha}I,
    \qquad
    \sum_{j>m}t_j a_j^2\|s_j\|_2^2
    \lesssim
    \sum_{j>m}t_j a_j^2
\]
for \(T=D\), \(T=\Sigma_0\), and their sum \(T=W:=D+\Sigma_0\),
whenever the corresponding expression is used below. To see these bounds,
for a power-law tail \(t_j\asymp j^{-q}\), \(q>1\), its expected sketched
covariance is \(M^{-1}\sum_{j>m}t_j I\), of order \(M^{-q}I\); a weighted
chi-square net bound adds at most the same order. A block of order \(M\)
columns with indices of order \(M\) gives the matching lower eigenvalue
bound. The scalar weighted-trace estimates follow from Gaussian quadratic-form
concentration. A finite intersection of these events still has failure
probability \(\exp(-cM)\). Since
\(h_{P,j}\asymp d_j+P^{-1}\sigma_{0,j}\), all resulting bounds for the
training covariance hold simultaneously and uniformly over \(P\ge2\); no
union bound over trajectory lengths is needed. Each indexed tail is bounded
using the corresponding restriction of \(F_\alpha\) or \(F_\theta\);
none of these statements requires reordering the primitive coordinates.

\paragraph{Regime I: \(\theta\ge\alpha\).}
For each coordinate, the exact position-averaged covariance identity is
\[
    h_{P,j}-d_j
    =
    (\sigma_{0,j}-d_j)
    \frac{1-a_j^{2P}}{P(1-a_j^2)}.
\]
When \(\theta\ge\alpha\), \(\sigma_{0,j}\lesssim d_j\), so
\[
    |\Delta|\preceq CP^{-1}D,
    \qquad
    \|D^{-1/2}\Delta D^{-1/2}\|_{\mathrm{op}}\lesssim P^{-1}.
\]
The innovation contribution to \(H\) yields \(H\succeq cD\): at every
position \(p\ge1\), the covariance contains at least \(\Sigma_\xi\), so
averaging over \(P\ge2\) positions contributes at least
\((P-1)\Sigma_\xi/P\), which is comparable to \(D\). Hence
\(G\succeq cG_\infty\) and
\[
    \|D^{1/2}S^\top G^{-1}G_\infty^{1/2}\|_{\mathrm{op}}\lesssim1.
\]
Factoring the exact oracle-difference identity gives
\[
\begin{aligned}
    2\mathrm{Shift}
    &\lesssim
    P^{-2}\|E_\infty D^{1/2}\|_F^2\\
    &=2P^{-2}\mathrm{Approx}.
\end{aligned}
\]
By Theorem~\ref{thm:approx-specific-regime-one},
\[
    \boxed{
    \mathrm{Shift}\lesssim P^{-2}M^{1-\beta_\alpha}.
    }
\]
No matching lower bound is claimed. For example, if \(\Sigma_0=D\), then
the positionwise covariance is stationary, \(H=D\), and
\(\mathrm{Shift}=0\).

\paragraph{Regime II: \(\alpha-2r\le\theta<\alpha\).}
Let
\[
    \delta:=\alpha-\theta>0,
    \qquad
    x_M:=\frac{M^\delta}{P}.
\]
We prove the saturation and perturbative estimates separately.

\emph{Saturation estimate.}
First we show, uniformly over \(P\ge2\), that
\[
    \mathrm{Shift}\lesssim M^{1-\beta_\alpha}.
\]
For \(w_j:=h_{P,j}\), the training normal equation implies
\[
    \frac12\|(B^*S-SA_*)D^{1/2}\|_F^2
    =\mathrm{Approx}+\mathrm{Shift}.
\]
It suffices to prove
\[
    \|(B^*S-SA_*)D^{1/2}\|_F^2
    \lesssim M^{1-\beta_\alpha}.
\]
Write
\[
    F_j:=(B^*-a_jI)s_j,
    \qquad
    Z:=\sum_{j>m}w_jF_js_j^\top,
    \qquad
    V:=ZS_0(S_0^\top S_0)^{-1}.
\]
The training normal equation is \(\sum_{j=1}^dw_jF_js_j^\top=0\), so
\[
    F_j=-w_j^{-1}Vq_j,\qquad j\le m,
\]
where \(q_j\) is the \(j\)-th standard basis vector in \(\mathbb R^m\).
The mixed-spectrum estimate gives
\[
    w_j\asymp j^{-\alpha}(1+j^\delta/P),
    \qquad
    \mathrm{Approx}_{\mathrm{tr}}
    \lesssim M^{1-\beta_\alpha}(1+x_M).
\]
For \(j>m\), \(d_j/w_j\lesssim(1+x_M)^{-1}\). Therefore
\[
\begin{aligned}
    \sum_{j>m}d_j\|F_j\|_2^2
    &\lesssim
    \frac{1}{1+x_M}\sum_{j>m}w_j\|F_j\|_2^2\\
    &\lesssim
    \frac{\mathrm{Approx}_{\mathrm{tr}}}{1+x_M}
    \lesssim M^{1-\beta_\alpha}.
\end{aligned}
\]

For the head, the assumptions imply \(\alpha>2\delta\). Indeed,
\(\delta\le2r<1\); if \(\alpha\ge2\), then \(2\delta<2\le\alpha\),
while if \(1<\alpha<2\), the standing condition \(\theta>1\) gives
\(\delta<\alpha-1\), hence \(2\delta<2\alpha-2<\alpha\). Thus
\(j^\alpha/(1+j^\delta/P)^2\) is increasing, and
\[
    \max_{j\le m}\frac{d_j}{w_j^2}
    \lesssim
    \frac{M^\alpha}{(1+x_M)^2}.
\]
By the matrix Cauchy--Schwarz inequality and the uniform Gaussian event,
\[
\begin{aligned}
    \|Z\|_F^2
    &\le
    \left(\sum_{j>m}w_j\|F_j\|_2^2\right)
    \left\|\sum_{j>m}w_js_js_j^\top\right\|_{\mathrm{op}}\\
    &\lesssim
    \mathrm{Approx}_{\mathrm{tr}}M^{-\alpha}(1+x_M)\\
    &\lesssim
    M^{1-\beta_\alpha-\alpha}(1+x_M)^2.
\end{aligned}
\]
Here the weighted-tail operator bound follows by adding the \(D\) and
\(P^{-1}\Sigma_0\) tail bounds in the common sketch event. Since
\(\|V\|_F\lesssim\|Z\|_F\), the head-residual identity yields
\[
\begin{aligned}
    \sum_{j\le m}d_j\|F_j\|_2^2
    &\le
    \max_{j\le m}\frac{d_j}{w_j^2}\|V\|_F^2\\
    &\lesssim M^{1-\beta_\alpha}.
\end{aligned}
\]
Combining head and tail proves the saturation estimate.

\emph{Perturbative estimate.}
Set \(W:=D+\Sigma_0\). In Regime II,
\(W_{jj}\asymp j^{-\theta}\), and the exact covariance identity gives
\[
    |\Delta|\preceq CP^{-1}W,
    \qquad
    \|W^{-1/2}\Delta W^{-1/2}\|_{\mathrm{op}}\lesssim P^{-1}.
\]
The common sketch event also gives
\[
    \boxed{SWS^\top\preceq CM^\delta G_\infty.}
\]
Indeed, for the head \(j\le m\), \(W_{jj}/d_j\lesssim M^\delta\), so its
sketched covariance is bounded by \(CM^\delta G_\infty\). For the tail,
\(\|S_{>m}W_{>m}S_{>m}^\top\|_{\mathrm{op}}\lesssim M^{-\theta}\),
whereas the retained order-\(M\) block gives
\(G_\infty\succeq cM^{-\alpha}I\); their ratio is \(M^\delta\).

We next establish the weighted stationary-oracle residual bound
\[
    \|E_\infty W^{1/2}\|_F^2
    \lesssim M^{1-\beta_\theta}.
\]
For completeness, first bound the Frobenius norm of the stationary oracle.
Write \(G_1:=G_{\infty,1}\), and split
\[
    B_\infty^{*\top}
    =G_\infty^{-1}SA_*DS^\top
    =G_\infty^{-1}S_0D_0A_0S_0^\top
    +G_\infty^{-1}S_{>m}D_{>m}A_{>m}S_{>m}^\top.
\]
Since \(cM^{-\alpha}I\preceq G_1\preceq CM^{-\alpha}I\),
\(\|G_\infty^{-1}G_1\|_{\mathrm{op}}\lesssim1\). The Woodbury identity
and \(S_0^\top S_0\succeq cI\) give
\[
    G_\infty^{-1}S_0D_0
    =G_1^{-1}S_0
    (D_0^{-1}+S_0^\top G_1^{-1}S_0)^{-1},
    \qquad
    \|G_\infty^{-1}S_0D_0\|_{\mathrm{op}}\lesssim1.
\]
Gaussian weighted-trace concentration gives
\[
    \|A_0S_0^\top\|_F^2\lesssim\sum_{j\le m}a_j^2
    \lesssim M^{1-2r},
    \qquad
    \|D_{>m}^{1/2}A_{>m}S_{>m}^\top\|_F^2
    \lesssim\sum_{j>m}d_ja_j^2
    \lesssim M^{1-\beta_\alpha}.
\]
Also,
\(\|G_1^{-1}S_{>m}D_{>m}^{1/2}\|_{\mathrm{op}}^2
=\|G_1^{-1}\|_{\mathrm{op}}\lesssim M^\alpha\). Consequently, the head
and tail terms above imply
\[
    \boxed{\|B_\infty^*\|_F^2\lesssim M^{1-2r}.}
\]

Let \(E_j:=E_\infty e_j=a_js_j-B_\infty^*s_j\). The stationary normal
equation gives \(\sum_jd_jE_js_j^\top=0\). Define
\[
    Z_D:=\sum_{j>m}d_jE_js_j^\top,
    \qquad
    V_D:=Z_DS_0(S_0^\top S_0)^{-1}.
\]
As before, \(E_j=-d_j^{-1}V_Dq_j\) for \(j\le m\), and
\(\|V_D\|_F\lesssim\|Z_D\|_F\). The stationary approximation theorem,
the weighted-tail event, and matrix Cauchy--Schwarz give
\[
    \|Z_D\|_F^2
    \lesssim
    \|E_\infty D^{1/2}\|_F^2
    \|G_1\|_{\mathrm{op}}
    \lesssim
    M^{1-\beta_\alpha-\alpha}.
\]
Since \(W_{jj}/d_j^2\lesssim M^{2\alpha-\theta}\) for \(j\le m\),
\[
    \sum_{j\le m}W_{jj}\|E_j\|_2^2
    \lesssim
    M^{2\alpha-\theta}\|Z_D\|_F^2
    \lesssim M^{1-\beta_\theta}.
\]
For the tail, use \(E_j=a_js_j-B_\infty^*s_j\), the weighted-trace event,
and the weighted-tail operator bound:
\[
\begin{aligned}
    \sum_{j>m}W_{jj}\|E_j\|_2^2
    &\lesssim
    \sum_{j>m}W_{jj}a_j^2\|s_j\|_2^2
    +\|B_\infty^*\|_F^2
    \left\|\sum_{j>m}W_{jj}s_js_j^\top\right\|_{\mathrm{op}}\\
    &\lesssim
    M^{1-\beta_\theta}+M^{1-2r}M^{-\theta}\\
    &\lesssim M^{1-\beta_\theta}.
\end{aligned}
\]
This proves the stated weighted stationary-oracle residual bound, including
the auxiliary Frobenius estimate used in its tail control.

Returning to the exact oracle-difference identity, factor it as
\[
    E_\infty\Delta S^\top G^{-1}G_\infty^{1/2}
    =
    E_\infty W^{1/2}
    (W^{-1/2}\Delta W^{-1/2})
    (W^{1/2}S^\top G^{-1}G_\infty^{1/2}).
\]
Since \(H\succeq cD\) for \(P\ge2\), \(G\succeq cG_\infty\). The
sketch comparison above then gives
\[
    \|W^{1/2}S^\top G^{-1}G_\infty^{1/2}\|_{\mathrm{op}}^2
    \lesssim M^\delta.
\]
Thus
\[
\begin{aligned}
    2\mathrm{Shift}
    &\lesssim
    P^{-2}M^\delta\|E_\infty W^{1/2}\|_F^2\\
    &\lesssim
    P^{-2}M^{1-\beta_\theta+\delta}\\
    &=P^{-2}M^{1-\beta_\alpha+2\delta}.
\end{aligned}
\]
Combining this perturbative estimate with the saturation estimate yields
\[
    \boxed{
    \mathrm{Shift}
    \lesssim
    M^{1-\beta_\alpha}
    \min\left\{1,\frac{M^{2\delta}}{P^2}\right\}.
    }
\]
Because \(\min\{1,x^2\}\asymp(x/(1+x))^2\) for \(x>0\), this is
equivalently
\[
    \boxed{
    \mathrm{Shift}
    \lesssim
    M^{1-\beta_\alpha}
    \left(\frac{M^\delta}{P+M^\delta}\right)^2.
    }
\]
No matching lower bound for \(\mathrm{Shift}\) is asserted.

\paragraph{Remark (Shift is dominated by stationary approximation).}
In both regimes, \(\mathrm{Approx}\asymp M^{1-\beta_\alpha}\). The Regime-I
bound gives \(\mathrm{Shift}\lesssim P^{-2}M^{1-\beta_\alpha}
\lesssim\mathrm{Approx}\). The Regime-II bound gives
\[
    \mathrm{Shift}
    \lesssim
    M^{1-\beta_\alpha}
    \left(
    \frac{M^{\alpha-\theta}}{P+M^{\alpha-\theta}}
    \right)^2
    \lesssim\mathrm{Approx}.
\]
Consequently,
\[
    \boxed{
    \mathrm{Approx}+\mathrm{Shift}
    \asymp\mathrm{Approx}
    \asymp M^{1-\beta_\alpha}.
    }
\]
Thus the shift does not change the polynomial order of the combined
representation and population-mismatch error. It need not be negligible in
every asymptotic regime: in Regime II it may be of the same order as
\(\mathrm{Approx}\) when \(P\lesssim M^{\alpha-\theta}\). A matching lower
bound for the shift is unnecessary, since \(\mathrm{Approx}\) already
provides the matching lower bound for the combined quantity.


\section{Bias bound}
\label{app:bias}

Throughout this section, let
\[
    G:=SHS^\top=\sum_{j=1}^M\mu_jv_jv_j^\top,
    \qquad
    B^*:=CG^{-1},
    \qquad
    D:=G_\infty=S\Sigma_\infty S^\top.
\]
Write \(\|B\|_T^2:=\operatorname{Tr}(BTB^\top)\) for a positive
semidefinite covariance \(T\), and set
\[
    q_j:=\mu_j\|B^*v_j\|_2^2.
\]
The empirical GD product and its population counterpart are
\[
    \widehat\Psi_L:=\prod_{t=1}^L(I-\bar\eta_t\widehat G),
    \qquad
    \psi_L(s):=\prod_{t=1}^L(1-\eta_ts).
\]
On the relative covariance event
\[
    \mathcal E_{\mathrm{cov}}(R)
    :=\left\{
    \left\|G^{-1/2}\widehat G G^{-1/2}-I\right\|_{\mathrm{op}}
    \le c_{\mathrm{rel}}R^{-1/2}
    \right\},
\]
the safeguard is inactive and \(\widehat\Psi_L=\psi_L(\widehat G)\).
For the main results in both regimes we instead use the optimization-scale
regularized event
\[
    \mathcal E_\eta(R^{-1})
    :=\left\{
    \left\|(G+R^{-1}I)^{-1/2}
      (\widehat G-G)(G+R^{-1}I)^{-1/2}\right\|_{\mathrm{op}}
    \le\eta
    \right\}.
\]
The original relative event is retained only for auxiliary comparisons under
the stronger, original sample conditions.
Define the pointwise stationary-evaluation bias and its conditional
expectation by
\[
    b_\infty(\mathcal D):=\frac12\|B^*\widehat\Psi_L\|_D^2,
    \qquad
    \mathrm{Bias}:=\mathbb E_{\mathcal D}
    [b_\infty(\mathcal D)\mid S].
\]
The deterministic population-filter bias is the auxiliary quantity
\[
    \mathrm{Bias}^{\mathrm{pop}}
    :=\frac12\|B^*\psi_L(G)\|_D^2,
\]
whereas the training-norm auxiliary quantity is
\[
    \mathrm{Bias}_{\mathrm{tr}}
    :=\frac12\mathbb E_{\mathcal D}
    [\|B^*\widehat\Psi_L\|_G^2\mid S].
\]
Only \(\mathrm{Bias}\) is the stationary-evaluation component in the main
theorem.

We use the joint sketch event of Lemma~\ref{lem:lin-sketch-transfer},
extended by the following stationary Gaussian-sketch bounds, simultaneously
for \(1\le m\le M\):
\[
    \|S_{>m}\Sigma_{\infty,>m}S_{>m}^\top\|_{\mathrm{op}}
    \lesssim m^{-\alpha},
    \qquad
    \lambda_j(SA_*^2\Sigma_\infty S^\top)
    \asymp j^{-\beta_\alpha},\quad 1\le j\le M.
\]
The first follows by sandwiching each original indexed tail against
\(F_{\alpha,>m}\). For \(m\le M\), integral comparison gives
\(M^{-1}\sum_{j>m}d_j\lesssim M^{-1}m^{1-\alpha}
\le m^{-\alpha}\); the weighted Gaussian-net fluctuation is bounded by
the same order. The second follows by sandwiching
\(A_*^2\Sigma_\infty\) against \(F_{\beta_\alpha}\) on the common
envelope event in Lemma~\ref{lem:lin-sketch-transfer}. A union bound over the
\(O(M)\) cutoffs preserves probability at least \(1-\exp(-cM)\). All these
events and constants are uniform over \(P\ge2\). We also use the same
Gaussian head--tail event for the training covariance,
\(\|S_{>m}H_{>m}S_{>m}^\top\|_{\mathrm{op}}\lesssim h_{P,m}\), and
\(S_0^\top S_0\succeq cI\) whenever the head width is at most a fixed
small fraction of \(M\); these follow from the mixed-spectrum tail and
rectangular Gaussian concentration in the proof of
Lemma~\ref{lem:lin-sketch-transfer}. To compare covariances, write
\(h_{P,j}=f_{P,j}\sigma_{0,j}+(1-f_{P,j})d_j\), where
\(f_{P,j}=P^{-1}\sum_{q=0}^{P-1}a_j^{2q}\le(1+\rho^2)/2\) for \(P\ge2\).
Thus \(H\succeq c\Sigma_\infty\) in both regimes, so \(D\preceq CG\).
When \(\theta\ge\alpha\), also \(\sigma_{0,j}\lesssim d_j\), giving
\(H\preceq C\Sigma_\infty\) and hence \(cG\preceq D\preceq CG\).

\subsection{General GD-bias upper bound}
\label{app:bias-general-upper}

\begin{lemma}[General GD-bias upper bound]
\label{lem:bias-general-upper}
On \(\mathcal E_{\mathrm{cov}}(R)\), the original training-norm bound remains
valid:
\[
    \|B^*\widehat\Psi_L\|_G^2
    \lesssim
    \sum_{j=1}^M q_j\min\{1,(R\mu_j)^{-1}\}.
\]
This is an auxiliary bound for \(\mathrm{Bias}_{\mathrm{tr}}\), not the
stationary-evaluation bias bound in Regime II.
\end{lemma}

\begin{proof}
Set \(\lambda=R^{-1}\). Lemma~\ref{lem:lin-gd-filters} gives the
noncommutative comparison
\[
    \widehat\Psi_LG\widehat\Psi_L
    \preceq C\lambda G(G+\lambda I)^{-1}.
\]
Taking the trace after congruence by \(B^*\) yields
\[
\begin{aligned}
    \|B^*\widehat\Psi_L\|_G^2
    &\lesssim
    \operatorname{Tr}\bigl[B^*\lambda G(G+\lambda I)^{-1}B^{*\top}\bigr]\\
    &=\sum_{j=1}^M q_j\frac{\lambda}{\mu_j+\lambda}
    \asymp\sum_{j=1}^M q_j\min\{1,(R\mu_j)^{-1}\}.
\end{aligned}
\]
\end{proof}

\begin{lemma}[Stationary covariance localization]
\label{lem:stationary-bias-localization}
On the joint sketch event, simultaneously for \(1\le m\le M\),
\[
    \Pi_{\ge m}D\Pi_{\ge m}
    \preceq Cm^{-\alpha}\Pi_{\ge m},
    \qquad
    \Pi_{\ge m}:=\sum_{j=m}^Mv_jv_j^\top.
\]
\end{lemma}

\begin{proof}
In Regime I, \(D\preceq CG\) and
\(\Pi_{\ge m}G\Pi_{\ge m}\preceq\mu_m\Pi_{\ge m}
\preceq Cm^{-\alpha}\Pi_{\ge m}\), by
Lemma~\ref{lem:lin-sketch-transfer}.

In Regime II, write \(\delta=\alpha-\theta>0\), and first consider
\(m\) larger than a fixed constant. Let
\(h=\lfloor c_0m\rfloor\), \(S=(S_0,S_1)\), and
\(H=H_0\oplus H_1\) in the teacher eigenbasis, where \(S_0\) contains
the first \(h\) sketch columns. The uniform Gaussian event gives
\[
    S_0^\top S_0\succeq cI,
    \qquad
    \|T\|_{\mathrm{op}}\lesssim h_{P,h},
    \qquad T:=S_1H_1S_1^\top.
\]
The spectral transfer and mixed-spectrum estimate give
\(\|G\Pi_{\ge m}\|_{\mathrm{op}}=\mu_m\lesssim h_{P,m}
\lesssim h_{P,h}\). Hence
\[
\begin{aligned}
    Y:=H_0S_0^\top\Pi_{\ge m}
    &=(S_0^\top S_0)^{-1}S_0^\top(G-T)\Pi_{\ge m},\\
    \|Y\|_{\mathrm{op}}&\lesssim h_{P,h}.
\end{aligned}
\]
Since \(H_0\) and the stationary covariance block
\(D_0=\operatorname{diag}(d_1,\ldots,d_h)\) are diagonal,
\[
\begin{aligned}
    \|\Pi_{\ge m}S_0D_0S_0^\top\Pi_{\ge m}\|_{\mathrm{op}}
    &\le \max_{j\le h}\frac{d_j}{h_{P,j}^2}\,\|Y\|_{\mathrm{op}}^2\\
    &\lesssim
    \frac{h^\alpha}{(1+h^\delta/P)^2}
    h^{-2\alpha}(1+h^\delta/P)^2
    \lesssim h^{-\alpha}.
\end{aligned}
\]
Here \(h_{P,j}\asymp j^{-\alpha}(1+j^\delta/P)\), and
\(j^\alpha/(1+j^\delta/P)^2\) is increasing because the standing
assumptions imply \(\alpha>2\delta\): indeed \(\delta\le2r<1\); if
\(\alpha\ge2\), then \(2\delta<2\le\alpha\), while if
\(1<\alpha<2\), then \(\theta>1\) gives
\(2\delta<2(\alpha-1)<\alpha\). The tail block obeys
\[
    \|S_1D_1S_1^\top\|_{\mathrm{op}}\lesssim h^{-\alpha}
\]
by the uniform stationary weighted-tail event. Since \(h\asymp m\),
combining the head and tail proves the claim. The finitely many smaller \(m\)
are covered by \(\|D\|_{\mathrm{op}}\lesssim1\), after adjusting constants.
\end{proof}

\begin{lemma}[Aggregate source energy on spectral blocks]
\label{lem:stationary-bias-source-block}
Let \(I_m=\{m,\ldots,\min(2m-1,M)\}\) and
\(\Pi_{I_m}=\sum_{j\in I_m}v_jv_j^\top\). On the joint sketch event,
uniformly over \(P\ge2\) and dyadic \(m\le M\),
\[
    \|B^*\Pi_{I_m}\|_F^2\lesssim m^{1-2r}.
\]
\end{lemma}

\begin{proof}
For \(m\le M/3\), set
\(R_m=\max\{1,h_{P,m}^{-1}\}\), so
\(R_mh_{P,m}\asymp1\). For \(j\in I_m\),
\(\mu_j\asymp h_{P,m}\), and therefore
\(\min\{1,(R_m\mu_j)^{-1}\}\asymp1\). Applying the upper estimate of
Lemma~\ref{lem:aggregate-bias-transfer} with cutoff \(m\),
\[
\begin{aligned}
    \sum_{j\in I_m}q_j
    &\lesssim R_m^{-1}m^{1-2r}
      +m^{1-\beta_\alpha}+P^{-1}m^{1-\beta_\theta}\\
    &\lesssim h_{P,m}m^{1-2r}.
\end{aligned}
\]
Dividing by \(\mu_j\gtrsim h_{P,m}\) proves the result for these blocks.
For blocks meeting \([M/3,M]\), apply the same aggregate bound with cutoff
\(\lfloor M/3\rfloor\) and \(R_M=\max\{1,h_{P,M}^{-1}\}\).
The lower spectral bound \(\mu_M\gtrsim h_{P,M}\) makes the filter weight
bounded below on the top spectral range. The resulting estimate is
\(\sum_{j\in I_m}\|B^*v_j\|_2^2\lesssim M^{1-2r}\), which is comparable
to \(m^{1-2r}\) for these blocks. No pointwise source-alignment estimate is
used.
\end{proof}

Summing Lemma~\ref{lem:stationary-bias-source-block} over dyadic blocks gives
the source Frobenius estimate
\begin{equation}
\label{eq:stationary-source-frobenius}
    \boxed{\|B^*\|_F^2\lesssim M^{1-2r}.}
\end{equation}
Indeed, \(1-2r>0\), so the dyadic sum is dominated by its largest block.

\begin{lemma}[Population-filter stationary bias upper bound]
\label{lem:stationary-population-bias-upper}
Let \(k=\kappa_{R,P}\), so \(Rh_{P,k}\asymp1\), with the Regime-I
interpretation \(k\asymp R^{1/\alpha}\). On the joint sketch event,
\[
    \mathrm{Bias}^{\mathrm{pop}}\lesssim k^{1-\beta_\alpha}.
\]
\end{lemma}

\begin{proof}
Use dyadic blocks \(I_m=\{m,\ldots,\min(2m-1,M)\}\). By the triangle
inequality, Lemmas~\ref{lem:stationary-bias-localization} and
\ref{lem:stationary-bias-source-block}, and the scalar GD filter bound,
\[
\begin{aligned}
    \|B^*\psi_L(G)D^{1/2}\|_F
    &\le\sum_{m\ \mathrm{dyadic}}
       \|B^*\psi_L(G)\Pi_{I_m}D^{1/2}\|_F\\
    &\lesssim\sum_{m\ \mathrm{dyadic}}
       m^{(1-\beta_\alpha)/2}e^{-cRh_{P,m}}.
\end{aligned}
\]
Indeed, \(\Pi_{I_m}D\Pi_{I_m}\preceq Cm^{-\alpha}\Pi_{I_m}\),
\(\|B^*\Pi_{I_m}\|_F^2\lesssim m^{1-2r}\), and
\(\sup_{j\in I_m}|\psi_L(\mu_j)|\le e^{-cRh_{P,m}}\).
For \(m\le k\), mixed-spectrum regularity gives
\(Rh_{P,m}\gtrsim(k/m)^\theta\) in Regime II (and exponent \(\alpha\)
in Regime I); for \(m\ge k\), the power factor
\(m^{(1-\beta_\alpha)/2}\) is dyadically summable. Splitting at \(k\)
and summing the two geometric tails yields
\[
    \sum_{m\ \mathrm{dyadic}}
       m^{(1-\beta_\alpha)/2}e^{-cRh_{P,m}}
    \lesssim k^{(1-\beta_\alpha)/2}.
\]
Squaring and including the factor \(1/2\) proves the result.
\end{proof}

\subsection{General GD-bias lower bound}
\label{app:bias-general-lower}

\begin{lemma}[Low-rank stationary spectral barrier]
\label{lem:stationary-low-rank-barrier}
For \(J\le c_*M\), let
\(\mathcal T_J(Q):=\inf_{\operatorname{rank}(Z)\le J}\|Q-Z\|_F\).
If \(K_\infty:=SA_*\Sigma_\infty^{1/2}\), then on the joint sketch event
\[
    \mathcal T_J(B^*S\Sigma_\infty^{1/2})
    \gtrsim J^{(1-\beta_\alpha)/2},
\]
provided \(J/M\) is smaller than a sufficiently small fixed constant.
Moreover, for every \(J\) and every rank-\(J\) matrix \(F_J\),
\[
    \|Q-F\|_F\ge\mathcal T_J(Q)-\|F-F_J\|_F.
\]
\end{lemma}

\begin{proof}
The last inequality is the triangle inequality followed by the definition of
\(\mathcal T_J\). The stationary sketch event gives
\[
    \lambda_j(K_\infty K_\infty^\top)
    =\lambda_j(SA_*^2\Sigma_\infty S^\top)
    \asymp j^{-\beta_\alpha},\qquad j\le M.
\]
The Eckart--Young theorem therefore implies
\[
    \mathcal T_J(K_\infty)^2
    =\sum_{j>J}\lambda_j(K_\infty K_\infty^\top)
    \gtrsim J^{1-\beta_\alpha}
\]
for \(J\le c_*M\). Stationary least-squares orthogonality and the exact
Appendix~B decomposition give
\[
\begin{aligned}
    \|B^*S\Sigma_\infty^{1/2}-K_\infty\|_F^2
    &=2(\mathrm{Approx}+\mathrm{Shift})\\
    &\lesssim M^{1-\beta_\alpha}.
\end{aligned}
\]
Hence
\[
\mathcal T_J(B^*S\Sigma_\infty^{1/2})
\ge cJ^{(1-\beta_\alpha)/2}
   -CM^{(1-\beta_\alpha)/2}.
\]
Since \(1-\beta_\alpha<0\), choosing \(J/M\) sufficiently small makes
the second term smaller than the first, proving the claim.
\end{proof}

\begin{lemma}[Empirical-to-population stationary filter transfer]
\label{lem:stationary-filter-transfer}
Suppose \(D\preceq CG\), \(\|G\|_{\mathrm{op}}\lesssim1\), and
\(\varepsilon:=\|G^{-1/2}(\widehat G-G)G^{-1/2}\|_{\mathrm{op}}\le1/2\).
On the covariance event, for every matrix \(B\),
\[
\boxed{
    \|B[\psi_L(\widehat G)-\psi_L(G)]\|_D
    \lesssim \varepsilon\log(e+R)\sqrt{\mathcal Q_R(B)},
    \qquad
    \mathcal Q_R(B):=
    \operatorname{Tr}\bigl[BG(I+RG)^{-1}B^\top\bigr].
}
\]
The estimate applies to every admissible deterministic schedule and does not assume that
\(G,\widehat G,D\) commute.
\end{lemma}

\begin{proof}
Write
\(\widehat G-G=G^{1/2}EG^{1/2}\), where \(\|E\|_{\mathrm{op}}=\varepsilon\).
For each step define
\[
    u_t:=\sum_{s<t}\eta_s,
    \qquad v_t:=\sum_{s>t}\eta_s,
    \qquad
    \widehat\Phi_{>t}:=\prod_{s=t+1}^L(I-\eta_s\widehat G),
    \qquad
    \Phi_{<t}:=\prod_{s=1}^{t-1}(I-\eta_sG).
\]
The exact product telescoping identity is
\[
    \psi_L(\widehat G)-\psi_L(G)
    =-\sum_{t=1}^L\eta_t\widehat\Phi_{>t}
      G^{1/2}EG^{1/2}\Phi_{<t}.
\]
All factors are positive contractions on the event. For any contiguous segment
of steps with total step mass \(w\), its scalar product \(\phi(s)\) obeys
\(0\le\phi(s)\le e^{-ws}\). Thus
\[
    \phi(s)^2s\lesssim \frac{s}{1+ws}.
\]
For the suffix empirical filter, functional calculus and congruence give
\[
    \widehat\Phi_{>t}G\widehat\Phi_{>t}
    \preceq (1-\varepsilon)^{-1}
    \widehat\Phi_{>t}\widehat G\widehat\Phi_{>t}
    \preceq C\widehat G(I+v_t\widehat G)^{-1}.
\]
The map \(X\mapsto X(I+wX)^{-1}\) is operator monotone for \(w\ge0\),
so the relative covariance event gives
\[
    \widehat G(I+w\widehat G)^{-1}
    \preceq (1+\varepsilon)G(I+w(1+\varepsilon)G)^{-1}
    \preceq C G(I+wG)^{-1}.
\]
Taking the trace after congruence by \(B\) yields
\[
    \|B\widehat\Phi_{>t}G^{1/2}\|_F
    \lesssim\sqrt{\mathcal Q_{v_t}(B)},
    \qquad
    \mathcal Q_w(B):=\operatorname{Tr}[BG(I+wG)^{-1}B^\top].
\]
For the prefix factor, \(D\preceq CG\) gives
\[
\begin{aligned}
    \|G^{1/2}\Phi_{<t}D^{1/2}\|_{\mathrm{op}}^2
    &\le C\sup_{s\in[0,\|G\|]}s^2\phi_{<t}(s)^2\\
    &\lesssim(1+u_t)^{-2}.
\end{aligned}
\]
For \(0\le w\le R\), the bounded spectrum of \(G\) yields
\[
    \mathcal Q_w(B)
    \lesssim\frac{1+R}{1+w}\mathcal Q_R(B).
\]
Indeed, for \(0<s\le\|G\|_{\mathrm{op}}\lesssim1\),
\(\{s/(1+ws)\}/\{s/(1+Rs)\}=(1+Rs)/(1+ws)\lesssim
(1+R)/(1+w)\); functional calculus then gives the Loewner comparison.
Combining these inequalities with the telescoping identity gives
\[
    \|B[\psi_L(\widehat G)-\psi_L(G)]\|_D
    \lesssim \varepsilon\sqrt{\mathcal Q_R(B)}
    \sum_{t=1}^L\frac{\eta_t}{1+u_t}
       \sqrt{\frac{1+R}{1+v_t}}.
\]
The schedule has \(\sum_t\eta_t=R\),
\(\eta_t\le\eta_0\lesssim1\), and \(u_t+v_t\gtrsim R\) for sufficiently large \(R\), up to
constants (bounded \(R\) is immediate). Split the sum into indices with
\(u_t\le cR\) and the remainder. The first part is bounded by
\(C\sum_t\eta_t/(1+u_t)\lesssim\log(e+R)\): since
\(u_{t+1}=u_t+\eta_t\) and \(\eta_t\lesssim1\), each summand is
bounded by a constant times
\(\log((1+u_{t+1})/(1+u_t))\), and these logarithms telescope. In the
remainder, \(1+u_t\gtrsim1+R\). Reversing the order of summation, and using
\(v_0:=\sum_{t=1}^L\eta_t\), \(v_L:=0\), and
\(v_{t-1}=v_t+\eta_t\), gives the finite-step bound
\[
    \sum_t\frac{\eta_t}{\sqrt{1+v_t}}\lesssim\sqrt{1+R},
\]
because each summand is bounded by a constant times
\(\int_{v_t}^{v_t+\eta_t}(1+x)^{-1/2}\,dx\); the intervals are
adjacent and their union has length at most \(R\). The remainder of the
original sum is therefore bounded by a constant. This verifies the discrete
deterministic schedule sum directly, without a continuous-time
approximation or commuting the covariance matrices.
\end{proof}

\begin{lemma}[General stationary GD-bias lower bound]
\label{lem:bias-general-lower}
Let \(k\asymp\kappa_{R,P}\) and assume
\(k\le c_{\mathrm{unsat}}M\), with the sufficiently small constant chosen
as in the low-rank barrier below. On the joint sketch event,
\[
    \frac12\|B^*\psi_L(G)\|_D^2
    \gtrsim k^{1-\beta_\alpha}.
\]
If additionally
\(\varepsilon\le c_{\mathrm{rel}}R^{-1/2}\), with \(c_{\mathrm{rel}}\)
sufficiently small, then
\[
    b_\infty(\mathcal D)\gtrsim k^{1-\beta_\alpha}.
\]
\end{lemma}

\begin{proof}
First choose a fixed spectral separation factor \(K_0\) sufficiently large
as follows. Apply Lemma~\ref{lem:stationary-low-rank-barrier} with
\(J=\lceil K_0k\rceil\). For a dyadic block at scale \(m\ge J\),
Lemmas~\ref{lem:stationary-bias-localization} and
\ref{lem:stationary-bias-source-block}, and
\(1-\psi_L(s)\le\min\{1,Rs\}\), give
\[
\|B^*(I-\psi_L(G))\Pi_{I_m}S\Sigma_\infty^{1/2}\|_F
\lesssim m^{(1-\beta_\alpha)/2}\min\{1,Rh_{P,m}\}.
\]
For \(m\ge J\), mixed-spectrum regularity gives
\(Rh_{P,m}\lesssim(k/m)^\theta\) in Regime II, and the same bound with
\(\theta\) replaced by \(\alpha\) in Regime I. Dyadic summation therefore
implies
\[
\|B^*(I-\psi_L(G))\Pi_{>J}S\Sigma_\infty^{1/2}\|_F
\lesssim J^{(1-\beta_\alpha)/2}K_0^{-\theta},
\]
where in Regime I the exponent is \(\alpha\).

Set
\[
\begin{aligned}
    Q&:=B^*S\Sigma_\infty^{1/2},
    &F&:=B^*(I-\psi_L(G))S\Sigma_\infty^{1/2},\\
    \Pi_{\le J}&:=\sum_{j=1}^Jv_jv_j^\top,
    &F_J&:=B^*(I-\psi_L(G))\Pi_{\le J}S\Sigma_\infty^{1/2}.
\end{aligned}
\]
Then \(\operatorname{rank}(F_J)\le J\), and
\(Q-F=B^*\psi_L(G)S\Sigma_\infty^{1/2}\). Choose \(K_0\) large enough
that the leakage bound is a strict fraction of the barrier lower bound.
Next choose \(c_{\mathrm{unsat}}\) small enough that
\(J\le c_*M\) and \(J/M\) meets the strict barrier condition. Lemma
\ref{lem:stationary-low-rank-barrier}, the triangle inequality in that lemma,
and the leakage estimate give
\[
    \|B^*\psi_L(G)\|_D
    \ge\|Q-F\|_F
    \ge\mathcal T_J(Q)-\|F-F_J\|_F
    \gtrsim k^{(1-\beta_\alpha)/2}.
\]
Squaring proves the population-filter lower bound.

Finally, Lemma~\ref{lem:stationary-filter-transfer} with \(B=B^*\),
combined with the aggregate estimate in Lemma~\ref{lem:aggregate-bias-transfer},
gives
\[
\|B^*[\psi_L(\widehat G)-\psi_L(G)]\|_D
\lesssim c_{\mathrm{rel}}k^{(1-\beta_\alpha)/2}
\]
uniformly in the admissible parameters. In Regime II we used
\[
    \mathcal Q_R(B^*)\lesssim
    k^{1-\beta_\alpha}(1+k^\delta/P),
    \qquad
    R^{-1}\asymp k^{-\alpha}(1+k^\delta/P),
\]
and \(\alpha>2\delta\) to absorb the factor
\((1+k^\delta/P)k^{-\alpha/2}\log(e+R)\); in Regime I,
\(k\asymp R^{1/\alpha}\) gives the same conclusion. Choose
\(c_{\mathrm{rel}}\) last, sufficiently small relative to the population
lower-bound constant. The reverse triangle inequality transfers the lower
bound to the empirical filter.
\end{proof}

The preceding eventwise bounds yield conditional expectation bounds. On the
complement of \(\mathcal E_{\mathrm{cov}}(R)\), the safeguarded GD product is
contractive, so
\[
\mathbb E_{\mathcal D}
\left[\|B^*\widehat\Psi_L\|_D^2
\mathbf 1_{\mathcal E_{\mathrm{cov}}(R)^c}\mid S\right]
\le \|D\|_{\mathrm{op}}\|B^*\|_F^2
\mathbb P(\mathcal E_{\mathrm{cov}}(R)^c\mid S).
\]
Summing the source-block bounds over dyadic blocks gives
\(\|B^*\|_F^2\lesssim M^{1-2r}\le M\). Under the original sample-size
conditions, the covariance-event complement is exponentially small, so this
contribution is \(O(Me^{-cM})\) for these auxiliary comparisons.
The two specific results below instead use the regularized event and absorb
its complement into their polynomial rates.

\subsection{Specific GD-bias bound for
\texorpdfstring{\(\theta\ge\alpha\)}{theta >= alpha}}
\label{app:bias-specific-regime-one}

\begin{theorem}[Stationary GD-bias bound for \(\theta\ge\alpha\)]
\label{thm:bias-specific-regime-one}
Suppose \(\theta\ge\alpha\), let \(k=\kappa_{R,P}\asymp R^{1/\alpha}\),
and assume \(k\le c_{\mathrm{unsat}}M\) for a sufficiently small constant.
Let \(\tau=C_{\mathrm{conf}}\log(M+R+2)\), with sufficiently large fixed
\(C_{\mathrm{conf}}>4\alpha+4\). If
\[
    NP\gtrsim
    \frac{k+\tau}{(1-\rho)^2}
    \left(\frac Mk\right)^{1-2r}\log^2(e+R)
\]
with a sufficiently large implicit constant, then, with high probability
over the Gaussian sketch,
\[
    \mathrm{Bias}\asymp R^{(1-\beta_\alpha)/\alpha}
\]
for the full conditional expectation over the training trajectories.
The boundary case \(\theta=\alpha\) is included.
\end{theorem}

\begin{proof}
Here \(D\asymp G\), \(R\asymp k^\alpha\).
Lemma~\ref{lem:stationary-population-bias-upper} and the population
part of Lemma~\ref{lem:bias-general-lower} give
\(\mathrm{Bias}^{\mathrm{pop}}\asymp k^{-b}\), where
\(b=\beta_\alpha-1=\alpha+2r-1\).
The Regime-I specialization of
Lemma~\ref{lem:rnn-regularized-covariance-concentration} gives
\(\mathbb P(\mathcal E_\eta(R^{-1})^c\mid S)\le2e^{-\tau}\), with
\[
    q_\tau=2(1-\rho)^{-2}\frac{k+\tau}{NP},
    \qquad \eta=C_{\mathrm{sp}}(\sqrt{q_\tau}+q_\tau).
\]
On this event the safeguard is inactive by
Lemma~\ref{lem:safeguarded-stability-moments}. The regularized bias
transfer and \eqref{eq:stationary-source-frobenius} yield
\[
    \|B^*[\psi_L(\widehat G)-\psi_L(G)]\|_D^2
    \lesssim
    \eta^2\log^2(e+R)\frac{M^{1-2r}}R
    \asymp \eta^2\left(\frac Mk\right)^{1-2r}
        \log^2(e+R)\,k^{-b}.
\]
The sample condition makes this a sufficiently small multiple of
population bias, so triangle and reverse triangle inequalities give
matching good-event rates.
On the complement, global safeguard contraction bounds the contribution
by \(CM^{1-2r}e^{-\tau}\). Relative to \(k^{-b}\) this is at most
\(CM^\alpha(M+R+2)^{-C_{\mathrm{conf}}}\), which is absorbed after
choosing \(C_{\mathrm{conf}}\) large enough. Averaging over all training
data proves the claim without the original full-spectrum condition.
\end{proof}

\subsection{Specific GD-bias bound for
\texorpdfstring{\(\alpha-2r\le\theta<\alpha\)}{alpha - 2r <= theta < alpha}}
\label{app:bias-specific-regime-two}

\begin{theorem}[Stationary GD-bias bound for heavier initialization]
\label{thm:bias-specific}
Suppose \(\alpha-2r\le\theta<\alpha\), and let
\(\delta=\alpha-\theta\). Let \(k=\kappa_{R,P}\) solve
\(R(k^{-\alpha}+P^{-1}k^{-\theta})\asymp1\), and assume the sufficiently
strict unsaturated condition \(k\le c_{\mathrm{unsat}}M\). Define
\[
    s_{k,P}:=1+\frac{k^\delta}{P},
    \qquad
    \chi_{M,k}:=\left(\frac Mk\right)^{1-2r},
    \qquad
    \omega_{0,M,P}:=\min\left\{1,\frac{M^\delta}{P}\right\},
\]
and let \(\tau=C_{\mathrm{conf}}\log(M+R+2)\), where
\(C_{\mathrm{conf}}\) is a sufficiently large fixed constant. If
\[
    \widetilde C_\rho(k+\tau)
    \left(\frac1{NP}+\frac{\omega_{0,M,P}}N\right)
    s_{k,P}\chi_{M,k}\log^2(e+R)
    \le c_\star
\]
for a sufficiently small fixed \(c_\star>0\), depending only on \(\rho\)
and the fixed primitive data, then, with high probability
over the sketch,
\[
    \mathrm{Bias}\asymp k^{1-\beta_\alpha}
\]
for the full conditional expectation over the training trajectories.
Equivalently,
with \(R_P=P^{\alpha/(\alpha-\theta)}\),
\[
    \mathrm{Bias}\asymp
    \begin{cases}
        R^{(1-\beta_\alpha)/\alpha},&R\lesssim R_P,\\[1mm]
        (R/P)^{(1-\beta_\alpha)/\theta},&R\gtrsim R_P.
    \end{cases}
\]
\end{theorem}

\begin{proof}
The population-filter upper bound is
Lemma~\ref{lem:stationary-population-bias-upper}; the low-rank barrier part of
Lemma~\ref{lem:bias-general-lower} gives the matching lower bound. Thus
\[
    \mathrm{Bias}^{\mathrm{pop}}
    \asymp k^{1-\beta_\alpha}.
\]
Put
\[
    q_\tau:=\widetilde C_\rho(k+\tau)
      \left(\frac1{NP}+\frac{\omega_{0,M,P}}N\right),
    \qquad
    \eta:=C_{\mathrm{sp}}(\sqrt{q_\tau}+q_\tau),
\]
where \(\widetilde C_\rho=(1-\rho)^{-2}\) and \(C_{\mathrm{sp}}\) is the
fixed-primitive constant in Lemma~\ref{lem:rnn-regularized-covariance-concentration}.
The assumed sample condition implies \(q_\tau\ll1\) and
\[
    \eta^2s_{k,P}\chi_{M,k}\log^2(e+R)\le c
\]
after decreasing \(c_\star\). Lemma
\ref{lem:rnn-regularized-covariance-concentration} therefore gives
\(\mathbb P_{\mathcal D}(\mathcal E_\eta(R^{-1})^c\mid S)\le2e^{-\tau}\).

On this event, Lemma~\ref{lem:regularized-bias-filter-transfer} and
\eqref{eq:stationary-source-frobenius} give
\[
\begin{aligned}
    \|B^*[\psi_L(\widehat G)-\psi_L(G)]\|_D^2
    &\lesssim
    \eta^2\frac{\log^2(e+R)}R M^{1-2r}\\
    &\asymp
    \eta^2s_{k,P}\chi_{M,k}\log^2(e+R)
       k^{1-\beta_\alpha}.
\end{aligned}
\]
Here \(R^{-1}\asymp k^{-\alpha}s_{k,P}\). The last display is a
sufficiently small multiple of the population bias, so the triangle and
reverse-triangle inequalities give matching eventwise bounds for
\(b_\infty(\mathcal D)\).

On the complement, global safeguard contraction and
\eqref{eq:stationary-source-frobenius} imply
\[
\begin{aligned}
    \mathbb E_{\mathcal D}\!\left[
      \|B^*\widehat\Psi_L\|_D^2
      \mathbf1_{\mathcal E_\eta^c}\,\middle|\,S\right]
    \lesssim M^{1-2r}e^{-\tau}.
\end{aligned}
\]
Strict unsaturation and the cutoff relation give \(R\lesssim M^\alpha\).
In fact, with \(b=\beta_\alpha-1\), its ratio to \(k^{-b}\) is at most
\[
    C M^{1-2r}k^b(M+R+2)^{-C_{\mathrm{conf}}}
    \le C M^\alpha(M+R+2)^{-C_{\mathrm{conf}}}.
\]
Taking \(C_{\mathrm{conf}}>4\alpha+4\), and increasing it if necessary
for the fixed constants, makes this negligible uniformly over the stated
regime. Averaging the eventwise bounds over all training data proves the
claim. Substituting the mixed cutoff on the two sides of
\(R_P=P^{\alpha/(\alpha-\theta)}\) gives the piecewise display.
\end{proof}


\section{Variance bound}
\label{app:variance}

All expectations are over the sampled trajectories conditional on a fixed
Gaussian sketch. Retain the training covariance and population oracle
\[
    G:=SHS^\top=\sum_{j=1}^M\mu_jv_jv_j^\top,
    \qquad
    B_*:=CG^{-1},
\]
and the empirical score
\[
    Q:=\widehat C-B_*\widehat G.
\]
The stationary evaluation covariance, used only in evaluation norms in this
appendix, is
\[
    D:=G_\infty=S\Sigma_\infty S^\top,
    \qquad d_j:=v_j^\top Dv_j.
\]
For a positive semidefinite matrix \(T\), write
\(\|A\|_T^2:=\operatorname{Tr}(ATA^\top)\). The empirical response filter is
\[
    \widehat g_L(\widehat G)
    :=
    \sum_{t=1}^L\bar\eta_t
    \prod_{s=t+1}^L(I-\bar\eta_s\widehat G),
\]
where the safeguarded deterministic schedule and empirical GD recursion are as in the setup.
On the covariance event
\[
    \mathcal E_{\mathrm{cov}}(R)
    :=\left\{
    \left\|G^{-1/2}(\widehat G-G)G^{-1/2}\right\|_{\mathrm{op}}
    \le c_{\mathrm{rel}}R^{-1/2}
    \right\},
\]
the safeguard is inactive and
\(\widehat g_L(\widehat G)=g_L(\widehat G)\), where
\[
    g_L(s):=\sum_{t=1}^L\eta_t
    \prod_{q=t+1}^L(1-\eta_qs).
\]
For the main theorems in both regimes we instead use
\[
    \mathcal E_\eta(R^{-1})
    :=\left\{
    \left\|(G+R^{-1}I)^{-1/2}
      (\widehat G-G)(G+R^{-1}I)^{-1/2}\right\|_{\mathrm{op}}
    \le\eta
    \right\}.
\]
The full-spectrum relative event is retained only in auxiliary comparisons
under the stronger, original sample conditions.
The stationary variance component and its auxiliary population-filter
version are
\[
    \boxed{
    \mathrm{Var}:=\frac12\mathbb E_{\mathcal D}
    [\|Q\widehat g_L(\widehat G)\|_D^2\mid S],
    }
    \qquad
    \boxed{
    \mathrm{Var}^{\mathrm{pop}}:=\frac12\mathbb E_{\mathcal D}
    [\|Qg_L(G)\|_D^2\mid S].
    }
\]
Set \(\lambda:=R^{-1}\) and
\[
    K_\lambda(G):=G(G+\lambda I)^{-2},
    \qquad
    d_{\mathrm{stat}}(R,M,P):=\operatorname{Tr}[DK_\lambda(G)].
\]
In the eigenbasis of \(G\),
\[
\begin{aligned}
    d_{\mathrm{stat}}(R,M,P)
    &=\sum_{j=1}^M d_j\frac{\mu_j}{(\mu_j+\lambda)^2}\\
    &\asymp\sum_{j=1}^M\frac{d_j}{\mu_j}
       \min\{1,(R\mu_j)^2\}.
\end{aligned}
\]
The training effective dimension is retained as an auxiliary quantity:
\[
    d_{\mathrm{eff}}(R,M,P)
    :=\sum_{j=1}^M\min\{1,(R\mu_j)^2\}.
\]
In Regime II these two effective dimensions need not be comparable.
By the covariance comparison established in Appendix~\ref{app:bias}, uniformly
for \(P\ge2\), \(D\preceq CG\); in Regime I, also \(cG\preceq D\).

\subsection{General GD-variance upper bound}
\label{app:variance-general-upper}

\begin{lemma}[Population-filter stationary variance]
\label{lem:stationary-population-variance}
On the joint sketch event, under the standing model and step-size assumptions,
for every \(N\ge1\) and \(P\ge2\),
\[
    \boxed{
    \mathrm{Var}^{\mathrm{pop}}
    \asymp\frac{1}{NP}d_{\mathrm{stat}}(R,M,P).
    }
\]
No commutation between \(D\) and \(G\), or sample-size condition ensuring a
covariance event, is required.
\end{lemma}

\begin{proof}
Lemmas~\ref{lem:nonstationary-score-lower} and
\ref{lem:variance-residual-moments} give, conditional on the sketch,
\[
    \frac{c}{NP}G
    \preceq\mathbb E_{\mathcal D}[Q^\top Q\mid S]
    \preceq\frac{C}{NP}G.
\]
Thus
\[
\begin{aligned}
    2\mathrm{Var}^{\mathrm{pop}}
    &=\operatorname{Tr}\left[
       g_L(G)Dg_L(G)\,\mathbb E(Q^\top Q\mid S)\right]\\
    &\asymp\frac1{NP}\operatorname{Tr}[g_L(G)Dg_L(G)G]\\
    &=\frac1{NP}\operatorname{Tr}[DGg_L(G)^2].
\end{aligned}
\]
The last equality uses cyclicity of trace and the fact that \(g_L(G)\)
commutes with \(G\), not with \(D\). Lemma~\ref{lem:lin-gd-filters}
gives \(s^2g_L(s)^2\asymp\min\{1,(Rs)^2\}\) on the spectrum of \(G\),
so \(Gg_L(G)^2\asymp K_\lambda(G)\). Taking the trace against \(D\succeq0\)
proves the result.
\end{proof}

\begin{lemma}[Stationary spectral upper bound]
\label{lem:stationary-variance-spectrum-upper}
On the joint Gaussian-sketch event,
\[
    d_j=v_j^\top Dv_j\lesssim j^{-\alpha},
    \qquad 1\le j\le M.
\]
\end{lemma}

\begin{proof}
In Regime II, apply the stationary covariance localization lemma
\ref{lem:stationary-bias-localization} from Appendix~\ref{app:bias} with
\(m=j\), and take the quadratic form in \(v_j\). In Regime I, use
\(D\preceq CG\) and \(\mu_j\asymp j^{-\alpha}\). The sketch events are
chosen jointly, with failure probability at most \(\exp(-cM)\).
\end{proof}

\begin{lemma}[Stationary spectral lower bound]
\label{lem:stationary-variance-spectrum-lower}
In Regime II, let
\(\delta:=\alpha-\theta>0\) and \(p:=\alpha/\theta\in(1,2)\). Uniformly
over \(P\ge2\), on a Gaussian-sketch event of probability at least
\(1-\exp(-cM)\),
\[
    D\succeq c \varphi_P(G),
    \qquad
    \varphi_P(t):=\frac{t^p}{t^{p-1}+P^{-p}},
\]
and consequently
\[
    d_j\gtrsim j^{-\alpha},
    \qquad 1\le j\le M.
\]
The notation \(\varphi_P\) distinguishes this matrix function from the
coordinate envelope \(f_P(j)\). Together with
Lemma~\ref{lem:stationary-variance-spectrum-upper}, this gives
\(d_j\asymp j^{-\alpha}\) in both regimes.
\end{lemma}

\begin{proof}
The standing assumptions give \(0<\delta\le2r<1<\theta\), so
\(p=1+\delta/\theta\in(1,2)\).
Write \(q:=p-1\in(0,1)\) and \(a:=P^{-p}\). For \(t>0\),
\[
    \varphi_P(t)=t-a\,g_a(t),
    \qquad g_a(t):=\frac{t}{t^q+a}.
\]
To verify the required matrix convexity, extend \(z^q\) analytically to the
slit plane with principal argument. If \(\arg z=\phi\in(0,\pi)\), then
\(0<\arg(z^q+a)<q\phi\), so
\(0<\arg g_a(z)<\phi<\pi\). The denominator has no zero in the slit
plane. Thus \(g_a\) is a Pick function positive on \((0,\infty)\), and the
L\"owner characterization implies that it is operator monotone on
\((0,\infty)\). An operator-monotone function on this interval is operator
concave. Therefore \(\varphi_P(t)=t-a g_a(t)\) is operator convex, with
\(\varphi_P(0)=0\). Direct scalar algebra also gives, for \(c\ge1\),
\[
    \varphi_P(ct)=c^p \varphi_P(t)
    \frac{t^q+a}{c^qt^q+a}
    \le c^p \varphi_P(t).
\]
The scalar function \(\varphi_P\) is increasing, as is seen by differentiating
\(t^p/(t^q+a)\).

For every index \(j\), set \(x=j^{-\alpha}\) and
\(y=P^{-1}j^{-\theta}\). If \(x\ge y\), then \(P\ge j^\delta\), and
\(\varphi_P(x)\asymp j^{-\alpha}\). If \(y\ge x\), then \(P\le j^\delta\),
and again \(\varphi_P(y)\asymp j^{-\alpha}\). Since
\(h_{P,j}\asymp x+y\), scalar monotonicity and the growth inequality imply
\[
    \varphi_P(h_{P,j})\asymp j^{-\alpha}
\]
uniformly in \(P\).

Split the sketch at \(h=\lfloor c_0M\rfloor\), with fixed small \(c_0\),
and write \(S=(S_h,S_t)\),
\[
    G=G_h+G_t,
    \quad G_h=S_hH_hS_h^\top,
    \quad G_t=S_tH_tS_t^\top,
\]
with analogous definitions for \(D_h,D_t\). On the uniform Gaussian
head--tail event, \(\|S_h\|_{\mathrm{op}}\le C\) and
\[
    G_t\preceq C h_{P,M}I,
    \qquad D_t\succeq cM^{-\alpha}I.
\]
The second inequality follows from a fixed block of more than \(M\)
columns with indices of order \(M\), whose stationary weights are all
comparable to \(M^{-\alpha}\); the ambient-dimension assumption ensures
such a block is present. Weighted Gaussian tail concentration gives the
first inequality. Both events have failure probability \(\exp(-cM)\),
uniformly over \(P\).

Normalize \(S_h=C_0U\) with \(\|U\|_{\mathrm{op}}\le1\), choosing a fixed
\(C_0\ge\max\{1,\|S_h\|_{\mathrm{op}}\}\) bounded on the event. To see the
compression step explicitly, set \(V=U^\top\) and dilate it to the isometry
\(\mathcal V z=(Vz,(I-V^\top V)^{1/2}z)\). The Hansen--Pedersen inequality
applied to \(H_h\oplus0\), together with \(\varphi_P(0)=0\), gives
\(\varphi_P(UH_hU^\top)\preceq U \varphi_P(H_h)U^\top\).
This compression is in the original teacher coordinates; it does not assume
that the diagonal entries of \(H_h\) are ordered.
The scalar growth bound and
\(\varphi_P(h_{P,j})\lesssim d_j^{\mathrm{amb}}\), where
\(d_j^{\mathrm{amb}}:=\sigma_{\infty,j}\asymp j^{-\alpha}\), yield
\[
    \varphi_P(G_h)
    \preceq C S_h \varphi_P(H_h)S_h^\top
    \preceq C D_h.
\]
Since \(G_t\preceq C h_{P,M}I\), scalar monotonicity and the growth
comparison show
\[
    \varphi_P(G_t)\preceq \varphi_P(C h_{P,M})I
    \preceq C M^{-\alpha}I
    \preceq C D_t.
\]
Finally, operator convexity gives
\[
\begin{aligned}
    \varphi_P(G_h+G_t)
    &=\varphi_P\left(2\frac{G_h+G_t}{2}\right)\\
    &\preceq 2^p \varphi_P\left(\frac{G_h+G_t}{2}\right)\\
    &\preceq 2^{p-1}\{\varphi_P(G_h)+\varphi_P(G_t)\}\\
    &\preceq C(D_h+D_t)=CD.
\end{aligned}
\]
This proves \(D\succeq c \varphi_P(G)\). Applying the matrix inequality to each
unit eigenvector \(v_j\), using
\(\mu_j\asymp h_{P,j}\) from Lemma~\ref{lem:lin-sketch-transfer}, gives
\[
    d_j\ge c \varphi_P(\mu_j)
       \asymp\varphi_P(h_{P,j})\asymp j^{-\alpha}.
\]
All sketch events used here can be intersected with the joint event of
Appendix~\ref{app:bias}; the finite intersection retains exponentially high
probability and constants uniform in \(P\).
\end{proof}

\begin{lemma}[Empirical response-filter transfer in the stationary norm]
\label{lem:stationary-variance-filter-transfer}
Suppose \(D\preceq CG\), \(\|G\|_{\mathrm{op}}\lesssim1\), and
\[
    \eta:=\|G^{-1/2}(\widehat G-G)G^{-1/2}\|_{\mathrm{op}}\le\frac12.
\]
On the covariance event, for every (possibly data-dependent) matrix \(Q\),
\[
    \boxed{
    \|Q[g_L(\widehat G)-g_L(G)]\|_D
    \lesssim\eta\log(e+R)\|Qg_L(G)\|_G.
    }
\]
The estimate is deterministic and does not assume commutation.
\end{lemma}

\begin{proof}
Write \(\widehat G-G=G^{1/2}EG^{1/2}\), with \(\|E\|_{\mathrm{op}}=\eta\).
For \(s=2,\ldots,L\), define
\[
    g_{s-1}(\widehat G)
    :=\sum_{t=1}^{s-1}\eta_t
      \prod_{q=t+1}^{s-1}(I-\eta_q\widehat G),
    \qquad
    \Psi_{>s}(G):=\prod_{q=s+1}^L(I-\eta_qG).
\]
For either matrix \(A\), the prefix filters satisfy the finite recursion
\[
    g_s(A)=\eta_s I+g_{s-1}(A)(I-\eta_sA),
    \qquad g_1(A)=\eta_1I.
\]
Subtracting this recursion at \(A=\widehat G\) and \(A=G\) gives
\[
    g_s(\widehat G)-g_s(G)
    =[g_{s-1}(\widehat G)-g_{s-1}(G)](I-\eta_sG)
    -\eta_sg_{s-1}(\widehat G)(\widehat G-G).
\]
Iterating this identity from \(s=2\) to \(L\), retaining each suffix on
the right, gives the exact ordered identity
\[
    g_L(\widehat G)-g_L(G)
    =-\sum_{s=2}^L\eta_s
      g_{s-1}(\widehat G)(\widehat G-G)\Psi_{>s}(G).
\]

Put \(u_{s-1}:=\sum_{q<s}\eta_q\). Stability of each factor and the
response telescoping identity imply, for \(x\ge0\),
\[
    0\le g_{s-1}(x)\le\min\{u_{s-1},x^{-1}\}.
\]
Since \(u_{s-1}\le R\), splitting at \(x=R^{-1}\) yields
\[
    xg_{s-1}(x)^2
    \lesssim\frac{x}{(x+R^{-1})^2}.
\]
Functional calculus therefore gives
\(g_{s-1}(\widehat G)\widehat Gg_{s-1}(\widehat G)
\preceq C K_\lambda(\widehat G)\).
For clarity, the ridge comparison remains valid for every \(\eta\le1/2\).
Set \(A=G+\lambda I\), \(J=A^{-1/2}GA^{-1/2}\),
\(\Delta=G^{-1/2}(\widehat G-G)G^{-1/2}\), and
\(E_0=A^{-1/2}(\widehat G-G)A^{-1/2}=J^{1/2}\Delta J^{1/2}\).
Then \(\|E_0\|_{\mathrm{op}}\le\eta\). With \(T=(I+E_0)^{-1}\),
\[
    \widehat A^{-1}G\widehat A^{-1}
    =A^{-1/2}TJT A^{-1/2},
    \qquad \widehat A:=\widehat G+\lambda I.
\]
The weighted similarity \(L:=J^{1/2}TJ^{-1/2}\) satisfies
\(L=I-J^{1/2}TJ^{1/2}\Delta\) and \(L^{-1}=I+J\Delta\), so
\(\|L\|_{\mathrm{op}}+\|L^{-1}\|_{\mathrm{op}}\le C\). Since
\(TJT=J^{1/2}L^\top LJ^{1/2}\), it follows that
\(\widehat A^{-1}G\widehat A^{-1}\asymp K_\lambda(G)\). The relative
Loewner bound on \(\widehat G\) then gives
\(K_\lambda(\widehat G)\asymp K_\lambda(G)\).
By congruence, the relative ordering and the ridge comparison imply
\[
\begin{aligned}
    g_{s-1}(\widehat G)Gg_{s-1}(\widehat G)
    &\preceq (1-\eta)^{-1}
       g_{s-1}(\widehat G)\widehat Gg_{s-1}(\widehat G)\\
    &\preceq C K_\lambda(\widehat G)
     \preceq C K_\lambda(G).
\end{aligned}
\]
The full-horizon scalar response bound in Lemma~\ref{lem:lin-gd-filters}
gives \(K_\lambda(G)\preceq Cg_L(G)Gg_L(G)\). Taking traces against
\(Q^\top Q\) shows
\[
    \|Qg_{s-1}(\widehat G)G^{1/2}\|_F
    \lesssim\|Qg_L(G)\|_G.
\]
This verifies the prefix-filter comparison directly; it does not infer it
from the full-horizon statement alone.
Let \(v_s:=\sum_{q>s}\eta_q\). Scalar product smoothing gives
\(\|G\Psi_{>s}(G)\|_{\mathrm{op}}\lesssim(1+v_s)^{-1}\). As
\(D\preceq CG\),
\[
    \|G^{1/2}\Psi_{>s}(G)D^{1/2}\|_{\mathrm{op}}
    \lesssim(1+v_s)^{-1}.
\]
Applying these two bounds termwise in the exact identity gives
\[
    \|Q[g_L(\widehat G)-g_L(G)]\|_D
    \lesssim\eta\|Qg_L(G)\|_G
       \sum_{s=2}^L\frac{\eta_s}{1+v_s}.
\]
Because \(0<\eta_s\le\eta_0\lesssim1\), this sum is bounded by
\(C\log(e+R)\): use
\(v_{s-1}=v_s+\eta_s\) to compare each summand to
\(\log((1+v_{s-1})/(1+v_s))\), then telescope. This proves the claim for
every admissible finite-step deterministic schedule. Every inequality is pathwise, so no
independence between \(Q\) and \(\widehat G\) is needed.
\end{proof}

\begin{proposition}[General stationary GD-variance upper bound]
\label{prop:variance-general-upper}
Under Assumptions~\ref{assump:stable-rnn}--\ref{assump:stepsize}, the joint
Gaussian-sketch event, and the original sample-size conditions ensuring
\(\mathcal E_{\mathrm{cov}}(R)\),
\[
    \boxed{
    \mathrm{Var}\lesssim\frac{C}{NP}
    \left[d_{\mathrm{stat}}(R,M,P)+R^2e^{-cM/2}\right].
    }
\]
\end{proposition}

\begin{proof}
On \(\mathcal E_{\mathrm{cov}}(R)\), the safeguard is inactive. Apply
Lemma~\ref{lem:stationary-variance-filter-transfer} and the squared triangle
inequality to obtain
\[
\begin{aligned}
    \mathbb E[\|Qg_L(\widehat G)\|_D^2\mathbf1_{\mathcal E_{\mathrm{cov}}}\mid S]
    &\lesssim \mathbb E\|Qg_L(G)\|_D^2\\
    &\quad+\frac{c_{\mathrm{rel}}^2\log^2(e+R)}{R}
       \mathbb E\|Qg_L(G)\|_G^2.
\end{aligned}
\]
Lemma~\ref{lem:stationary-population-variance} bounds the first term by
\(C d_{\mathrm{stat}}/(NP)\); the score-covariance upper bound and Lemma
\ref{lem:lin-gd-filters} bound the second by
\(C c_{\mathrm{rel}}^2\log^2(e+R)d_{\mathrm{eff}}/(RNP)\).
The absorption estimates below show that this is at most
\(C d_{\mathrm{stat}}/(NP)\), with a sufficiently small fixed
\(c_{\mathrm{rel}}\).

On the complement, safeguard contraction gives
\(\|\widehat g_L(\widehat G)\|_{\mathrm{op}}\le R\). Since \(P\ge2\),
\(D\preceq CG\), and thus
\[
    \|Q\widehat g_L(\widehat G)\|_D^2
    \lesssim R^2\|Q\|_F^2.
\]
Lemma~\ref{lem:score-chaos-fourth-moment}, Cauchy--Schwarz, and the original
covariance event imply
\[
    \mathbb E[\|Q\|_F^2\mathbf1_{\mathcal E_{\mathrm{cov}}^c}\mid S]
    \lesssim\frac{1}{NP}e^{-cM/2},
\]
using \(\mathbb E\|Q\|_F^2\lesssim\operatorname{Tr}(G)/(NP)\lesssim1/(NP)\).
Combining the event contributions proves the proposition.
\end{proof}

\subsection{General GD-variance lower bound}
\label{app:variance-general-lower}

\begin{proposition}[General stationary GD-variance lower bound]
\label{prop:variance-general-lower}
Under the same sketch and sample conditions, and the perturbation-absorption
condition
\[
    \frac{\log^2(e+R)}R
    \frac{d_{\mathrm{eff}}(R,M,P)}{d_{\mathrm{stat}}(R,M,P)}
    \le C_0
\]
for a fixed constant \(C_0\), with \(c_{\mathrm{rel}}\) sufficiently small,
\[
    \boxed{
    \mathrm{Var}\gtrsim\frac{c}{NP}d_{\mathrm{stat}}(R,M,P).
    }
\]
\end{proposition}

\begin{proof}
By Lemma~\ref{lem:stationary-population-variance},
\(\mathbb E\|Qg_L(G)\|_D^2\asymp d_{\mathrm{stat}}/(NP)\). Lemma
\ref{lem:score-chaos-fourth-moment}, applied with deterministic
\(W=g_L(G)D^{1/2}\), gives
\[
    \mathbb E\left[\|Qg_L(G)\|_D^2
       \mathbf1_{\mathcal E_{\mathrm{cov}}^c}\mid S\right]
    \lesssim e^{-cM/2}\mathbb E\|Qg_L(G)\|_D^2.
\]
Thus restriction to the good event retains a fixed fraction of the
population-filter second moment. On the good event, apply the reverse triangle
inequality in the Hilbert space of square-integrable matrices and Lemma
\ref{lem:stationary-variance-filter-transfer}. The squared perturbation is at
most
\[
    \frac{C c_{\mathrm{rel}}^2\log^2(e+R)}{RNP}
       d_{\mathrm{eff}}(R,M,P).
\]
The displayed absorption condition and a sufficiently small
\(c_{\mathrm{rel}}\) make its square root strictly smaller than the
population contribution. The squared reverse-triangle bound proves the
proposition.
\end{proof}

\paragraph{Absorption in the two regimes.}
In Regime I, \(D\asymp G\), hence
\(d_{\mathrm{stat}}\asymp d_{\mathrm{eff}}\), and
\(\log^2(e+R)/R\lesssim1\) for \(R\gtrsim1\). In Regime II, put
\(\delta=\alpha-\theta\), \(k=\kappa_{R,P}\), and
\(k_0:=\min\{k,M\}\). Lemma~\ref{lem:two-scale-cutoff} gives the
mixed-spectrum cutoff and the training effective-dimension estimate; the
stationary spectral comparison in Lemmas~\ref{lem:stationary-variance-spectrum-upper}
and~\ref{lem:stationary-variance-spectrum-lower} then gives, including the
saturated range,
\[
    d_{\mathrm{eff}}\asymp k_0,
    \qquad
    d_{\mathrm{stat}}\asymp\frac{k_0}{1+k_0^\delta/P}.
\]
Consequently, the absorption ratio is
\[
    \frac{\log^2(e+R)}R
    \frac{d_{\mathrm{eff}}}{d_{\mathrm{stat}}}
    \asymp\frac{\log^2(e+R)}R(1+k_0^\delta/P).
\]
If \(k\le M\), the cutoff relation gives
\(R\asymp k^\alpha/(1+k^\delta/P)\), and hence
\[
    \frac{\log^2(e+R)}R
    \frac{d_{\mathrm{eff}}}{d_{\mathrm{stat}}}
    \lesssim
    \frac{(1+k^\delta/P)^2\log^2(e+R)}{k^\alpha}
    \lesssim k^{2\delta-\alpha}\log^2(e+k)\lesssim1.
\]
Here \(0<\delta\le2r<1\), \(\theta>1\), and thus
\(\alpha=\theta+\delta>2\delta\).

If \(k\ge M\), all training spectral modes are at or above the cutoff, so
\(d_{\mathrm{eff}}\asymp M\) and
\(d_{\mathrm{stat}}\asymp M/(1+M^\delta/P)\). Monotonicity of the cutoff
relation implies
\(R\gtrsim R_M:=M^\alpha/(1+M^\delta/P)\). Since
\(\log^2(e+R)/R\) is decreasing for \(R\ge1\),
\[
\begin{aligned}
    \frac{\log^2(e+R)}R
    \frac{d_{\mathrm{eff}}}{d_{\mathrm{stat}}}
    &\lesssim
    \frac{(1+M^\delta/P)^2\log^2(e+R_M)}{M^\alpha}\\
    &\lesssim M^{2\delta-\alpha}\log^2(e+M)\lesssim1.
\end{aligned}
\]
The finitely many small \(M\) are absorbed into the constants. Thus the
perturbation-absorption condition holds uniformly in both regimes, including
the saturated range. Choose \(c_{\mathrm{rel}}\) sufficiently small after
these uniform bounds.
\subsection{Specific GD-variance bound for
\texorpdfstring{\(\theta\ge\alpha\)}{theta >= alpha}}
\label{app:variance-specific-regime-one}

\begin{theorem}[Stationary GD-variance bound for \(\theta\ge\alpha\)]
\label{thm:variance-specific-regime-one}
Under the assumptions, strict unsaturation, and regularized sample
condition of Theorem~\ref{thm:bias-specific-regime-one}, on the joint
Gaussian-sketch event,
\[
    \mathrm{Var}\asymp\frac{R^{1/\alpha}}{NP}
\]
for the full conditional expectation over the training trajectories.
The boundary case \(\theta=\alpha\) is included.
\end{theorem}

\begin{proof}
Here \(D\asymp G\), and the spectral transfer gives
\(\mu_j\asymp j^{-\alpha}\) and \(d_j\asymp j^{-\alpha}\). Therefore
\[
\begin{aligned}
    d_{\mathrm{stat}}(R,M,P)
    &\asymp\sum_{j=1}^M\min\{1,(Rj^{-\alpha})^2\}\\
    &\asymp k,\qquad k=\kappa_{R,P}\asymp R^{1/\alpha},
\end{aligned}
\]
where strict unsaturation gives \(k\le cM\).
Lemma~\ref{lem:stationary-population-variance} implies
\(\mathrm{Var}^{\mathrm{pop}}\asymp k/(NP)\).
Use the same \(q_\tau,\eta,\tau\) and regularized event as in
Theorem~\ref{thm:bias-specific-regime-one}. On this event,
Lemma~\ref{lem:regularized-response-filter-transfer} gives
\[
    \mathbb E\!\left[
       \|Q[g_L(\widehat G)-g_L(G)]\|_D^2
       \mathbf1_{\mathcal E_\eta}\mid S\right]
    \lesssim \eta^2\log^2(e+R)\frac{k}{NP}.
\]
The score covariance upper bound and
\(\operatorname{Tr}[G(G+R^{-1}I)^{-1}]\asymp k\) justify taking this
expectation even though \(Q\) and \(\widehat G\) are dependent.
On the complement, both the actual safeguarded response filter and the
population response filter have operator norm at most \(R\).
Lemma~\ref{lem:score-chaos-fourth-moment} therefore bounds the expected
squared difference there by
\[
    C R^2(\mathbb E\|Q\|_F^4)^{1/2}e^{-\tau/2}
    \lesssim \frac{R^2}{NP}e^{-\tau/2}.
\]
Its ratio to \(k/(NP)\) is at most
\(CM^{2\alpha}(M+R+2)^{-C_{\mathrm{conf}}/2}\).
The sample condition, \((M/k)^{1-2r}\ge1\), and sufficiently large
\(C_{\mathrm{conf}}\) make the full-data squared filter difference
arbitrarily small relative to population variance. The triangle and reverse
triangle inequalities in the square-integrable \(D\)-norm prove the result.
\end{proof}

\subsection{Specific GD-variance bound for
\texorpdfstring{\(\alpha-2r\le\theta<\alpha\)}{alpha - 2r <= theta < alpha}}
\label{app:variance-specific-regime-two}

\begin{theorem}[Stationary GD-variance bound for heavier initialization]
\label{thm:variance-specific}
Suppose \(\alpha-2r\le\theta<\alpha\), put
\(\delta=\alpha-\theta\), and let \(k=\kappa_{R,P}\) satisfy
\[
    R\left(k^{-\alpha}+P^{-1}k^{-\theta}\right)\asymp1.
\]
Assume \(k\le c_{\mathrm{unsat}}M\), and define
\[
    s_{k,P}:=1+\frac{k^\delta}{P},
    \quad
    \chi_{M,k}:=\left(\frac Mk\right)^{1-2r},
    \quad
    \omega_{0,M,P}:=\min\left\{1,\frac{M^\delta}{P}\right\},
\]
\(\tau=C_{\mathrm{conf}}\log(M+R+2)\). If
\[
    \widetilde C_\rho(k+\tau)
    \left(\frac1{NP}+\frac{\omega_{0,M,P}}N\right)
    s_{k,P}\chi_{M,k}\log^2(e+R)
    \le c_\star,
\]
where \(\widetilde C_\rho=(1-\rho)^{-2}\) and \(c_\star>0\) is sufficiently
small depending on \(\rho\) and the fixed primitive data, then, with high
probability over the sketch,
\[
    \boxed{
    \mathrm{Var}\asymp\frac1{NP}
    \frac{\kappa_{R,P}}{1+\kappa_{R,P}^{\alpha-\theta}/P}
    }
\]
for the full conditional expectation over the training trajectories.
Equivalently, with \(R_P=P^{\alpha/\delta}\),
\[
    \mathrm{Var}\asymp
    \begin{cases}
        \dfrac{R^{1/\alpha}}{NP},&R\lesssim R_P,\\[2mm]
        \dfrac1N\left(\dfrac RP\right)^{(1-\delta)/\theta},
        &R\gtrsim R_P.
    \end{cases}
\]
\end{theorem}

\begin{proof}
The mixed-spectrum transfer gives
\(\mu_j\asymp j^{-\alpha}(1+j^\delta/P)\), and
Lemmas~\ref{lem:stationary-variance-spectrum-upper} and
~\ref{lem:stationary-variance-spectrum-lower} give
\(d_j\asymp j^{-\alpha}\). Thus
\[
\begin{aligned}
    d_{\mathrm{stat}}(R,M,P)
    &\asymp\sum_{j=1}^M\frac1{1+j^\delta/P}
       \min\left\{1,
       R^2j^{-2\alpha}(1+j^\delta/P)^2\right\}.
\end{aligned}
\]
Split at \(j\asymp k\). Since \(0<\delta<1\), integral comparison gives
\[
    \sum_{j\le k}\frac1{1+j^\delta/P}
    \asymp\frac{k}{1+k^\delta/P}.
\]
For the tail, \(2\alpha-\delta>1\) and the cutoff relation yield
\[
\begin{aligned}
    R^2\sum_{j>k}j^{-2\alpha}(1+j^\delta/P)
    &\asymp R^2k^{1-2\alpha}(1+k^\delta/P)\\
    &\asymp\frac{k}{1+k^\delta/P}.
\end{aligned}
\]
The unsaturated condition ensures the retained head and tail ranges needed for
the comparisons. Hence
\[
    d_{\mathrm{stat}}\asymp\frac{k}{s_{k,P}},
    \qquad
    \mathrm{Var}^{\mathrm{pop}}\asymp\frac{k}{NPs_{k,P}}.
\]

As in the bias proof, put
\[
    q_\tau:=\widetilde C_\rho(k+\tau)
      \left(\frac1{NP}+\frac{\omega_{0,M,P}}N\right),
    \qquad
    \eta:=C_{\mathrm{sp}}(\sqrt{q_\tau}+q_\tau),
\]
with the same fixed-primitive \(C_{\mathrm{sp}}\) as in the regularized
covariance lemma.
The sample condition implies
\(\eta^2s_{k,P}\chi_{M,k}\log^2(e+R)\le c\), and hence also
\(\eta^2s_{k,P}\log^2(e+R)\le c\), because
\(\chi_{M,k}\ge1\). Lemma
\ref{lem:rnn-regularized-covariance-concentration} gives
\(\mathbb P_{\mathcal D}(\mathcal E_\eta(R^{-1})^c\mid S)\le2e^{-\tau}\).

On the good event, Lemma~\ref{lem:regularized-response-filter-transfer}
gives
\[
\begin{aligned}
    \mathbb E\!\left[
      \|Q[g_L(\widehat G)-g_L(G)]\|_D^2
      \mathbf1_{\mathcal E_\eta}\,\middle|\,S\right]
    &\lesssim\eta^2\log^2(e+R)
      \mathbb E\|Q(G+R^{-1}I)^{-1/2}\|_F^2\\
    &\lesssim\eta^2\log^2(e+R)\frac{k}{NP}.
\end{aligned}
\]
The second inequality follows from the trajectory-score covariance upper
bound and
\(\operatorname{Tr}[G(G+R^{-1}I)^{-1}]\asymp k\). Relative to
\(\mathrm{Var}^{\mathrm{pop}}\), this is at most
\(C\eta^2s_{k,P}\log^2(e+R)\), so it is absorbed.

To retain the population contribution on the good event, apply
Lemma~\ref{lem:score-chaos-fourth-moment} to
\(Y=\|Qg_L(G)\|_D^2\). Cauchy--Schwarz gives
\[
    \mathbb E[Y\mathbf1_{\mathcal E_\eta^c}\mid S]
    \lesssim e^{-\tau/2}\mathbb E[Y\mid S].
\]
Thus the good event retains a fixed fraction of the population-filter
variance. The reverse triangle inequality in the Hilbert space of
square-integrable matrices gives the empirical lower bound, and the ordinary
triangle inequality gives the matching good-event upper bound.

On the complement, safeguard contraction gives
\(\|\widehat g_L(\widehat G)\|_{\mathrm{op}}\le R\). Applying the same
fourth-moment lemma with \(W=I\) yields
\[
    \mathbb E\!\left[
      \|Q\widehat g_L(\widehat G)\|_D^2
      \mathbf1_{\mathcal E_\eta^c}\,\middle|\,S\right]
    \lesssim\frac{R^2}{NP}e^{-\tau/2}.
\]
Since \(0<\delta<1\), \(k\gtrsim1\), and \(P\ge2\), one has
\(s_{k,P}/k\lesssim1\). Strict unsaturation also gives
\(R\lesssim M^\alpha\). Thus the ratio of the last display to
\(k/(NPs_{k,P})\) is at most
\[
    C R^2\frac{s_{k,P}}k(M+R+2)^{-C_{\mathrm{conf}}/2}
    \lesssim M^{2\alpha}(M+R+2)^{-C_{\mathrm{conf}}/2}.
\]
The same choice \(C_{\mathrm{conf}}>4\alpha+4\) as in the bias proof,
increased if needed for fixed constants, makes this negligible.
Combining the good- and bad-event contributions proves
the first rate. On either side of the crossover
\(R_P=P^{\alpha/\delta}\), the cutoff satisfies respectively
\(k\asymp R^{1/\alpha}\) or \(k\asymp(R/P)^{1/\theta}\). Substitution
proves the piecewise form.
\end{proof}


\section{Auxiliary lemmas}
\label{app:auxiliary}

\begin{lemma}[Primitive dynamics imply a two-scale covariance decomposition]
\label{lem:primitive-positionwise-regularity}
Suppose Assumptions~\ref{assump:stable-rnn} and
\ref{assump:power-law-source} hold and \(P\ge2\). Write
\[
    \Sigma_p
    =
    \underbrace{A_*^p\Sigma_0(A_*^\top)^p}_{\Sigma_p^{(0)}}
    +
    \underbrace{\sum_{q=0}^{p-1}A_*^q\Sigma_\xi(A_*^\top)^q}_{\Sigma_p^{(\xi)}},
\]
and let \(H_0\) and \(H_\xi\) be the respective position averages. Set
\[
    G_0:=SH_0S^\top,
    \qquad
    G_\xi:=SH_\xi S^\top,
    \qquad
    G=G_0+G_\xi.
\]
Then
\[
    h_{0,P,j}:=e_j^\top H_0e_j\asymp P^{-1}j^{-\theta},
    \qquad
    h_{\xi,P,j}:=e_j^\top H_\xi e_j\asymp j^{-\alpha},
\]
so
\[
    h_{P,j}:=e_j^\top He_j
    \asymp
    j^{-\alpha}+P^{-1}j^{-\theta}.
\]
Moreover,
\[
    A_*\Sigma_0A_*^\top\preceq C_{\mathrm{sm}}\Sigma_\xi,
    \qquad
    A_*\Sigma_pA_*^\top
    \preceq
    C_{\rho,C_{\mathrm{sm}}}\Sigma_\xi
\]
uniformly in \(p\), and
\(\Sigma_p^{(\xi)}\preceq2H_\xi\) for \(0\le p<P\). By the state recursion,
\(\Sigma_p\preceq(C_{\rho,C_{\mathrm{sm}}}+1)\Sigma_\xi\) for every
\(p\ge1\). When \(\theta\ge\alpha\), the same comparison holds at
\(p=0\) up to the power-law constants. In Regime II no uniform comparison
of \(\Sigma_0\) with \(\Sigma_\xi\) is asserted. The exact identities
\[
    C_p=SA_*\Sigma_pS^\top,
    \qquad
    C=SA_*HS^\top,
    \qquad
    B_*=SA_*HS^\top(SHS^\top)^{-1}
\]
remain valid. The bound used below is the one-step-smoothed comparison
\(A_*\Sigma_pA_*^\top\preceq C_{\rho,C_{\mathrm{sm}}}\Sigma_\xi\),
which holds also at \(p=0\);
no downstream argument requires \(\Sigma_0\preceq C\Sigma_\xi\) in
Regime II.
\end{lemma}

\begin{proof}
In the common eigenbasis,
\[
    \sigma_{p,j}
    =
    a_j^{2p}\sigma_{0,j}
    +
    \sum_{q=0}^{p-1}a_j^{2q}\sigma_{\xi,j}.
\]
Therefore
\[
    h_{0,P,j}
    =
    \sigma_{0,j}\frac1P\sum_{p=0}^{P-1}a_j^{2p}
    \asymp
    P^{-1}j^{-\theta},
\]
because the geometric sum is between \(1\) and \((1-\rho^2)^{-1}\). Also,
\[
    h_{\xi,P,j}
    =
    \sigma_{\xi,j}\frac1P
    \sum_{q=0}^{P-2}(P-1-q)a_j^{2q}
    \asymp
    j^{-\alpha},
\]
because the weighted sum lies between \((P-1)/P\) and
\((1-\rho^2)^{-1}\). The smoothing condition follows from
\[
    a_j^2\sigma_{0,j}
    \asymp
    j^{-(\theta+2r)}
    \lesssim
    j^{-\alpha}
    \asymp
    \sigma_{\xi,j}.
\]
Applying the recursion once more gives
\[
    A_*\Sigma_pA_*^\top
    \preceq
    \left(C_{\mathrm{sm}}+\frac{\rho^2}{1-\rho^2}\right)\Sigma_\xi.
\]
For \(p\ge1\), the recursion and this bound at index \(p-1\) give
\[
    \Sigma_p=A_*\Sigma_{p-1}A_*^\top+\Sigma_\xi
    \preceq(C_{\rho,C_{\mathrm{sm}}}+1)\Sigma_\xi.
\]
At \(p=0\), this last comparison follows from the power laws when
\(\theta\ge\alpha\), but need not hold uniformly in Regime II. The
one-step-smoothed bound above remains valid at \(p=0\) by the assumed
condition \(A_*\Sigma_0A_*^\top\preceq C_{\mathrm{sm}}\Sigma_\xi\).
For the within-trajectory comparison, let
\(b_p(a)=\sum_{q=0}^{p-1}a^{2q}\).
The sequence \(b_p(a)\) is increasing and discretely concave, with \(b_0=0\).
For \(0\le p<P\), concavity gives
\[
    \frac1P\sum_{q=0}^{P-1}b_q(a)
       \ge\frac12b_{P-1}(a)\ge\frac12b_p(a).
\]
Apply this coordinatewise to obtain
\(\Sigma_p^{(\xi)}\preceq2H_\xi\), with the constant independent of
\(\rho\) and \(P\).
The cross-covariance identities follow from innovation independence and
position averaging.
\end{proof}

\begin{lemma}[Componentwise geometric dependence]
\label{lem:primitive-excitation-mixing}
Under the conditions of Lemma~\ref{lem:primitive-positionwise-regularity},
write \(x_{i,p}=x_{i,p}^{(0)}+x_{i,p}^{(\xi)}\), where
\[
    x_{i,p}^{(0)}:=A_*^px_{i,0},
    \qquad
    x_{i,p}^{(\xi)}:=\sum_{q=0}^{p-1}A_*^q\xi_{i,p-q}.
\]
For \(q\ge p\),
\[
    \mathbb E[x_{i,q}^{(\xi)}(x_{i,p}^{(\xi)})^\top]
    =
    A_*^{q-p}\Sigma_p^{(\xi)},
\]
and, with \(G_\xi:=SH_\xi S^\top\),
\[
    \left\|
    G_\xi^{-1/2}
    \mathbb E[u_{i,q}^{(\xi)}(u_{i,p}^{(\xi)})^\top]
    G_\xi^{-1/2}
    \right\|_{\mathrm{op}}
    \lesssim
    \rho^{q-p}.
\]
The initial component satisfies
\[
    \mathbb E[x_{i,q}^{(0)}(x_{i,p}^{(0)})^\top]
    =
    A_*^q\Sigma_0(A_*^\top)^p,
\]
whose norm decays geometrically in \(p+q\). Thus innovation-driven quadratic
forms have effective sample size \(NP\), whereas the initialization component
is an average of \(N\) independent Gaussian quadratic forms.
\end{lemma}

\begin{proof}
The identities follow from the two independent components of the state
recursion. By Lemma~\ref{lem:primitive-positionwise-regularity},
\[
    A_*^{q-p}\Sigma_p^{(\xi)}(A_*^\top)^{q-p}
    \preceq
    C\rho^{2(q-p)}H_\xi,
    \qquad
    \Sigma_p^{(\xi)}\preceq CH_\xi.
\]
The same factorization and Cauchy--Schwarz argument used for normalized
cross-covariances gives the displayed sketched bound. The initial-state claim
follows from \(\|A_*^p\|_{\mathrm{op}}\le\rho^p\).
\end{proof}

\begin{lemma}[Innovation--corrector coercivity]
\label{lem:innovation-corrector-coercivity}
Let \(\Sigma\succeq0\) and \(\Sigma_\xi\succ0\), and suppose
\[
    A_*\Sigma A_*^\top\preceq\Lambda^2\Sigma_\xi.
\]
Then, for every symmetric matrix \(K\) and vectors \(r,c\),
\[
\begin{aligned}
    &\left\|
    \Sigma_\xi^{1/2}
    \left(rc^\top+2KA_*\right)
    \Sigma^{1/2}
    \right\|_F^2
    +
    2\left\|
    \Sigma_\xi^{1/2}K\Sigma_\xi^{1/2}
    \right\|_F^2 \\
    &\qquad\ge
    \frac{1}{1+2\Lambda^2}
    (r^\top\Sigma_\xi r)
    (c^\top\Sigma c).
\end{aligned}
\]
\end{lemma}

\begin{proof}
Set
\[
    X:=\Sigma_\xi^{1/2}rc^\top\Sigma^{1/2},
    \qquad
    U:=2\Sigma_\xi^{1/2}KA_*\Sigma^{1/2},
    \qquad
    W:=\Sigma_\xi^{1/2}K\Sigma_\xi^{1/2}.
\]
The covariance comparison implies
\[
    \left\|
    \Sigma_\xi^{-1/2}A_*\Sigma^{1/2}
    \right\|_{\mathrm{op}}
    \le\Lambda,
\]
and therefore \(\|U\|_F\le2\Lambda\|W\|_F\). If \(\Lambda=0\), the
claim is immediate. Otherwise,
\[
    2\|W\|_F^2
    \ge
    \frac{1}{2\Lambda^2}\|U\|_F^2.
\]
For any vectors in a Hilbert space and any \(a>0\), completing the square gives
\[
    \|X+U\|^2+a\|U\|^2
    \ge
    \frac{a}{1+a}\|X\|^2.
\]
Taking \(a=(2\Lambda^2)^{-1}\) proves the result because
\(\|X\|_F^2=(r^\top\Sigma_\xi r)(c^\top\Sigma c)\).
\end{proof}

\begin{lemma}[Nonstationary trajectory-score lower bound]
\label{lem:nonstationary-score-lower}
Suppose \(\|A_*\|_{\mathrm{op}}\le\rho<1\), the matrices
\(A_*\), \(\Sigma_0\), and \(\Sigma_\xi\) are diagonal in a common basis,
and the one-step smoothing condition
\[
    A_*\Sigma_0A_*^\top
    \preceq
    C_{\mathrm{sm}}\Sigma_\xi
\]
holds.
Let
\[
    D:=SA_*-B_*S,
    \qquad
    \Omega_\xi:=S\Sigma_\xi S^\top,
\]
and
\[
    Q
    :=
    \widehat C-B_*\widehat G
    =
    \frac1{NP}
    \sum_{i=1}^N\sum_{p=0}^{P-1}
    e_{i,p}u_{i,p}^\top.
\]
There is a constant \(c_{\rho,C_{\mathrm{sm}}}>0\), independent of
\(M,N,P,L\), such
that
\[
    \mathbb E_{\mathcal D}[Q^\top Q\mid S]
    \succeq
    \frac{c_{\rho,C_{\mathrm{sm}}}\operatorname{Tr}(\Omega_\xi)}{NP}G.
\]
On the Gaussian-sketch event of Lemma~\ref{lem:lin-sketch-transfer},
\(\operatorname{Tr}(\Omega_\xi)\asymp1\), and hence
\[
    \mathbb E_{\mathcal D}[Q^\top Q\mid S]
    \succeq
    \frac{c_{\rho,C_{\mathrm{sm}}}}{NP}G.
\]
No stationarity assumption is used.
\end{lemma}

\begin{proof}
Fix \(v\in\mathbb R^M\) and one trajectory, suppressing its trajectory
index. Define
\[
    T(v):=\sum_{p=0}^{P-1}e_p(v^\top u_p).
\]
The averaged normal equation gives
\[
    DHS^\top
    =
    SA_*HS^\top-B_*SHS^\top
    =
    C-B_*G
    =0.
\]
Therefore
\[
    \mathbb E[T(v)]
    =
    \sum_{p=0}^{P-1}D\Sigma_pS^\top v
    =
    PDHS^\top v
    =0.
\]

Let \(f_m\) denote the \(m\)-th standard basis vector of \(\mathbb R^M\),
and set
\[
    r_m:=S^\top f_m,
    \qquad
    d_m:=D^\top f_m,
    \qquad
    c:=S^\top v,
    \qquad
    F_m:=\frac12(d_mc^\top+cd_m^\top).
\]
Write \(T_m(v):=f_m^\top T(v)\). Its position-\(p\) contribution is
\[
    f_m^\top e_p(v^\top u_p)
    =
    x_p^\top F_mx_p
    +
    (r_m^\top\xi_{p+1})(c^\top x_p).
\]
Define
\[
    q_{p,m}(x)
    :=
    x^\top F_mx-\operatorname{Tr}(F_m\Sigma_p).
\]
Since
\[
    \sum_{p=0}^{P-1}\operatorname{Tr}(F_m\Sigma_p)
    =
    P d_m^\top Hc
    =
    P f_m^\top DHS^\top v
    =0,
\]
the full trajectory score is unchanged if each quadratic term is replaced by
\(q_{p,m}\).

For each output coordinate, define the finite-horizon backward sequence
\[
    K_{P,m}:=0,
    \qquad
    K_{p,m}
    :=
    F_m+A_*^\top K_{p+1,m}A_*,
    \qquad p=P-1,\ldots,0,
\]
and
\[
    h_{p,m}(x)
    :=
    x^\top K_{p,m}x-
    \operatorname{Tr}(K_{p,m}\Sigma_p).
\]
Let
\(\mathcal F_p:=\sigma(x_0,\xi_1,\ldots,\xi_p)\). Using
\(x_{p+1}=A_*x_p+\xi_{p+1}\) and
\(\Sigma_{p+1}=A_*\Sigma_pA_*^\top+\Sigma_\xi\), direct substitution gives
the finite-horizon Poisson identity
\[
    q_{p,m}(x_p)
    =
    h_{p,m}(x_p)
    -
    \mathbb E[h_{p+1,m}(x_{p+1})\mid\mathcal F_p].
\]
Define
\[
\begin{aligned}
    M_{p+1,m}
    :={}&
    (r_m^\top\xi_{p+1})(c^\top x_p)
    +
    h_{p+1,m}(x_{p+1}) \\
    &-
    \mathbb E[h_{p+1,m}(x_{p+1})\mid\mathcal F_p].
\end{aligned}
\]
Then \(\mathbb E[M_{p+1,m}\mid\mathcal F_p]=0\), and summing the
martingale--coboundary identity over positions yields
\[
    T_m(v)
    =
    h_{0,m}(x_0)
    +
    \sum_{p=0}^{P-1}M_{p+1,m},
\]
because \(h_{P,m}=0\). The initial term and the martingale differences are
mutually orthogonal in \(L^2\), so
\[
    \mathbb E[T_m(v)^2]
    \ge
    \sum_{p=0}^{P-1}\mathbb E[M_{p+1,m}^2].
\]

Expanding the conditional-centering term gives
\[
\begin{aligned}
    M_{p+1,m}
    ={}&
    \xi_{p+1}^\top
    \left(r_mc^\top+2K_{p+1,m}A_*\right)x_p \\
    &+
    \xi_{p+1}^\top K_{p+1,m}\xi_{p+1}
    -
    \operatorname{Tr}(K_{p+1,m}\Sigma_\xi).
\end{aligned}
\]
The bilinear term and the centered quadratic term are orthogonal by Gaussian
symmetry. Therefore
\[
\begin{aligned}
    \mathbb E[M_{p+1,m}^2]
    ={}&
    \left\|
    \Sigma_\xi^{1/2}
    \left(r_mc^\top+2K_{p+1,m}A_*\right)
    \Sigma_p^{1/2}
    \right\|_F^2 \\
    &+
    2\left\|
    \Sigma_\xi^{1/2}K_{p+1,m}\Sigma_\xi^{1/2}
    \right\|_F^2.
\end{aligned}
\]

It remains to verify the uniform covariance comparison needed for
Lemma~\ref{lem:innovation-corrector-coercivity}. In the common eigenbasis,
\[
    \sigma_{p,j}
    =
    a_j^{2p}\sigma_{0,j}
    +
    \sum_{q=0}^{p-1}a_j^{2q}\sigma_{\xi,j},
\]
and hence
\[
    A_*\Sigma_pA_*^\top
    =
    A_*^{p+1}\Sigma_0(A_*^\top)^{p+1}
    +
    \sum_{q=1}^{p}A_*^q\Sigma_\xi(A_*^\top)^q
    \preceq
    \Lambda_{\rho,C_{\mathrm{sm}}}^2\Sigma_\xi,
    \qquad
    \Lambda_{\rho,C_{\mathrm{sm}}}^2
    :=
    C_{\mathrm{sm}}+\frac{\rho^2}{1-\rho^2},
\]
uniformly in \(p\). Lemma~\ref{lem:innovation-corrector-coercivity} now gives
\[
    \mathbb E[M_{p+1,m}^2]
    \ge
    c_{\rho,C_{\mathrm{sm}}}
    (r_m^\top\Sigma_\xi r_m)
    (c^\top\Sigma_pc).
\]
Summing over output coordinates and positions,
\[
\begin{aligned}
    \mathbb E\|T(v)\|_2^2
    &\ge
    c_{\rho,C_{\mathrm{sm}}}
    \operatorname{Tr}(\Omega_\xi)
    \sum_{p=0}^{P-1}v^\top S\Sigma_pS^\top v \\
    &=
    c_{\rho,C_{\mathrm{sm}}}P
    \operatorname{Tr}(\Omega_\xi)v^\top Gv.
\end{aligned}
\]

The trajectory scores are independent and centered across \(i\). Since
\(Qv=(NP)^{-1}\sum_{i=1}^NT_i(v)\),
\[
    \mathbb E\|Qv\|_2^2
    =
    \frac{1}{NP^2}\mathbb E\|T(v)\|_2^2
    \ge
    \frac{c_{\rho,C_{\mathrm{sm}}}\operatorname{Tr}(\Omega_\xi)}{NP}v^\top Gv.
\]
This proves the Loewner lower bound. Finally,
\(\operatorname{Tr}(\Sigma_\xi)\asymp1\) because \(\alpha>1\), and standard
Gaussian quadratic-form concentration gives
\(\operatorname{Tr}(S\Sigma_\xi S^\top)\asymp1\) on the sketch event.
\end{proof}

\begin{lemma}[Residual trajectory-score covariance upper bound]
\label{lem:variance-residual-moments}
Under Assumptions~\ref{assump:stable-rnn} and
\ref{assump:power-law-source}, there is a Gaussian-sketch event with
probability at least \(1-\exp(-cM)\) on which the residual score
\[
    Q
    :=
    \widehat C-B_*\widehat G
    =
    \frac1{NP}\sum_{i=1}^N\sum_{p=0}^{P-1}
    \left\{
    e_{i,p}u_{i,p}^\top-
    \mathbb E[e_{i,p}u_{i,p}^\top]
    \right\}
\]
satisfies
\[
    \mathbb E_{\mathcal D}[Q^\top Q\mid S]
    \preceq
    \frac{C_{\rho,C_{\mathrm{sm}}}}{NP}G.
\]
On the same event,
\[
    \|B_*\|_{\mathrm{op}}\le C,
    \qquad
    \operatorname{Tr}(D\Sigma_0D^\top)\le C,
    \qquad
    D:=SA_*-B_*S.
\]
\end{lemma}

\begin{proof}
We first record two sketch estimates used to control the initial boundary
term in the finite-horizon decomposition. Fix a sufficiently small constant
\(c_0\in(0,1/3)\), set \(k=\lfloor c_0M\rfloor\), and split the primitive
eigenbasis into a head and a tail. Write
\[
    G=G_{\mathrm h}+G_{\mathrm t},
    \qquad
    C=C_{\mathrm h}+C_{\mathrm t},
\]
where
\[
\begin{aligned}
    G_{\mathrm h}&:=S_{\mathrm h}H_{\mathrm h}S_{\mathrm h}^\top,
    &\quad
    C_{\mathrm h}&:=S_{\mathrm h}H_{\mathrm h}A_{\mathrm h}S_{\mathrm h}^\top,\\
    G_{\mathrm t}&:=S_{\mathrm t}H_{\mathrm t}S_{\mathrm t}^\top,
    &\quad
    C_{\mathrm t}&:=S_{\mathrm t}H_{\mathrm t}A_{\mathrm t}S_{\mathrm t}^\top.
\end{aligned}
\]
The actual diagonal sequence need not be decreasing. Its comparison
envelope \(f_P(j)=j^{-\alpha}+P^{-1}j^{-\theta}\) is decreasing and doubling,
and the coordinate bounds \(h_{P,j}\asymp f_P(j)\) give
\[
    \frac1M\sum_{j>k}h_{P,j}\asymp h_{P,M}
\]
because \(k\asymp M\); the finite sum has the same order because it contains a
fixed annulus beyond \(k\). Sandwiching the indexed tails between the
restrictions of \(F_\alpha+P^{-1}F_\theta\) gives these bounds simultaneously
for every \(P\) on the two primitive envelope events.
The condition \(d\ge C_{\mathrm{amb}}M\) ensures
that the next \(2M\) tail columns exist. Weighted Gaussian tail concentration,
together with the lower bound obtained from those columns, therefore gives
\[
    c h_{P,M}I
    \preceq
    G_{\mathrm t}
    \preceq
    C h_{P,M}I.
\]
On the same event, rectangular Gaussian concentration gives
\[
    \|S_{\mathrm h}\|_{\mathrm{op}}\le C,
    \qquad
    S_{\mathrm h}^\top S_{\mathrm h}\succeq cI.
\]
These events have joint probability at least \(1-\exp(-cM)\), uniformly in
\(P\).

Woodbury's identity yields
\[
    G^{-1}S_{\mathrm h}H_{\mathrm h}
    =
    G_{\mathrm t}^{-1}S_{\mathrm h}
    \left(
    H_{\mathrm h}^{-1}
    +S_{\mathrm h}^\top G_{\mathrm t}^{-1}S_{\mathrm h}
    \right)^{-1}.
\]
By the preceding two displays,
\[
    S_{\mathrm h}^\top G_{\mathrm t}^{-1}S_{\mathrm h}
    \succeq
    c h_{P,M}^{-1}I,
\]
and hence
\[
    \left\|G^{-1}S_{\mathrm h}H_{\mathrm h}\right\|_{\mathrm{op}}
    \le C.
\]
It follows that
\[
    \|G^{-1}C_{\mathrm h}\|_{\mathrm{op}}
    \le
    \left\|G^{-1}S_{\mathrm h}H_{\mathrm h}\right\|_{\mathrm{op}}
    \|A_{\mathrm h}S_{\mathrm h}^\top\|_{\mathrm{op}}
    \le C.
\]

Let \(\bar a_k:=\sup_{j>k}|a_j|\). Since \(A_*\) and \(H\) are diagonal in
the same basis,
\[
    \left\|
    G_{\mathrm t}^{-1/2}C_{\mathrm t}G_{\mathrm t}^{-1/2}
    \right\|_{\mathrm{op}}
    \le
    \bar a_k
    \le
    \rho.
\]
Furthermore, \(G\succeq G_{\mathrm t}\), so
\[
\begin{aligned}
    \|G^{-1}G_{\mathrm t}\|_{\mathrm{op}}
    &\le
    \sqrt{\operatorname{cond}(G_{\mathrm t})}
    \le C,\\
    \|G_{\mathrm t}^{-1}C_{\mathrm t}\|_{\mathrm{op}}
    &\le
    \bar a_k\sqrt{\operatorname{cond}(G_{\mathrm t})}
    \le C.
\end{aligned}
\]
Consequently,
\[
    \|G^{-1}C_{\mathrm t}\|_{\mathrm{op}}
    \le
    \|G^{-1}G_{\mathrm t}\|_{\mathrm{op}}
    \|G_{\mathrm t}^{-1}C_{\mathrm t}\|_{\mathrm{op}}
    \le C.
\]
Because \(B_*^\top=G^{-1}C\), the head and tail estimates prove
\[
    \|B_*\|_{\mathrm{op}}\le C.
\]

The primitive trace sequences are summable:
\[
    \operatorname{Tr}(\Sigma_0)\lesssim1,
    \qquad
    \operatorname{Tr}(A_*\Sigma_0A_*^\top)\lesssim1,
    \qquad
    \operatorname{Tr}(\Sigma_\xi)\asymp1,
    \qquad
    \operatorname{Tr}(A_*HA_*^\top)\lesssim1,
\]
because \(\alpha,\theta>1\) and
\(\beta_\alpha,\beta_\theta>1\). Gaussian quadratic-form concentration and
the operator bound on \(B_*\) imply
\[
\begin{aligned}
    \operatorname{Tr}(S\Sigma_0S^\top)&\le C,\\
    \operatorname{Tr}(SA_*\Sigma_0A_*^\top S^\top)&\le C,\\
    \operatorname{Tr}(\Omega_\xi)&\asymp1,\\
    \operatorname{Tr}(SA_*HA_*^\top S^\top)&\le C,\\
    \operatorname{Tr}(B_*S\Sigma_0S^\top B_*^\top)&\le C.
\end{aligned}
\]
Therefore, for \(D:=SA_*-B_*S\),
\[
    \operatorname{Tr}(D\Sigma_0D^\top)
    \le
    2\operatorname{Tr}(SA_*\Sigma_0A_*^\top S^\top)
    +2\operatorname{Tr}(B_*S\Sigma_0S^\top B_*^\top)
    \le C.
\]

We now apply the finite-horizon martingale--coboundary decomposition from
Lemma~\ref{lem:nonstationary-score-lower}. For one trajectory and a
deterministic \(v\), write
\[
    T(v)
    :=
    \sum_{p=0}^{P-1}
    \left\{
    e_p(v^\top u_p)-
    \mathbb E[e_p(v^\top u_p)]
    \right\}.
\]
The averaged normal equation implies that the sum of the subtracted
expectations is zero. Thus this is the same trajectory score used in
Lemma~\ref{lem:nonstationary-score-lower}. Retain its notation
\[
    r_m=S^\top f_m,
    \qquad
    d_m=D^\top f_m,
    \qquad
    c=S^\top v,
    \qquad
    F_m=\frac12(d_mc^\top+cd_m^\top),
\]
and its backward recursion
\[
    K_{P,m}=0,
    \qquad
    K_{p,m}=F_m+A_*^\top K_{p+1,m}A_*.
\]
Then
\[
    T_m(v)
    =
    h_{0,m}(x_0)
    +
    \sum_{p=0}^{P-1}M_{p+1,m},
\]
where the initial term and all martingale differences are mutually
orthogonal in \(L^2\). Hence
\[
    \mathbb E\|T(v)\|_2^2
    =
    \sum_{m=1}^M\mathbb E[h_{0,m}(x_0)^2]
    +
    \sum_{p=0}^{P-1}\sum_{m=1}^M
    \mathbb E[M_{p+1,m}^2].
\]

Because \(A_*\) is diagonal, the backward recursion gives entrywise
\[
    (K_{p,m})_{j\ell}
    =
    (F_m)_{j\ell}
    \sum_{q=0}^{P-1-p}(a_ja_\ell)^q.
\]
Thus, for every diagonal covariance \(\Gamma\succeq0\),
\[
    \left\|
    \Gamma^{1/2}K_{p,m}\Gamma^{1/2}
    \right\|_F^2
    \le
    \frac{C}{(1-\rho^2)^2}
    (d_m^\top\Gamma d_m)
    (c^\top\Gamma c).
\]
Indeed, the geometric multiplier is bounded by \((1-\rho^2)^{-1}\), and
the claimed estimate follows by expanding the two rank-one terms in \(F_m\).

Since \(x_0\) is Gaussian,
\[
    \mathbb E[h_{0,m}(x_0)^2]
    =
    2\left\|
    \Sigma_0^{1/2}K_{0,m}\Sigma_0^{1/2}
    \right\|_F^2.
\]
Summing this estimate with \(\Gamma=\Sigma_0\) over output coordinates and
using the preceding bound on \(\operatorname{Tr}(D\Sigma_0D^\top)\) gives
\[
\begin{aligned}
    \sum_{m=1}^M\mathbb E[h_{0,m}(x_0)^2]
    &\le
    C\operatorname{Tr}(D\Sigma_0D^\top)
    v^\top S\Sigma_0S^\top v \\
    &\le
    CP\,v^\top G_0v
    \le
    CP\,v^\top Gv.
\end{aligned}
\]
where the second inequality uses
\(H_0\succeq P^{-1}\Sigma_0\).

For the martingale terms, the exact variance formula from
Lemma~\ref{lem:nonstationary-score-lower} is
\[
\begin{aligned}
    \mathbb E[M_{p+1,m}^2]
    ={}&
    \left\|
    \Sigma_\xi^{1/2}
    \left(r_mc^\top+2K_{p+1,m}A_*\right)
    \Sigma_p^{1/2}
    \right\|_F^2 \\
    &+
    2\left\|
    \Sigma_\xi^{1/2}K_{p+1,m}\Sigma_\xi^{1/2}
    \right\|_F^2.
\end{aligned}
\]
The direct innovation contribution satisfies
\[
\begin{aligned}
    \sum_{m=1}^M
    \left\|
    \Sigma_\xi^{1/2}r_mc^\top\Sigma_p^{1/2}
    \right\|_F^2
    =
    \operatorname{Tr}(\Omega_\xi)
    v^\top S\Sigma_pS^\top v.
\end{aligned}
\]
The one-step smoothing estimate gives
\[
    A_*\Sigma_pA_*^\top
    \preceq
    \Lambda_{\rho,C_{\mathrm{sm}}}^2\Sigma_\xi.
\]
Consequently,
\[
    \left\|
    \Sigma_\xi^{1/2}K_{p+1,m}A_*\Sigma_p^{1/2}
    \right\|_F
    \le
    \Lambda_{\rho,C_{\mathrm{sm}}}
    \left\|
    \Sigma_\xi^{1/2}K_{p+1,m}\Sigma_\xi^{1/2}
    \right\|_F.
\]
Applying the resolvent estimate with \(\Gamma=\Sigma_\xi\), summing over
\(m\), and using \(\Sigma_\xi\preceq CH\), yields
\[
\begin{aligned}
    &\sum_{m=1}^M
    \left\|
    \Sigma_\xi^{1/2}K_{p+1,m}\Sigma_\xi^{1/2}
    \right\|_F^2 \\
    &\qquad\le
    C\operatorname{Tr}(D\Sigma_\xi D^\top)
    v^\top\Omega_\xi v
    \le
    C v^\top Gv.
\end{aligned}
\]
For the final inequality, the normal equation gives
\[
    DHD^\top
    =
    SA_*HA_*^\top S^\top-B_*GB_*^\top
    \preceq
    SA_*HA_*^\top S^\top,
\]
whose trace is constant order on the sketch event because
\(\sum_j a_j^2h_{P,j}<\infty\). Combining the direct and corrector estimates
with
\(\|X+Y\|_F^2\le2\|X\|_F^2+2\|Y\|_F^2\) gives
\[
    \sum_{m=1}^M\mathbb E[M_{p+1,m}^2]
    \le
    C_{\rho,C_{\mathrm{sm}}}
    \left\{
    \operatorname{Tr}(\Omega_\xi)
    v^\top S\Sigma_pS^\top v
    +v^\top Gv
    \right\}.
\]
Summing over positions and using
\(\sum_pS\Sigma_pS^\top=PG\) and
\(\operatorname{Tr}(\Omega_\xi)\asymp1\),
\[
    \sum_{p=0}^{P-1}\sum_{m=1}^M
    \mathbb E[M_{p+1,m}^2]
    \le
    C_{\rho,C_{\mathrm{sm}}}P\,v^\top Gv.
\]
The orthogonal decomposition, initial-boundary estimate, and martingale
estimate prove
\[
    \mathbb E\|T(v)\|_2^2
    \le
    C_{\rho,C_{\mathrm{sm}}}P\,v^\top Gv.
\]

Finally, the complete trajectory scores are centered and independent across
\(i\). Since
\(Qv=(NP)^{-1}\sum_{i=1}^NT_i(v)\),
\[
    \mathbb E\|Qv\|_2^2
    =
    \frac{1}{NP^2}\mathbb E\|T(v)\|_2^2
    \le
    \frac{C_{\rho,C_{\mathrm{sm}}}}{NP}v^\top Gv.
\]
Since this holds for every \(v\in\mathbb R^M\), the claimed Loewner bound
follows.
\end{proof}

\begin{lemma}[Gaussian-chaos fourth-moment control]
\label{lem:score-chaos-fourth-moment}
Conditional on a fixed sketch \(S\), let \(W\) be any deterministic matrix of
compatible dimension. Then
\[
    \mathbb E_{\mathcal D}\|QW\|_F^4
    \le
    C\left(
    \mathbb E_{\mathcal D}\|QW\|_F^2
    \right)^2,
\]
where \(C\) is an absolute constant. In particular, for
\(Y:=\|Qg_L(G)\|_G^2\),
\[
    \mathbb E_{\mathcal D}[Y^2\mid S]
    \le
    C\left(
    \mathbb E_{\mathcal D}[Y\mid S]
    \right)^2.
\]
\end{lemma}

\begin{proof}
Every state is a linear function of the independent Gaussian initial state and
innovations. Consequently, every coordinate of \(QW\) is a polynomial of total
degree at most two in independent Gaussian variables. Scalar Gaussian
hypercontractivity gives
\[
    \mathbb E Z^4
    \le
    C(\mathbb E Z^2)^2
\]
for every such scalar polynomial \(Z\). Writing the squared Frobenius norm as a
sum over coordinates and applying Cauchy--Schwarz to every pair gives
\[
\begin{aligned}
    \mathbb E\|QW\|_F^4
    &=
    \sum_{k,\ell}\mathbb E[Z_k^2Z_\ell^2] \\
    &\le
    \sum_{k,\ell}
    (\mathbb E Z_k^4)^{1/2}
    (\mathbb E Z_\ell^4)^{1/2} \\
    &\lesssim
    \left(\sum_k\mathbb E Z_k^2\right)^2.
\end{aligned}
\]
Taking \(W=g_L(G)G^{1/2}\) proves the final claim.
\end{proof}

\begin{lemma}[High-probability covariance event for stable trajectories]
\label{lem:rnn-covariance-concentration}
Suppose Assumptions~\ref{assump:stable-rnn} and
\ref{assump:power-law-source} hold and \(P\ge2\). Let
\[
    \widehat G
    :=
    \frac1{NP}
    \sum_{i=1}^N\sum_{p=0}^{P-1}
    u_{i,p}u_{i,p}^\top,
    \qquad
    G=\frac1P\sum_{p=0}^{P-1}G_p.
\]
Let \(R=\mathsf H_L\), and define
\[
    \omega_{0,M,P}
    :=
    \sup_{1\le j\le 2M}
    \frac{P^{-1}j^{-\theta}}
    {j^{-\alpha}+P^{-1}j^{-\theta}}
    \asymp
    \begin{cases}
        P^{-1}, & \theta\ge\alpha,\\[1mm]
        \min\{1,M^{\alpha-\theta}/P\}, & \theta<\alpha.
    \end{cases}
\]
There is an event \(\mathcal E_S\), depending only on the Gaussian sketch and
satisfying \(\mathbb P_S(\mathcal E_S)\ge1-\exp(-cM)\), such that, conditional
on every \(S\in\mathcal E_S\), the following holds for every \(t\ge0\):
\[
\begin{aligned}
    \left\|
    G^{-1/2}\widehat G G^{-1/2}-I
    \right\|
    \le{}&
    C\left\{
    \sqrt{\frac{C_\rho(M+t)}{NP}}
    +
    \frac{C_\rho(M+t)}{NP}
    \right.\\
    &\left.
    +\omega_{0,M,P}\sqrt{\frac{M+t}{N}}
    +\omega_{0,M,P}\frac{M+t}{N}
    \right\},
\end{aligned}
\]
with conditional probability at least \(1-2\exp(-ct)\), where
\(C_\rho\asymp(1+\rho)/(1-\rho)\). Consequently,
\(\mathcal E_{\mathrm{cov}}(R)\) holds with probability at least
\(1-\exp(-\Omega(M))\) under either sufficient condition
\[
\begin{array}{ll}
    \theta\ge\alpha:
    &NP\gtrsim C_\rho RM,\\[1mm]
    \alpha-2r\le\theta<\alpha:
    &NP\gtrsim C_\rho RM
    \quad\text{and}\quad N\gtrsim RM.
\end{array}
\]
\end{lemma}

\begin{proof}
We first record the relative size of the initialization component after
sketching. Write
\[
    H_0=\operatorname{diag}(h_{0,P,j}),
    \qquad
    H=\operatorname{diag}(h_{P,j}).
\]
By Lemma~\ref{lem:primitive-positionwise-regularity}, uniformly in \(P\),
\[
    h_{0,P,j}\asymp P^{-1}j^{-\theta},
    \qquad
    h_{P,j}\asymp j^{-\alpha}+P^{-1}j^{-\theta}.
\]
Let \(S_{\le2M}\) and \(S_{>2M}\) denote the corresponding column blocks.
The ambient condition \(d\ge C_{\mathrm{amb}}M\) ensures that the first
block and a further constant-width tail block are available.
Standard rectangular-Gaussian singular-value concentration and a weighted
Gaussian tail bound imply, simultaneously with probability at least
\(1-\exp(-cM)\),
\[
    S_{\le2M}H_{\le2M}S_{\le2M}^\top
    \succeq
    c h_{P,M}I,
    \qquad
    \left\|
    S_{>2M}(H_0)_{>2M}S_{>2M}^\top
    \right\|_{\mathrm{op}}
    \le
    C h_{0,P,M}.
\]
For completeness, the second estimate follows by applying scalar Bernstein to
\(\sum_{j>2M}h_{0,P,j}(a^\top s_j)^2\) on a \(1/4\)-net: its mean is
\(M^{-1}\sum_{j>2M}h_{0,P,j}\lesssim h_{0,P,M}\), and its largest
summand scale is \(h_{0,P,2M}\lesssim h_{0,P,M}\). The first estimate follows
from
\(H_{\le2M}\succeq c f_P(2M)I\) by the coordinate envelope,
\(f_P(2M)\asymp f_P(M)\asymp h_{P,M}\), and a constant lower bound on the smallest
singular value of the \(M\)-by-\(2M\) Gaussian block. These are the same
head--tail estimates used in the proof of the Gaussian-sketch spectral lemma
of \citet{lin2024scaling}.

For \(j\le2M\), the definition of \(\omega_{0,M,P}\) gives
\(h_{0,P,j}\le C\omega_{0,M,P}h_{P,j}\), while
\(h_{0,P,M}/h_{P,M}\le C\omega_{0,M,P}\). Therefore, on the preceding sketch
event,
\[
\begin{aligned}
    G_0
    &\preceq
    C\omega_{0,M,P}
    S_{\le2M}H_{\le2M}S_{\le2M}^\top
    +C h_{0,P,M}I \\
    &\preceq
    C\omega_{0,M,P}G,
\end{aligned}
\]
and hence
\begin{equation}
\label{eq:initial-relative-sketch}
    \left\|G^{-1/2}G_0G^{-1/2}\right\|_{\mathrm{op}}
    \le C\omega_{0,M,P}.
\end{equation}

Fix a unit vector \(a\in\mathbb R^M\) and, for one trajectory, define the
Gaussian vector \(z_i(a)\in\mathbb R^P\) by
\[
    z_{i,p}(a):=a^\top G^{-1/2}u_{i,p},
    \qquad p=0,\ldots,P-1.
\]
Let \(K_a:=\operatorname{Cov}(z_i(a))\), and decompose
\(K_a=K_{a,\xi}+K_{a,0}\) according to the independent innovation and
initialization components. Since \(G=P^{-1}\sum_pG_p\),
\[
    \operatorname{Tr}(K_a)=P.
\]
The within-trajectory bound \(\Sigma_p^{(\xi)}\preceq2H_\xi\) from
Lemma~\ref{lem:primitive-positionwise-regularity}, alignment in the teacher
basis, and geometric decay give
\[
    |(K_{a,\xi})_{pq}|\le2\rho^{|p-q|}.
\]
Indeed, \(A_*^{|p-q|}\Sigma_{\min(p,q)}^{(\xi)}
\preceq2\rho^{|p-q|}H_\xi\), and \(SH_\xi S^\top\preceq G\).
The row-sum bound and the sum of squared entries therefore give
\begin{equation}
\label{eq:innovation-kernel-norms}
    |(K_{a,\xi})_{pq}|
    \le2\rho^{|p-q|},
    \qquad
    \|K_{a,\xi}\|_{\mathrm{op}}
    \le2C_\rho,
    \qquad
    \|K_{a,\xi}\|_F^2
    \le4C_\rho P.
\end{equation}
Moreover, \eqref{eq:initial-relative-sketch} implies
\[
    \operatorname{Tr}(K_{a,0})
    =
    P a^\top G^{-1/2}G_0G^{-1/2}a
    \le
    CP\omega_{0,M,P}.
\]
Since \(K_{a,0}\succeq0\),
\begin{equation}
\label{eq:initial-kernel-norms}
    \|K_{a,0}\|_{\mathrm{op}}
    \le CP\omega_{0,M,P},
    \qquad
    \|K_{a,0}\|_F
    \le CP\omega_{0,M,P}.
\end{equation}

We now use an exact Gaussian quadratic-form inequality. If
\(g\sim\mathcal N(0,K)\), then diagonalizing \(K\) and applying the exponential
moment bound for centered chi-square variables gives, for every \(s\ge0\),
\begin{equation}
\label{eq:gaussian-quadratic-concentration}
    \mathbb P\left(
    |\|g\|_2^2-\operatorname{Tr}(K)|
    >
    2\|K\|_F\sqrt{s}+2\|K\|_{\mathrm{op}}s
    \right)
    \le2e^{-s}.
\end{equation}
Apply \eqref{eq:gaussian-quadratic-concentration} to the block Gaussian vector
formed by the \(N\) independent copies of \(z_i(a)\). Using
\eqref{eq:innovation-kernel-norms}--\eqref{eq:initial-kernel-norms}, the triangle inequality for the
operator and Frobenius norms, and dividing by \(NP\), we obtain
\[
\begin{aligned}
    &\left|
    a^\top G^{-1/2}(\widehat G-G)G^{-1/2}a
    \right| \\
    &\quad\le
    C\left\{
    \sqrt{\frac{C_\rho s}{NP}}
    +\frac{C_\rho s}{NP}
    +\omega_{0,M,P}\sqrt{\frac{s}{N}}
    +\omega_{0,M,P}\frac{s}{N}
    \right\}
\end{aligned}
\]
with probability at least \(1-2e^{-s}\). Notice that the mixed
initialization--innovation terms are already included in
\(\|z_i(a)\|_2^2\); no separate independence or decoupling claim is needed.

Finally, apply this fixed-direction estimate on a \(1/4\)-net of the unit
sphere, whose cardinality is at most \(9^M\), with
\(s=C_1(M+t)\). The standard self-adjoint net inequality
\[
    \|A\|_{\mathrm{op}}
    \le2\sup_{a\in\mathcal N}|a^\top Aa|
\]
gives the asserted operator-norm bound after increasing the constants.

Take \(t\asymp M\). The condition \(NP\gtrsim C_\rho RM\) controls the
innovation terms. If \(\theta<\alpha\), then
\(\omega_{0,M,P}\le1\), and \(N\gtrsim RM\) controls the initialization
terms. If \(\theta\ge\alpha\), then
\(\omega_{0,M,P}\asymp P^{-1}\); the condition on \(NP\) implies
\[
    P^{-1}\sqrt{\frac{M}{N}}\lesssim R^{-1/2},
    \qquad
    P^{-1}\frac{M}{N}\lesssim R^{-1},
\]
so no separate condition on \(N\) is needed. This proves the two event
statements.
\end{proof}

\begin{lemma}[Regularized covariance concentration for stable trajectories]
\label{lem:rnn-regularized-covariance-concentration}
Suppose Assumptions~\ref{assump:stable-rnn} and
\ref{assump:power-law-source} hold and \(P\ge2\). Define
\(\widetilde C_\rho:=(1-\rho)^{-2}\). For any \(\lambda>0\), set
\[
    A_\lambda:=G+\lambda I,
    \qquad F_\lambda:=A_\lambda^{-1/2},
\]
and define
\[
    Q_\xi:=F_\lambda S\Sigma_\xi S^\top F_\lambda,
    \qquad
    Q_0:=P^{-1}F_\lambda S\Sigma_0S^\top F_\lambda,
\]
\[
    v_\lambda:=\frac{\operatorname{Tr}(Q_\xi)}{NP}
      +\frac{\operatorname{Tr}(Q_0)}N,
    \qquad
    h_\lambda:=\frac{\|Q_\xi\|_{\mathrm{op}}}{NP}
      +\frac{\|Q_0\|_{\mathrm{op}}}N.
\]
Conditional on any fixed sketch, for every \(\tau\ge1\),
\[
\begin{aligned}
    \mathbb P_{\mathcal D}\!\left(
    \left\|F_\lambda(\widehat G-G)F_\lambda\right\|_{\mathrm{op}}
    >C\left[
       \sqrt{\widetilde C_\rho(v_\lambda+\tau h_\lambda)}
       +\widetilde C_\rho(v_\lambda+\tau h_\lambda)
    \right]\,\middle|\,S\right)
    \le 2e^{-\tau}.
\end{aligned}
\]
Here \(C\) is universal: no further dependence on \(\rho\) or on
primitive spectral constants is absorbed into it.

On the joint Gaussian-sketch event used throughout the paper, take
\(\lambda=R^{-1}\), let \(k=\kappa_{R,P}\le cM\), and in either regime put
\(\delta=(\alpha-\theta)_+\) and
\[
    \omega_{0,M,P}:=\min\{1,M^\delta/P\},
    \qquad
    q_\tau:=\widetilde C_\rho(k+\tau)
       \left(\frac1{NP}+\frac{\omega_{0,M,P}}N\right).
\]
Then, for a constant \(C_{\mathrm{sp}}\) depending only on the fixed
primitive exponents, comparison constants, and \(\rho\),
\[
    \left\|A_\lambda^{-1/2}(\widehat G-G)A_\lambda^{-1/2}\right\|_{\mathrm{op}}
    \le C_{\mathrm{sp}}(\sqrt{q_\tau}+q_\tau)
\]
with conditional probability at least \(1-2e^{-\tau}\).
Consequently, the regularized event
\[
    \mathcal E_\eta(\lambda)
    :=\left\{
    \left\|A_\lambda^{-1/2}
       (\widehat G-G)A_\lambda^{-1/2}\right\|_{\mathrm{op}}
    \le\eta
    \right\}
\]
implies, by congruence,
\[
    (1-\eta)A_\lambda\preceq\widehat G+\lambda I
       \preceq(1+\eta)A_\lambda.
\]
It
has that probability whenever
\[
    \widetilde C_\rho(k+\tau)
    \left(\frac1{NP}+\frac{\omega_{0,M,P}}N\right)
    \le c_{\mathrm{sp}}\eta^2,\qquad 0<\eta\le1/2,
\]
where \(c_{\mathrm{sp}}>0\) is sufficiently small and has only the same
fixed-primitive dependence. These specialized constants are independent of
\(M,N,P,R,\tau\), but are not asserted to be uniform as \(\rho\uparrow1\).
All statements are uniform over the admissible trajectory lengths.
In Regime I, \(\omega_{0,M,P}=1/P\), so
\[
    q_\tau=2\widetilde C_\rho\frac{k+\tau}{NP},
    \qquad k\asymp R^{1/\alpha}.
\]
Thus the initialization contribution also has the \(1/(NP)\) scale in
this regime; a separate \(1/N\) term is not needed.
\end{lemma}

\begin{proof}
Fix a unit vector \(v\in\mathbb R^M\). Collect the independent standardized
Gaussian initial states and innovations from all \(N\) trajectories into one
standard Gaussian vector \(z\). The exact trajectory representation
\[
    x_{i,p}=A_*^px_{i,0}+\sum_{q=1}^pA_*^{p-q}\xi_{i,q}
\]
defines a deterministic linear map \(T_v\), conditional on \(S\), such that
\[
    T_vz=\frac1{\sqrt{NP}}
      \bigl(v^\top F_\lambda u_{i,p}\bigr)_{i,p}.
\]
Therefore
\[
    \|T_vz\|_2^2-\mathbb E\|T_vz\|_2^2
    =v^\top F_\lambda(\widehat G-G)F_\lambda v.
\]
Let \(\mathcal T=\{T_v:\|v\|_2=1\}\), and denote its Frobenius and
operator radii by \(d_F\) and \(d_{\mathrm{op}}\). Since
\(F_\lambda GF_\lambda\preceq I\),
\[
    d_F^2=\sup_{\|v\|=1}\mathbb E\|T_vz\|_2^2\le1.
\]

Set \(\Delta v=v-w\), and split the standardized input coordinates into
innovations and initial states:
\(T_{\Delta v}=[T_{\Delta v,\xi},T_{\Delta v,0}]\).
Let
\[
    b_\xi=\Sigma_\xi^{1/2}S^\top F_\lambda\Delta v,
    \qquad b_0=\Sigma_0^{1/2}S^\top F_\lambda\Delta v.
\]
Alignment implies that the innovation convolution at lag \(q\) has row
kernel \(b_\xi^\top A_*^q\), of norm at most \(\rho^q\|b_\xi\|_2\).
Discrete Young's inequality on each trajectory block therefore gives
\[
    \|T_{\Delta v,\xi}\|_{\mathrm{op}}^2
    \le\frac1{NP}\left(\sum_{q\ge0}\rho^q\right)^2\|b_\xi\|_2^2
    =\widetilde C_\rho\,\frac{\Delta v^\top Q_\xi\Delta v}{NP}.
\]
For the initial-state block, the row at position \(p\) is
\(b_0^\top A_*^p\), so
\[
    \|T_{\Delta v,0}\|_{\mathrm{op}}^2
    \le\frac{\|b_0\|_2^2}{NP}\sum_{p=0}^{P-1}\rho^{2p}
    \le\frac1{1-\rho^2}\,
       \frac{\Delta v^\top Q_0\Delta v}{N}.
\]
The full map is a direct sum over trajectories; its operator norm is the
largest block norm, so no further factor of \(N\) enters these estimates.
Its output Gram matrix is
\(T_{\Delta v,\xi}T_{\Delta v,\xi}^\top+
T_{\Delta v,0}T_{\Delta v,0}^\top\).
Using \((1-\rho^2)^{-1}\le\widetilde C_\rho\), we obtain
\[
    \|T_v-T_w\|_{\mathrm{op}}^2
    \le \widetilde C_\rho(v-w)^\top
       \left(\frac{Q_\xi}{NP}+\frac{Q_0}{N}\right)(v-w).
\]
In particular, \(d_{\mathrm{op}}^2\le\widetilde C_\rho h_\lambda\).
For \(\mathsf H_\lambda:=Q_\xi/(NP)+Q_0/N\), this metric is dominated by
\(\sqrt{\widetilde C_\rho}
\|\mathsf H_\lambda^{1/2}(v-w)\|_2\).
The majorizing-measure comparison for a Euclidean ellipsoid and
\(\mathbb E\|\mathsf H_\lambda^{1/2}g\|_2
\le\sqrt{\operatorname{Tr}\mathsf H_\lambda}\), for a standard Gaussian
\(g\in\mathbb R^M\), therefore give
\[
    \gamma_2^2(\mathcal T,\|\cdot\|_{\mathrm{op}})
    \le C\widetilde C_\rho\operatorname{Tr}(\mathsf H_\lambda)
    =C\widetilde C_\rho v_\lambda.
\]

The Gaussian-chaos inequality of \citet{krahmer2014suprema}, in the
sharpened form of \citet[Theorem~6.5]{dirksen2015tail}, gives
\[
\begin{aligned}
    \sup_{T\in\mathcal T}
      \left|\|Tz\|_2^2-\mathbb E\|Tz\|_2^2\right|
    \lesssim{}&
      \gamma_2^2+d_F\gamma_2
      +\sqrt\tau\,d_Fd_{\mathrm{op}}
      +\tau d_{\mathrm{op}}^2
\end{aligned}
\]
outside an event of probability at most \(2e^{-\tau}\).
The squared Schatten-four radius in that theorem is at most
\(d_Fd_{\mathrm{op}}\), since
\(\|T\|_{S_4}^2\le\|T\|_F\|T\|_{\mathrm{op}}\).
Substituting the three complexity bounds and combining
\(\sqrt{v_\lambda}+\sqrt{\tau h_\lambda}
\lesssim\sqrt{v_\lambda+\tau h_\lambda}\) proves the first claim.
All constants up to this point are universal.
Compactness and continuity allow finite nets to recover the supremum.

We now specialize the complexity parameters. Since \(P\ge2\),
\(H\succeq H_\xi\succeq\Sigma_\xi/2\), so
\[
    Q_\xi\preceq2F_\lambda GF_\lambda.
\]
For Regime I, \(\theta\ge\alpha\) implies
\(\Sigma_0\preceq C\Sigma_\xi\). Consequently,
\[
    Q_0\preceq \frac{C}{P}F_\lambda GF_\lambda,
\]
using \(H\succeq\Sigma_\xi/2\). For Regime II, the initialization
comparison in the proof of Lemma~\ref{lem:rnn-covariance-concentration}
and \(P^{-1}\Sigma_0\preceq H_0\) give
\[
    P^{-1}S\Sigma_0S^\top
       \preceq C_{\mathrm{init}}\omega_{0,M,P}G.
\]
This latter comparison can alternatively be obtained by sandwiching the
indexed head and tail of \(P^{-1}\Sigma_0\) against those of \(F_P\),
as in that proof. Thus in both regimes
\(Q_0\preceq C\omega_{0,M,P}F_\lambda GF_\lambda\).
The spectral envelope comparison in Lemma~\ref{lem:lin-sketch-transfer}
and the cutoff relation give
\[
    \operatorname{Tr}[G(G+\lambda I)^{-1}]\asymp k.
\]
Indeed, split \(\sum_j\mu_j/(\mu_j+\lambda)\) at \(j\asymp k\) and
use integral comparison for the two powers. Strict unsaturation provides
the retained head needed for the matching lower bound.
Consequently, for a fixed-primitive constant \(C_0\),
\[
\begin{aligned}
    \operatorname{Tr}(Q_\xi)&\le C_0k,
    &\|Q_\xi\|_{\mathrm{op}}&\le2,\\
    \operatorname{Tr}(Q_0)&\le C_0\omega_{0,M,P}k,
    &\|Q_0\|_{\mathrm{op}}&\le C_0\omega_{0,M,P}.
\end{aligned}
\]
Here \(C_{\mathrm{init}}\), \(C_0\), and the effective-dimension
comparison constants may depend on the fixed primitive data, including
\(\rho\); they are not universal. Substitution into the universal
concentration estimate gives the stated \(C_{\mathrm{sp}}\) bound.
If \(q_\tau\le c_{\mathrm{sp}}\eta^2\) and \(\eta\le1/2\), its linear
term is also absorbed, proving the event assertion.
\end{proof}

\begin{lemma}[Safeguarded stability and empirical cross-covariance moments]
\label{lem:safeguarded-stability-moments}
Suppose Assumptions~\ref{assump:stable-rnn} and
\ref{assump:power-law-source} hold and \(P\ge2\). On the Gaussian-sketch
event of Lemma~\ref{lem:lin-sketch-transfer}, there are constants
\(C_G,C>0\) such that
\[
    \|G\|_{\mathrm{op}}\le C_G,
    \qquad
    \mathbb E_{\mathcal D}
    [\|\widehat C\|_F^4\mid S]\le C,
    \qquad
    \|B_*\|_F^2\le CM.
\]
For every trajectory sample, the safeguarded GD schedule satisfies
\[
    0\preceq I-\bar\eta_t\widehat G\preceq I,
    \qquad
    \sum_{t=1}^L\bar\eta_t\le R,
    \qquad
    \|B_L\|_F\le R\|\widehat C\|_F.
\]
Moreover, after decreasing the deterministic constant in the step-size
assumption, \(a_{\mathrm{safe}}=1\) both on
\(\mathcal E_{\mathrm{cov}}(R)\) and on
\(\mathcal E_\eta(R^{-1})\) whenever \(\eta\le1/2\).
\end{lemma}

\begin{proof}
The sketch-transfer event gives
\(\mu_j(G)\asymp j^{-\alpha}+P^{-1}j^{-\theta}\). Hence
\(\|G\|_{\mathrm{op}}\le C_G\) and, because \(\alpha,\theta>1\),
\[
    \operatorname{Tr}(G)
    \asymp
    \sum_{j=1}^M
    \left(j^{-\alpha}+P^{-1}j^{-\theta}\right)
    \asymp1.
\]
Write \(Z_{i,p}:=y_{i,p}u_{i,p}^\top\), so that
\(\widehat C=(NP)^{-1}\sum_{i,p}Z_{i,p}\). Convexity of
\(A\mapsto\|A\|_F^4\) gives
\[
    \mathbb E_{\mathcal D}[\|\widehat C\|_F^4\mid S]
    \le
    \frac1{NP}
    \sum_{i=1}^N\sum_{p=0}^{P-1}
    \mathbb E[\|Z_{i,p}\|_F^4\mid S].
\]
Since \(\|Z_{i,p}\|_F^4=\|y_{i,p}\|_2^4\|u_{i,p}\|_2^4\),
Cauchy--Schwarz yields
\[
    \mathbb E\|Z_{i,p}\|_F^4
    \le
    \left(
    \mathbb E\|y_{i,p}\|_2^8
    \mathbb E\|u_{i,p}\|_2^8
    \right)^{1/2}.
\]
The vectors are Gaussian. The exact covariance recursion and
\(\alpha,\theta>1\) give uniformly bounded traces for their positionwise
covariances on the sketch event. The standard Gaussian moment bound
\(\mathbb E\|z\|_2^8\lesssim\operatorname{Tr}(\operatorname{Cov}(z))^4\)
therefore makes the preceding display uniformly bounded, proving the fourth
moment claim.

Lemma~\ref{lem:variance-residual-moments} gives
\(\|B_*\|_{\mathrm{op}}\le C\) on the joint sketch event. Consequently,
\[
    \|B_*\|_F^2
    \le
    M\|B_*\|_{\mathrm{op}}^2
    \lesssim
    M.
\]

By construction,
\(\bar\eta_t\|\widehat G\|_{\mathrm{op}}\le1/2\), which proves the
contraction inequalities. Since \(a_{\mathrm{safe}}\le1\),
\(\sum_t\bar\eta_t\le\sum_t\eta_t=R\). Since
\(B_0=0\), iteration of GD gives
\[
    B_L
    =
    \widehat C
    \sum_{t=1}^L
    \bar\eta_t
    \prod_{s=t+1}^L(I-\bar\eta_s\widehat G).
\]
Every product has operator norm at most one, so
\(\|B_L\|_F\le R\|\widehat C\|_F\).

Finally, on \(\mathcal E_{\mathrm{cov}}(R)\),
\(\widehat G\preceq2G\). On \(\mathcal E_\eta(R^{-1})\),
\[
    \widehat G+R^{-1}I
    \preceq(1+\eta)(G+R^{-1}I),
\]
so \(\widehat G\preceq(1+\eta)G+\eta R^{-1}I\). Since
\(R\gtrsim1\), both events therefore imply
\(\|\widehat G\|_{\mathrm{op}}\le C\) uniformly in
\(M,N,P,L\). Choosing the deterministic peak step size below the reciprocal
of twice this fixed bound proves that the safeguard is inactive on either
good event.
\end{proof}

\begin{lemma}[Two-scale induced covariance and source]
\label{lem:primitive-to-H-power-law}
Suppose Assumptions~\ref{assump:stable-rnn} and
\ref{assump:power-law-source} hold and \(P\ge2\). Then
\[
    h_{P,j}:=e_j^\top He_j
    \asymp
    j^{-\alpha}+P^{-1}j^{-\theta},
\]
and the source energy
\[
    \tau_{P,j}
    :=
    h_{P,j}a_j^2
    \asymp
    j^{-\beta_\alpha}+P^{-1}j^{-\beta_\theta}.
\]
If \(\theta<\alpha\), define
\[
    m_P:=P^{1/(\alpha-\theta)}.
\]
Then \(h_{P,j}\asymp j^{-\alpha}\) for \(j\lesssim m_P\), while
\(h_{P,j}\asymp P^{-1}j^{-\theta}\) for \(j\gtrsim m_P\).
\end{lemma}

\begin{proof}
By Lemma~\ref{lem:primitive-positionwise-regularity},
\[
    h_{P,j}
    \asymp
    P^{-1}j^{-\theta}+j^{-\alpha}.
\]
Multiplication by \(a_j^2\asymp j^{-2r}\) gives the source formula. In the
heavier-initialization regime, the two covariance terms are equal when
\(j^{\alpha-\theta}\asymp P\), which gives \(m_P\) and the two cases.
\end{proof}

\begin{lemma}[Mixed-spectrum optimization cutoff]
\label{lem:two-scale-cutoff}
For \(R\gtrsim1\), define \(\kappa_{R,P}\) by
\[
    R\left(
    \kappa_{R,P}^{-\alpha}
    +P^{-1}\kappa_{R,P}^{-\theta}
    \right)
    \asymp1.
\]
If \(\theta<\alpha\), let
\(R_P:=P^{\alpha/(\alpha-\theta)}\). Then
\[
    \kappa_{R,P}
    \asymp
    \begin{cases}
        R^{1/\alpha}, & R\lesssim R_P,\\[1mm]
        (R/P)^{1/\theta}, & R\gtrsim R_P.
    \end{cases}
\]
Moreover, for the mixed sequence
\(\mu_j\asymp j^{-\alpha}+P^{-1}j^{-\theta}\),
\[
    \sum_{j=1}^M\min\{1,(R\mu_j)^2\}
    \asymp
    \min\{M,\kappa_{R,P}\}.
\]
\end{lemma}

\begin{proof}
If \(\theta\ge\alpha\), the initialization term is dominated uniformly by
the innovation term, and hence \(\kappa_{R,P}\asymp R^{1/\alpha}\). If
\(\theta<\alpha\), the cutoff lies before the spectral crossover when
\(R m_P^{-\alpha}\lesssim1\), equivalently \(R\lesssim R_P\), and after it
when \(R\gtrsim R_P\). Substitution gives the two displayed formulas. For
the effective dimension, split the sum at
\(j\asymp\kappa_{R,P}\). The head contributes
\(\min\{M,\kappa_{R,P}\}\); the tail has the same order by integral
comparison, using \(\alpha,\theta>1\) and the doubling property of the mixed
sequence.
\end{proof}

\begin{lemma}[Joint Gaussian-sketch spectral transfer]
\label{lem:lin-sketch-transfer}
Suppose Assumptions~\ref{assump:stable-rnn} and
\ref{assump:power-law-source} hold and \(S\) is a Gaussian sketch with entries
\(N(0,1/M)\). Let
\[
    G:=SHS^\top,
    \qquad
    C:=SA_*HS^\top,
\]
and denote their eigenvalues in non-increasing order by
\(\mu_j(G)\) and \(\nu_j(C)\), respectively. Then, with probability at least
\(1-\exp(-cM)\), simultaneously for \(j=1,\ldots,M\),
\[
    \mu_j(G)
    \asymp
    j^{-\alpha}+P^{-1}j^{-\theta},
\]
\[
    \nu_j(C)
    \asymp
    j^{-(\alpha+r)}+P^{-1}j^{-(\theta+r)}.
\]
Moreover, for every integer \(0\le k\le M/3\), the tail covariance
\[
    G_{>k}:=S_{>k}H_{>k}S_{>k}^\top
\]
satisfies
\[
    \frac{\mu_{M/2}(G_{>k})}{\mu_M(G_{>k})}
    \le C.
\]
The event may be chosen so that the primitive weighted trace estimates used
below hold simultaneously. In particular, for every \(1\le k\le M/3\),
\[
    \operatorname{Tr}
    \left(S_{\le k}A_{*,\le k}^2S_{\le k}^\top\right)
    \le
    C\sum_{j\le k}a_j^2,
\]
\[
    \operatorname{Tr}
    \left(S_{>k}A_{*,>k}^2H_{>k}S_{>k}^\top\right)
    \le
    C\sum_{j>k}a_j^2h_{P,j},
\]
and \(\operatorname{Tr}(S\Sigma_\xi S^\top)\asymp1\). The event may also be
chosen to be contained in the sketch event of
Lemma~\ref{lem:variance-residual-moments}. All constants are uniform over
\(P\ge2\).
\end{lemma}

\begin{proof}
We give the Gaussian head--tail argument directly, following the method of
\citet[Appendix~G]{lin2024scaling}, and keep the original teacher coordinates.
For each exponent in the finite set
\[
    \mathcal Q:=\{\alpha,\theta,\alpha+r,\theta+r,
                    \beta_\alpha,\beta_\theta\},
\]
write \(Z_q:=SF_qS^\top\). Every \(q\in\mathcal Q\) is larger than one.
Integral comparison and a weighted Gaussian quadratic-form bound on a
\(1/4\)-net give, simultaneously for \(1\le m\le M\),
\[
\begin{aligned}
    \|S_{>m}F_{q,>m}S_{>m}^\top\|_{\mathrm{op}}
    &\lesssim
       \frac1M\sum_{\ell>m}\ell^{-q}
       +\sqrt{\frac1M\sum_{\ell>m}\ell^{-2q}}
       +(m+1)^{-q}\\
    &\lesssim m^{-q}.
\end{aligned}
\]
Indeed, the three terms are bounded respectively by
\(M^{-1}m^{1-q}\), \(M^{-1/2}m^{1/2-q}\), and \(m^{-q}\);
\(m\le M\) gives the last inequality. The net has at most \(9^M\) points,
and choosing a sufficiently large fixed multiple of \(M\) as its confidence
parameter gives failure probability \(\exp(-cM)\).

For \(j\ge2\), split \(Z_q\) after \(m=\lfloor j/2\rfloor\).
The head has rank at most \(m\le j-1\), so the min--max principle and the
tail bound give \(\mu_j(Z_q)\lesssim j^{-q}\).
For \(j=1\), the first-column norm and the tail bound with \(m=1\) give
\(\mu_1(Z_q)\lesssim1\).
For \(j\le M/3\), rectangular Gaussian concentration yields
\(S_{\le j}^\top S_{\le j}\succeq cI_j\), whence
\(\mu_j(Z_q)\gtrsim j^{-q}\).
For \(j>M/3\), retain the \(2M\) columns indexed by
\(J=\{M+1,\ldots,3M\}\). On the same Gaussian event,
\[
    Z_q\succeq(3M)^{-q}S_JS_J^\top
          \succeq cM^{-q}I,
\]
which gives the lower bound for the remaining indices. The standing
ambient-dimension constant can accommodate this fixed block. A finite
intersection over \(\mathcal Q\) and a union over \(O(M)\) splits give a
single event on which
\[
    \mu_j(Z_q)\asymp j^{-q},\qquad 1\le j\le M,\quad q\in\mathcal Q.
\]

Now \(SF_PS^\top=Z_\alpha+P^{-1}Z_\theta\). PSD ordering gives
\[
    \mu_j(SF_PS^\top)
    \ge\max\{\mu_j(Z_\alpha),P^{-1}\mu_j(Z_\theta)\}
    \gtrsim f_P(j).
\]
For the upper bound, choose \(i=\lceil j/2\rceil\) and
\(\ell=j-i+1\) in Weyl's inequality:
\[
    \mu_j(SF_PS^\top)
    \le\mu_i(Z_\alpha)+P^{-1}\mu_\ell(Z_\theta)
    \lesssim f_P(j).
\]
This also covers \(j=1\). The deterministic sandwich
\(cSF_PS^\top\preceq G\preceq CSF_PS^\top\) proves the claim for \(G\).
Apply exactly the same argument to
\(Z_{\alpha+r}+P^{-1}Z_{\theta+r}\) and the envelope of \(A_*H\) to prove
the claim for \(C\). These comparisons hold for every \(P\ge2\) on the
same event; no coordinate permutation is used.

For the tail ratio, take \(b=\lfloor M/4\rfloor\).
For every \(k\le M/3\), the part of \(G_{>k}\) indexed by
\(\{k+1,\ldots,k+b\}\) has rank at most \(b\).
The remainder is bounded above by
\[
    C\{(k+b)^{-\alpha}+P^{-1}(k+b)^{-\theta}\}I
    \preceq C f_P(M)I
\]
using the two primitive envelope tail bounds. Thus
\(\mu_{\lfloor M/2\rfloor}(G_{>k})\lesssim f_P(M)\).
The block \(J\) above is contained in every such indexed tail and gives
\(G_{>k}\succeq c f_P(M)I\). This proves the ratio, including \(k=0\);
integer roundings are immaterial for sufficiently large \(M\).

We also include a geometric event used in the aggregate source estimate.
For \(1\le k\le M/3\), let \(\mathsf P_k\) project onto the orthogonal
complement of the columns indexed by
\(\{k+1,\ldots,k+b'\}\), where \(b'=\lfloor M/2\rfloor\).
This projector is independent of \(S_{\le k}\), has rank \(M-b'\), and
does not depend on \(P\). Conditional rectangular Gaussian concentration,
followed by a union bound over \(k\), gives
\[
    S_{\le k}^\top\mathsf P_kS_{\le k}\succeq cI_k,
    \qquad \|S_{\le k}\|_{\mathrm{op}}\le C
\]
on a further event of failure probability \(\exp(-cM)\).

Finally, scalar weighted Gaussian-trace concentration applies to the head
weights \(j^{-2r}\) and the tail weights
\(j^{-\beta_\alpha},j^{-\beta_\theta}\).
Their envelope comparisons imply all displayed actual-weight trace bounds,
uniformly in \(P\); \(\operatorname{Tr}(S\Sigma_\xi S^\top)\asymp1\)
follows in the same way. Intersect with the envelope events in the proof of
Lemma~\ref{lem:variance-residual-moments}. All intersections involve only
finitely many primitive exponents and \(O(M)\) indexed blocks, and preserve
probability \(1-\exp(-cM)\), independently of \(P\).
\end{proof}

\begin{lemma}[Aggregate Gaussian-sketch transfer for GD bias]
\label{lem:aggregate-bias-transfer}
On the event of Lemma~\ref{lem:lin-sketch-transfer}, let
\[
    G=\sum_{j=1}^M\mu_jv_jv_j^\top,
    \qquad
    q_j:=\mu_j\|B_*v_j\|_2^2.
\]
For every \(1\le k\le M/3\) and \(R\gtrsim1\),
\[
\begin{aligned}
    \sum_{j=1}^M
    q_j\min\{1,(R\mu_j)^{-1}\}
    \lesssim{}&
    \frac1R k^{1-2r}
    +
    k^{1-\beta_\alpha}
    +
    P^{-1}k^{1-\beta_\theta}.
\end{aligned}
\]
Let \(k\asymp\kappa_{R,P}\), so that \(Rh_{P,k}\asymp1\). If
\(k\le cM\), then, for a sufficiently small constant \(c_0>0\),
\[
    \sum_{j:R\mu_j\le c_0}q_j
    \gtrsim
    k^{1-\beta_\alpha}
    +
    P^{-1}k^{1-\beta_\theta}.
\]
\end{lemma}

\begin{proof}
For the upper bound, define the spectral filter
\[
    \phi_R(s)^2:=\min\{1,(Rs)^{-1}\}.
\]
Then
\[
    \sum_{j=1}^Mq_j\phi_R(\mu_j)^2
    =
    \|B_*\phi_R(G)\|_G^2.
\]
Let \(f_m\) be the \(m\)-th standard basis vector of \(\mathbb R^M\), and
define the ambient scalar target
\[
    w_m:=A_*S^\top f_m.
\]
Its population sketched coefficient is
\[
    B_*^\top f_m
    =
    G^{-1}SHw_m.
\]
Split \(w_m=w_{m,\le k}+w_{m,>k}\). Since
\(0\preceq G\phi_R(G)^2\preceq R^{-1}I\), the head contribution is bounded
by
\[
    \frac{C}{R}\|w_{m,\le k}\|_2^2
\]
provided
\(\|G^{-1}S_{\le k}H_{\le k}\|_{\mathrm{op}}\lesssim1\). To verify this
uniformly, set \(T_k:=G_{>k}\) and use Woodbury's identity:
\[
    G^{-1}S_{\le k}H_{\le k}
    =
    T_k^{-1}S_{\le k}
    \left(
    H_{\le k}^{-1}
    +S_{\le k}^\top T_k^{-1}S_{\le k}
    \right)^{-1}.
\]
Use the projector \(\mathsf P_k\) in the proof of
Lemma~\ref{lem:lin-sketch-transfer}. Since it annihilates the next
\(\lfloor M/2\rfloor\) tail columns, the primitive envelope tail bounds give
\[
    \mathsf P_kT_k\mathsf P_k
       \preceq C f_P(M)\mathsf P_k,\qquad
    T_k\succeq c f_P(M)I.
\]
The inverse-compression inequality, obtained by restricting the variational
formula for \(x^\top T_k^{-1}x\) to \(\operatorname{ran}\mathsf P_k\), gives
\[
    T_k^{-1}\succeq
    \mathsf P_k
       \bigl(\mathsf P_kT_k\mathsf P_k
          |_{\operatorname{ran}\mathsf P_k}\bigr)^{-1}
       \mathsf P_k
    \succeq \frac{c}{f_P(M)}\mathsf P_k.
\]
Consequently,
\[
    S_{\le k}^\top T_k^{-1}S_{\le k}
       \succeq \frac{c}{f_P(M)}I_k.
\]
Woodbury's identity, \(\|T_k^{-1}\|\lesssim f_P(M)^{-1}\), and
\(\|S_{\le k}\|\lesssim1\) therefore yield
\[
    \|G^{-1}S_{\le k}H_{\le k}\|_{\mathrm{op}}
    \lesssim1
\]
simultaneously in \(k\) and \(P\). The geometric event is independent of
the covariance weights, so no concentration on a \(P\)-dependent random
eigenspace is needed.

For the tail, use
\(0\preceq G\phi_R(G)^2\preceq G\) and
\[
    \left\|
    H_{>k}^{1/2}S_{>k}^\top G^{-1}S_{>k}H_{>k}^{1/2}
    \right\|_{\mathrm{op}}
    \le1.
\]
This bounds its contribution by \(\|w_{m,>k}\|_{H_{>k}}^2\). Summing over
output coordinates gives
\[
\begin{aligned}
    \|B_*\phi_R(G)\|_G^2
    \lesssim{}&
    \frac1R
    \operatorname{Tr}
    \left(S_{\le k}A_{*,\le k}^2S_{\le k}^\top\right)\\
    &+
    \operatorname{Tr}
    \left(S_{>k}A_{*,>k}^2H_{>k}S_{>k}^\top\right).
\end{aligned}
\]
The simultaneous weighted trace estimates on the sketch event yield
\[
    \operatorname{Tr}
    \left(S_{\le k}A_{*,\le k}^2S_{\le k}^\top\right)
    \lesssim
    \sum_{j\le k}a_j^2
    \asymp
    k^{1-2r},
\]
and
\[
\begin{aligned}
    \operatorname{Tr}
    \left(S_{>k}A_{*,>k}^2H_{>k}S_{>k}^\top\right)
    &\lesssim
    \sum_{j>k}a_j^2h_{P,j}\\
    &\asymp
    k^{1-\beta_\alpha}
    +P^{-1}k^{1-\beta_\theta}.
\end{aligned}
\]
This proves the upper bound.

For the lower bound, let
\[
    \Pi_R:=\mathbf 1_{\{G\preceq c_0R^{-1}I\}}.
\]
Since \(B_*=CG^{-1}\),
\[
    \sum_{j:R\mu_j\le c_0}q_j
    =
    \operatorname{Tr}(\Pi_RG^{-1}C^2).
\]
On the range of \(\Pi_R\),
\(G^{-1}\succeq cR I\), and therefore
\[
    \operatorname{Tr}(\Pi_RG^{-1}C^2)
    =
    \operatorname{Tr}(C\Pi_RG^{-1}\Pi_RC)
    \gtrsim
    R\operatorname{Tr}(\Pi_RC^2).
\]
The eigenvalue estimate for \(G\) shows that \(\Pi_R\) has rank at least
\(M-C_0k\). Ky Fan's minimum principle and the spectral estimate for \(C\)
then imply
\[
\begin{aligned}
    \operatorname{Tr}(\Pi_RC^2)
    &\ge
    \sum_{j>C_0k}\nu_j(C)^2\\
    &\gtrsim
    k\left(a_kh_{P,k}\right)^2,
\end{aligned}
\]
where \(k\le cM\) ensures that the retained annulus has length comparable to
\(k\). Finally, \(Rh_{P,k}\asymp1\) gives
\[
\begin{aligned}
    Rk(a_kh_{P,k})^2
    &\asymp
    ka_k^2h_{P,k}\\
    &\asymp
    k^{1-\beta_\alpha}
    +P^{-1}k^{1-\beta_\theta}.
\end{aligned}
\]
This proves the lower bound.
\end{proof}

\begin{lemma}[Scalar GD filters and noncommutative empirical transfer]
\label{lem:lin-gd-filters}
Suppose Assumption~\ref{assump:stepsize} holds, \(G\succ0\), and
\(\|G\|\lesssim1\). For the deterministic schedule in
Section~\ref{sec:preliminaries}, let
\[
    \psi_L(s):=\prod_{t=1}^L(1-\eta_ts),
    \qquad
    g_L(s):=\sum_{t=1}^L\eta_t\prod_{q=t+1}^L(1-\eta_qs).
\]
Then, for the cumulative horizon \(R=\sum_t\eta_t\) and every \(s\ge0\)
satisfying \(\eta_{\max}s\le1/2\),
\[
    e^{-2Rs}\le\psi_L(s)\le e^{-Rs},\qquad
    1-\psi_L(s)\asymp\min\{1,Rs\},
\]
and
\[
    \psi_L(s)^2\lesssim \min\{1,(Rs)^{-1}\},
    \qquad
    s^2g_L(s)^2\asymp \min\{1,(Rs)^2\}.
\]
Let \(\widehat G\succ0\) satisfy
\[
    \delta
    :=
    \left\|G^{-1/2}\widehat G G^{-1/2}-I\right\|_{\mathrm{op}}
    \le
    c_{\mathrm{rel}}R^{-1/2},
\]
with \(c_{\mathrm{rel}}\) sufficiently small, and set \(\lambda:=R^{-1}\).
Define
\[
    \widehat\Psi_L:=\psi_L(\widehat G),
    \qquad
    \widetilde g_L:=g_L(\widehat G),
    \qquad
    K_\lambda(G):=G(G+\lambda I)^{-2}.
\]
Then the following deterministic Loewner comparisons hold:
\[
    \widehat\Psi_LG\widehat\Psi_L
    \preceq
    C\lambda G(G+\lambda I)^{-1},
\]
and
\[
    cK_\lambda(G)
    \preceq
    \widetilde g_LG\widetilde g_L
    \preceq
    CK_\lambda(G).
\]
In addition, if
\[
    \Pi_{\mathrm{low}}
    :=
    \mathbf 1_{[0,c_0\lambda]}(G),
\]
then
\[
    \left\|
    (I-\widehat\Psi_L)\Pi_{\mathrm{low}}
    \right\|_{\mathrm{op}}
    \le
    C\sqrt{c_0}.
\]
\end{lemma}

\begin{proof}
For \(0\le z\le1/2\), \(-2z\le\log(1-z)\le-z\). Summing
with \(z=\eta_ts\) proves
\[
    e^{-2Rs}\le\psi_L(s)\le e^{-Rs}.
\]
Therefore
\[
    \psi_L(s)^2
    \lesssim
    \min\{1,(Rs)^{-1}\}.
\]
The telescoping identity gives \(sg_L(s)=1-\psi_L(s)\). Moreover,
\[
    1-e^{-Rs}
    \le
    1-\psi_L(s)
    \le
    \min\{1,Rs\},
\]
and hence
\[
    s^2g_L(s)^2
    \asymp
    \min\{1,(Rs)^2\}.
\]
Equivalently, up to constants,
\[
    s\psi_L(s)^2
    \lesssim
    \frac{\lambda s}{s+\lambda},
    \qquad
    sg_L(s)^2
    \asymp
    \frac{s}{(s+\lambda)^2}.
\]

We now prove the matrix transfer without assuming that \(G\) and
\(\widehat G\) commute. The relative event gives
\[
    (1-\delta)G
    \preceq
    \widehat G
    \preceq
    (1+\delta)G.
\]
Since \(G\preceq(1-\delta)^{-1}\widehat G\), congruence by
\(\widehat\Psi_L\) and the first scalar estimate above
give
\[
\begin{aligned}
    \widehat\Psi_LG\widehat\Psi_L
    &\preceq
    C\widehat\Psi_L\widehat G\widehat\Psi_L\\
    &\preceq
    C\lambda\widehat G(\widehat G+\lambda I)^{-1}.
\end{aligned}
\]
The map \(X\mapsto\lambda X(X+\lambda I)^{-1}\) is operator monotone on
positive semidefinite matrices. Applying the upper comparison in
the relative covariance display, followed by scalar comparison of
\(f((1+\delta)G)\) and \(f(G)\), proves
\[
    \widehat\Psi_LG\widehat\Psi_L
    \preceq
    C\lambda G(G+\lambda I)^{-1}.
\]

For the response filter, congruence in the relative covariance display and the
second scalar estimate above first give
\[
    \widetilde g_LG\widetilde g_L
    \asymp
    \widetilde g_L\widehat G\widetilde g_L
    \asymp
    K_\lambda(\widehat G).
\]
It remains to compare the two ridge-response kernels. Set
\[
    A:=G+\lambda I,
    \qquad
    \widehat A:=\widehat G+\lambda I,
    \qquad
    D:=A^{-1/2}GA^{-1/2}.
\]
Writing
\[
    \Delta:=G^{-1/2}(\widehat G-G)G^{-1/2},
    \qquad
    E:=A^{-1/2}(\widehat G-G)A^{-1/2}
    =D^{1/2}\Delta D^{1/2},
\]
we have \(\|E\|\le\delta\). Let \(T:=(I+E)^{-1}\). Then
\[
    \widehat A^{-1}G\widehat A^{-1}
    =
    A^{-1/2}TDT A^{-1/2}.
\]
The weighted similarity transform
\[
    L:=D^{1/2}TD^{-1/2}
\]
satisfies
\[
    L
    =
    I-D^{1/2}TD^{1/2}\Delta,
    \qquad
    L^{-1}
    =
    I+D\Delta.
\]
Hence \(\|L\|+\|L^{-1}\|\le C\), and therefore
\[
    TDT
    =
    D^{1/2}L^\top LD^{1/2},
    \qquad
    cD\preceq TDT\preceq CD.
\]
Consequently,
\[
    \widehat A^{-1}G\widehat A^{-1}
    \asymp
    A^{-1}GA^{-1}
    =
    K_\lambda(G).
\]
Using the relative covariance comparison once more inside congruence by
\(\widehat A^{-1}\) yields
\[
    K_\lambda(\widehat G)
    =
    \widehat A^{-1}\widehat G\widehat A^{-1}
    \asymp
    \widehat A^{-1}G\widehat A^{-1}
    \asymp
    K_\lambda(G).
\]
Together with the preceding response-filter display, this proves the two-sided response-filter
comparison.

Finally, the telescoping identity
\(I-\widehat\Psi_L=\widehat G\widetilde g_L\), contraction of every GD factor,
and \(\sum_t\eta_t\le R\) imply
\[
    0\preceq I-\widehat\Psi_L\preceq R\widehat G,
    \qquad
    (I-\widehat\Psi_L)^2\preceq I-\widehat\Psi_L.
\]
Therefore
\[
\begin{aligned}
    &\Pi_{\mathrm{low}}
    (I-\widehat\Psi_L)^2
    \Pi_{\mathrm{low}}\\
    &\qquad\preceq
    R\Pi_{\mathrm{low}}\widehat G\Pi_{\mathrm{low}}
    \preceq
    Cc_0\Pi_{\mathrm{low}},
\end{aligned}
\]
which proves the final norm estimate. This estimate alone is not used to claim
a global empirical bias lower comparison; the source-aware trace argument in
the next lemma supplies the required lower bound.
\end{proof}

\begin{lemma}[Regularized finite-step bias-filter transfer]
\label{lem:regularized-bias-filter-transfer}
Let \(G,\widehat G,D\succeq0\), \(D\preceq CG\), and
\(\|G\|_{\mathrm{op}}\le C_G\), where \(C_G\) is fixed.
Use the deterministic schedule of Section~\ref{sec:preliminaries}, with \(R=\sum_t\eta_t\gtrsim1\),
\(\lambda=R^{-1}\), and \(A_\lambda=G+\lambda I\). Suppose
\[
    \left\|A_\lambda^{-1/2}
      (\widehat G-G)A_\lambda^{-1/2}\right\|_{\mathrm{op}}
    \le\eta\le\frac12,
\]
and the deterministic peak step size is small enough that the safeguard is
inactive. For every matrix \(B\) that is deterministic conditional on the
sketch,
\[
    \boxed{
    \|B[\psi_L(\widehat G)-\psi_L(G)]\|_D
    \lesssim
    \eta\frac{\log(e+R)}{\sqrt R}\|B\|_F.}
\]
The implicit constant depends only on \(C_G\), the covariance-comparison
constant, and the fixed step-size bound.
No commutation among \(G,\widehat G,D\) is assumed.
\end{lemma}

\begin{proof}
For a contiguous segment of total step mass \(h\le R\), let \(\phi_h\)
be its scalar product. On either admissible spectrum, stability gives
\(0\le\phi_h(x)\le e^{-hx}\). Since the spectra are bounded by a fixed
constant and \(R\gtrsim1\), elementary maximization yields
\[
    \sup_x(x+R^{-1})\phi_h(x)^2\lesssim(1+h)^{-1},
\]
\[
    \sup_x x(x+R^{-1})\phi_h(x)^2\lesssim(1+h)^{-2}.
\]
The first bound uses \(h\le R\); the constants depend only on the fixed
spectral and step-size bounds. These estimates apply separately to the
population and empirical spectra, without commuting their covariances.
For each step let
\[
    u_t:=\sum_{j<t}\eta_j,
    \qquad v_t:=\sum_{j>t}\eta_j,
\]
and let \(\widehat\Psi_{>t}\) and \(\Psi_{<t}\) denote the corresponding
ordered empirical suffix and population prefix products. The exact product
identity is
\[
    \psi_L(\widehat G)-\psi_L(G)
    =-\sum_{t=1}^L\eta_t\widehat\Psi_{>t}
      (\widehat G-G)\Psi_{<t}.
\]
Write
\(\widehat G-G=A_\lambda^{1/2}E_\lambda A_\lambda^{1/2}\), where
\(\|E_\lambda\|_{\mathrm{op}}\le\eta\). The regularized Loewner comparison
and scalar product smoothing give
\[
    \|B\widehat\Psi_{>t}A_\lambda^{1/2}\|_F
    \lesssim\frac{\|B\|_F}{\sqrt{1+v_t}}.
\]
Indeed, compare \(A_\lambda\) with
\(\widehat G+\lambda I\), which commutes with the empirical suffix, and use
\((x+\lambda)\widehat\psi_{>t}(x)^2\lesssim(1+v_t)^{-1}\).
Similarly, \(D\preceq CG\) and functional calculus for \(G\) give
\[
    \|A_\lambda^{1/2}\Psi_{<t}D^{1/2}\|_{\mathrm{op}}
    \lesssim\frac1{1+u_t};
\]
the scalar bound is
\(x(x+\lambda)\psi_{<t}(x)^2\lesssim(1+u_t)^{-2}\), using
\(u_t\le R=\lambda^{-1}\).
Termwise application of these bounds yields
\[
    \|B[\psi_L(\widehat G)-\psi_L(G)]\|_D
    \lesssim\eta\|B\|_F
    \sum_{t=1}^L
      \frac{\eta_t}{(1+u_t)\sqrt{1+v_t}}.
\]

For completeness, \(u_t+\eta_t+v_t=R\). Since \(\eta_t\le\eta_0\),
on indices with \(u_t\le R/2\), one has \(1+v_t\gtrsim1+R\), and discrete integration of
\(\eta_t/(1+u_t)\) gives \(C\log(e+R)\). On the remaining indices,
\(1+u_t\gtrsim1+R\), while reversing the order and discretely integrating
\(\eta_t/\sqrt{1+v_t}\) gives \(C\sqrt{1+R}\). Hence
\[
    \sum_{t=1}^L
      \frac{\eta_t}{(1+u_t)\sqrt{1+v_t}}
    \lesssim\frac{\log(e+R)}{\sqrt R},
\]
which proves the claim uniformly over the admissible deterministic schedules.
For the discrete integrations, each interval has length at most \(\eta_0\),
so a summand is bounded by a fixed multiple of the corresponding integral
over \([u_t,u_t+\eta_t]\) or \([v_t,v_t+\eta_t]\). These intervals partition
\([0,R]\), and the bounded-\(R\) case is absorbed into the constants.
\end{proof}

\begin{lemma}[Regularized finite-step response-filter transfer]
\label{lem:regularized-response-filter-transfer}
Under the assumptions of Lemma~\ref{lem:regularized-bias-filter-transfer},
for every matrix \(Q\), including a matrix depending on the same data as
\(\widehat G\),
\[
    \boxed{
    \|Q[g_L(\widehat G)-g_L(G)]\|_D
    \lesssim
    \eta\log(e+R)
    \|Q(G+R^{-1}I)^{-1/2}\|_F.}
\]
The inequality is pathwise and requires no independence between \(Q\) and
\(\widehat G\).
\end{lemma}

\begin{proof}
Let \(g_{t-1}\) be the empirical response prefix and \(\Psi_{>t}\) the
population suffix. The exact ordered recursion gives
\[
    g_L(\widehat G)-g_L(G)
    =-\sum_{t=2}^L\eta_tg_{t-1}(\widehat G)
       (\widehat G-G)\Psi_{>t}(G).
\]
For \(x\ge0\),
\[
    0\le g_{t-1}(x)\le\min\{u_t,x^{-1}\},
    \qquad u_t\le R,
\]
so \((x+R^{-1})g_{t-1}(x)\le2\). Comparing
\(A_\lambda\) with \(\widehat G+\lambda I\) therefore yields
\[
    \|A_\lambda^{1/2}g_{t-1}(\widehat G)
       A_\lambda^{1/2}\|_{\mathrm{op}}\lesssim1,
\]
and hence
\[
    \|Qg_{t-1}(\widehat G)A_\lambda^{1/2}\|_F
    \lesssim\|QA_\lambda^{-1/2}\|_F.
\]
The same population suffix estimate as in the preceding proof gives
\[
    \|A_\lambda^{1/2}\Psi_{>t}(G)D^{1/2}\|_{\mathrm{op}}
    \lesssim\frac1{1+v_t}.
\]
Insert the regularized factorization of \(\widehat G-G\) into the ordered
identity and apply the last two bounds. Since
\[
    \sum_{t=2}^L\frac{\eta_t}{1+v_t}
    \lesssim\log(e+R)
\]
by discrete reverse-order integration over intervals of length at most
\(\eta_0\), the claimed pathwise estimate follows.
\end{proof}

\begin{lemma}[Source-aware empirical bias transfer]
\label{lem:empirical-bias-transfer}
Suppose \(G,\widehat G\succ0\), \(C=C^\top\), and
\[
    B_*:=CG^{-1},
    \qquad
    \delta
    :=
    \left\|G^{-1/2}\widehat G G^{-1/2}-I\right\|_{\mathrm{op}}
    \le\frac12.
\]
Set \(\lambda:=R^{-1}\), let
\(\widehat\Psi_L:=\psi_L(\widehat G)\), and define
\[
    \widehat\Pi_{\mathrm{low}}
    :=
    \mathbf 1_{[0,c_0\lambda]}(\widehat G),
\]
where \(c_0>0\) is sufficiently small. Then
\[
\begin{aligned}
    \|B_*\widehat\Psi_L\|_G^2
    \ge{}&
    cR\operatorname{Tr}
    \left(\widehat\Pi_{\mathrm{low}}C^2\right)
    -
    C\delta^2\operatorname{Tr}
    \left(CG^{-1}C\right).
\end{aligned}
\]
On the joint Gaussian-sketch event of
Lemmas~\ref{lem:lin-sketch-transfer} and
\ref{lem:variance-residual-moments}, let
\(k\asymp\kappa_{R,P}\) and suppose \(k\le cM\). If, in addition,
\(\delta\le c_{\mathrm{rel}}R^{-1/2}\) with
\(c_{\mathrm{rel}}\) sufficiently small, then
\[
    \|B_*\widehat\Psi_L\|_G^2
    \gtrsim
    k^{1-\beta_\alpha}
    +
    P^{-1}k^{1-\beta_\theta}.
\]
\end{lemma}

\begin{proof}
The relative covariance comparison gives
\[
    (1-\delta)G
    \preceq
    \widehat G
    \preceq
    (1+\delta)G.
\]
Define
\[
    T:=\widehat G^{1/2}G^{-1}\widehat G^{1/2},
    \qquad
    Z:=\widehat G^{-1/2}C.
\]
The eigenvalues of \(T\) lie in
\([1-\delta,1+\delta]\), and hence
\[
    \|T-I\|_{\mathrm{op}}\le\delta.
\]
Because \(G\succeq(1+\delta)^{-1}\widehat G\) and
\(B_*^\top=G^{-1}C\),
\[
\begin{aligned}
    \|B_*\widehat\Psi_L\|_G^2
    &\gtrsim
    \|B_*\widehat\Psi_L\|_{\widehat G}^2\\
    &=
    \left\|
    \widehat\Psi_L\widehat G^{1/2}G^{-1}C
    \right\|_F^2\\
    &=
    \|\widehat\Psi_LTZ\|_F^2.
\end{aligned}
\]

Every GD factor is positive semidefinite on the covariance event. Moreover,
for \(0\le s\le c_0R^{-1}\),
\[
    \psi_L(s)
    \ge
    1-s\sum_{t=1}^L\eta_t
    \ge
    1-c_0
    \ge
    \frac12,
\]
after decreasing \(c_0\). Since
\(\widehat\Pi_{\mathrm{low}}\) commutes with
\(\widehat\Psi_L\),
\[
\begin{aligned}
    \|\widehat\Psi_LTZ\|_F^2
    &\gtrsim
    \|\widehat\Pi_{\mathrm{low}}TZ\|_F^2\\
    &\ge
    \frac12
    \|\widehat\Pi_{\mathrm{low}}Z\|_F^2
    -
    \|\widehat\Pi_{\mathrm{low}}(T-I)Z\|_F^2\\
    &\ge
    \frac12
    \|\widehat\Pi_{\mathrm{low}}Z\|_F^2
    -
    \delta^2\|Z\|_F^2.
\end{aligned}
\]
The two terms satisfy
\[
\begin{aligned}
    \|\widehat\Pi_{\mathrm{low}}Z\|_F^2
    &=
    \operatorname{Tr}
    \left(
    \widehat\Pi_{\mathrm{low}}\widehat G^{-1}C^2
    \right)\\
    &\ge
    \frac{R}{c_0}
    \operatorname{Tr}
    \left(\widehat\Pi_{\mathrm{low}}C^2\right),
\end{aligned}
\]
and
\[
\begin{aligned}
    \|Z\|_F^2
    &=
    \operatorname{Tr}\left(C\widehat G^{-1}C\right)\\
    &\le
    \frac{1}{1-\delta}
    \operatorname{Tr}\left(CG^{-1}C\right).
\end{aligned}
\]
This proves the first claim.

We now evaluate the two terms on the sketch event. Since
\(\widehat G\preceq(1+\delta)G\), the eigenvalue estimate for \(G\) and the
fixed-scale decay of \(h_{P,j}\) imply, for a sufficiently large absolute
constant \(C_0\),
\[
    \operatorname{rank}
    \left(\widehat\Pi_{\mathrm{low}}\right)
    \ge
    M-C_0k.
\]
Ky Fan's minimum principle and the eigenvalue estimate for \(C\) therefore
give
\[
\begin{aligned}
    \operatorname{Tr}
    \left(\widehat\Pi_{\mathrm{low}}C^2\right)
    &\ge
    \sum_{j>C_0k}\nu_j(C)^2\\
    &\gtrsim
    k\left(a_kh_{P,k}\right)^2.
\end{aligned}
\]

For the global source energy, set
\[
    U:=SH^{1/2},
    \qquad
    V:=SA_*H^{1/2},
    \qquad
    \Pi_U:=U^\top(UU^\top)^{-1}U.
\]
Then \(G=UU^\top\), \(C=VU^\top=UV^\top\), and
\[
    CG^{-1}C
    =
    V\Pi_UV^\top
    \preceq
    VV^\top.
\]
Consequently, the weighted trace estimate on the joint sketch event gives
\[
    \operatorname{Tr}\left(CG^{-1}C\right)
    \le
    \operatorname{Tr}
    \left(SA_*^2HS^\top\right)
    \lesssim1.
\]
Combining the preceding estimates and using
\(Rh_{P,k}\asymp1\),
\[
\begin{aligned}
    \|B_*\widehat\Psi_L\|_G^2
    &\gtrsim
    Rk\left(a_kh_{P,k}\right)^2
    -
    C\frac{c_{\mathrm{rel}}^2}{R}\\
    &\asymp
    ka_k^2h_{P,k}
    -
    Cc_{\mathrm{rel}}^2h_{P,k}.
\end{aligned}
\]
Finally, \(ka_k^2\asymp k^{1-2r}\gtrsim1\). Choosing
\(c_{\mathrm{rel}}\) sufficiently small absorbs the second term and yields
\[
    \|B_*\widehat\Psi_L\|_G^2
    \gtrsim
    k^{1-\beta_\alpha}
    +
    P^{-1}k^{1-\beta_\theta}.
\]
\end{proof}


\section{Additional Experimental Details}
\label{app:additional-experiments}

\paragraph{Primitive model and stationary evaluation.}
The ambient teacher is diagonal with
\[
    a_j=\rho j^{-r},
    \qquad \sigma_{\xi,j}=c_\xi j^{-\alpha},
    \qquad \sigma_{0,j}=c_0j^{-\theta},
\]
and generates nonstationary training trajectories from
\(x_{i,0}\sim\mathcal N(0,\Sigma_0)\). For each trajectory length \(P\),
the code forms
\[
    H_P=\frac1P\sum_{p=0}^{P-1}\Sigma_p,
    \quad G_P=SH_PS^\top,
    \quad C_P=SA_*H_PS^\top,
    \quad B_P^*=C_PG_P^{-1}.
\]
Stationary evaluation instead uses
\[
    \Sigma_\infty=\operatorname{diag}\!\left(
    \frac{\sigma_{\xi,j}}{1-a_j^2}\right),
    \quad G_\infty=S\Sigma_\infty S^\top,
    \quad B_\infty^*=SA_*\Sigma_\infty S^\top G_\infty^{-1}.
\]
For a fixed sketch, the stationary Approximation is
\[
    \mathrm{Approx}
    =\frac12\|(B_\infty^*S-SA_*)\Sigma_\infty^{1/2}\|_F^2,
\]
which is exactly independent of \(P\). The training/evaluation oracle
mismatch is
\[
    \mathrm{Shift}=\frac12\|B_P^*-B_\infty^*\|_{G_\infty}^2.
\]

\paragraph{Population bias, empirical variance, and WSD schedule.}
The deterministic population-filter bias experiment reports
\[
    \mathrm{Bias}^{\mathrm{pop}}
    =\frac12\|B_P^*\psi_L(G_P)\|_{G_\infty}^2,
    \qquad
    \psi_L(s)=\prod_{t=1}^L(1-\gamma_ts).
\]
This is a population-filter scaling proxy; it is not the empirical-filter
\(\mathrm{Bias}\) component. Empirical checkpoints sample whole dependent
trajectories and form
\[
    \widehat G=\frac1{NP}\sum_{i,p}u_{i,p}u_{i,p}^\top,
    \quad \widehat C=\frac1{NP}\sum_{i,p}y_{i,p}u_{i,p}^\top,
    \quad Q=\widehat C-B_P^*\widehat G.
\]
The empirical covariance safeguard from the main analysis is retained. The
code uses \(\gamma=0.1\) and the WSD schedule \citep{wen2025wsd} with approximately
\(L_w=\lfloor0.02L\rfloor\) warmup steps, a stable phase, and approximately
\(L_d=\lfloor0.02L\rfloor\) exponentially decaying steps; integer edge cases
are adjusted in the implementation. The final decay approaches \(\gamma/L\).
The experimental plots and checkpoints use the nominal horizon
\[
    R=\gamma L.
\]
For the saved checkpoints the safeguard never clips, so the nominal schedule
is the realized schedule. This experimental \(R\) is a constant-factor
proxy for the cumulative horizon \(\mathsf H_L=\sum_t\gamma_t\) used in the
theory; it is not generally equal to it.

The empirical components are evaluated in the stationary norm. In
particular,
\[
    \mathrm{Var}
    =\frac12\mathbb E_{\mathcal D}
    \left[\|Q\widehat g_L(\widehat G)\|_{G_\infty}^2\mid S\right],
\]
where \(\widehat g_L\) is the WSD response filter using the safeguarded
steps. Direct Monte Carlo uses complete trajectories, retaining their exact
within-trajectory dependence. The stationary finite-filter proxy is
\[
    d_{\mathrm{stat}}^{(L)}
    =\operatorname{Tr}\!\left[
    G_\infty G_Pg_L(G_P)^2\right],
    \qquad
    \mathrm{Var}_{\mathrm{guide}}\propto\frac{d_{\mathrm{stat}}^{(L)}}{NP}.
\]
It uses the finite population spectrum and WSD response filter with the
leading score-covariance scaling, but omits exact lag-dependent trajectory-
score covariance. Plotted normalized guides are one-point-normalized slope
references only, not independent tests of absolute prefactors.

The exact decomposition checked in Experiment 08 is
\[
\begin{aligned}
    \mathcal E_{\mathrm{direct}}
    &:=\frac12\|(B_LS-SA_*)\Sigma_\infty^{1/2}\|_F^2,\\
    \mathcal E_{\mathrm{decomp}}
    &:=\mathrm{Approx}+\mathrm{Shift}+\mathrm{Bias}+\mathrm{Var}
      +\mathrm{Cross}_{\mathrm{BV}}+\mathrm{Cross}_{\mathrm{S}}.
\end{aligned}
\]

\paragraph{Regime-I parameters and grids.}
Regime I uses
\[
    (\alpha,\theta,r,\rho,c_\xi,c_0,\gamma)
    =(1.5,2.5,0.2,0.55,0.05,0.01,0.1),
    \qquad \beta_\alpha=1.9.
\]
The Approximation sweep uses \(d=768\), \(P=4096\), eight sketches, and
\(M\in\{12,16,24,32,48,64,96,128\}\). The population-bias sweep uses
\((M,d,P)=(512,2048,4096)\), three sketches, and
\[
    L\in\{4000,6000,8000,12000,16000,24000\},
    \qquad R\in\{400,600,800,1200,1600,2400\}.
\]
The empirical-variance sweeps use \((M,d)=(20,128)\), three sketches, and
two trajectory repetitions per sketch:
\[
\begin{array}{ll}
N\in\{1024,1536,2048,3072,4096\},&P=4096,\ R=6400,\\
P\in\{1024,1536,2048,3072,4096\},&N=4096,\ R=6400,\\
R\in\{50,100,200,300,400,600,800,1600,3200,6400\},&N=P=4096.
\end{array}
\]

\paragraph{Regime-II parameters and grids.}
Regime II uses
\[
    (\alpha,\theta,r,\rho,c_\xi,c_0,\gamma)
    =(1.8,1.25,0.35,0.55,0.03,0.30,0.1),
    \quad \beta_\alpha=2.5,\quad \delta=0.55.
\]
The stationary Approximation experiment uses \((M,d)=(512,2048)\), eight
sketches reused across \(P\in\{4,6,8,12,16,24,32,48\}\). The population-
bias sweep uses \((M,d,R,L)=(1024,2048,1200,12000)\), five shared sketches,
and \(P\in\{4,6,8,12,16,24\}\). Its saved finite-\(\kappa\) reference is
\(\kappa_{R,P}^{1-\beta_\alpha}\); fitting that guide on the actual \(P\)
grid gives slope \(+0.927\), while the measured bias slope is \(+0.954\).

The empirical-variance sweeps use \((M,d)=(40,128)\), eight sketches, and
two trajectory repetitions. At fixed \(R=2400\), \(L=24000\), the \(N\)-
sweep uses
\(N\in\{16384,24576,32768,49152,65536,98304,131072\}\) at \(P=40\);
the \(P\)-sweep uses \(P\in\{8,12,16,24,32,40\}\) at \(N=65536\).
The optimization sweep fixes \((N,P)=(65536,40)\) and uses
\[
\begin{aligned}
    L&\in\{14000,20000,28000,38000,50000,\\
      &\qquad64000,80000\},\\
    R&\in\{1400,2000,2800,3800,5000,\\
      &\qquad6400,8000\}.
\end{aligned}
\]
Nested trajectory prefixes are used for the \(N\)- and \(P\)-sweeps; the
\(L\)-sweep reuses empirical moments across filters.

\paragraph{Fixed-token allocation and bootstrap.}
Experiment 07 fixes \((M,d,R)=(48,192,240)\), targets \(T=NP=262144\),
and uses \(P\in\{4,5,6,7,8,10\}\) with
\(N=\lfloor T/P\rfloor\). It uses eight independent Gaussian sketches and
four trajectory repetitions per sketch, with shared randomness across the
\(P\) grid. The measured stationary-variance slope is \(+0.427\); the
finite-filter guide slope is \(+0.347\), and the initialization-dominated
asymptotic exponent is \((\alpha-1)/\theta=+0.640\). A 1000-resample
sketch-cluster bootstrap keeps each sketch's full curve together and gives
percentile 95\% interval \([+0.335,+0.499]\). The effective number of
independent clusters is eight, so the interval is not a high-precision
asymptotic confidence interval. The approximation is invariant across \(P\);
Bias and total stationary excess risk increase over the tested grid. The
experiment holds \(M\) and \(R\) fixed and does not test the joint allocation
of \(M,N,P,R\) in Corollary~\ref{cor:compute-optimal-allocation}.

\paragraph{Exact total-risk check.}
Experiment 08 uses \((M,d,N,P)=(24,96,8192,12)\), three sketches, two
trajectory repetitions, and
\[
    R\in\{30,45,65,90,120,160\},
    \qquad L\in\{300,450,650,900,1200,1600\}.
\]
The maximum difference between direct stationary excess risk and the exact
six-component decomposition is \(3.21\times10^{-13}\). This verifies the
algebraic implementation identity, not the asymptotic matching total-risk
exponent.

\begin{figure}[!t]
    \centering
    \includegraphics[width=0.88\linewidth]{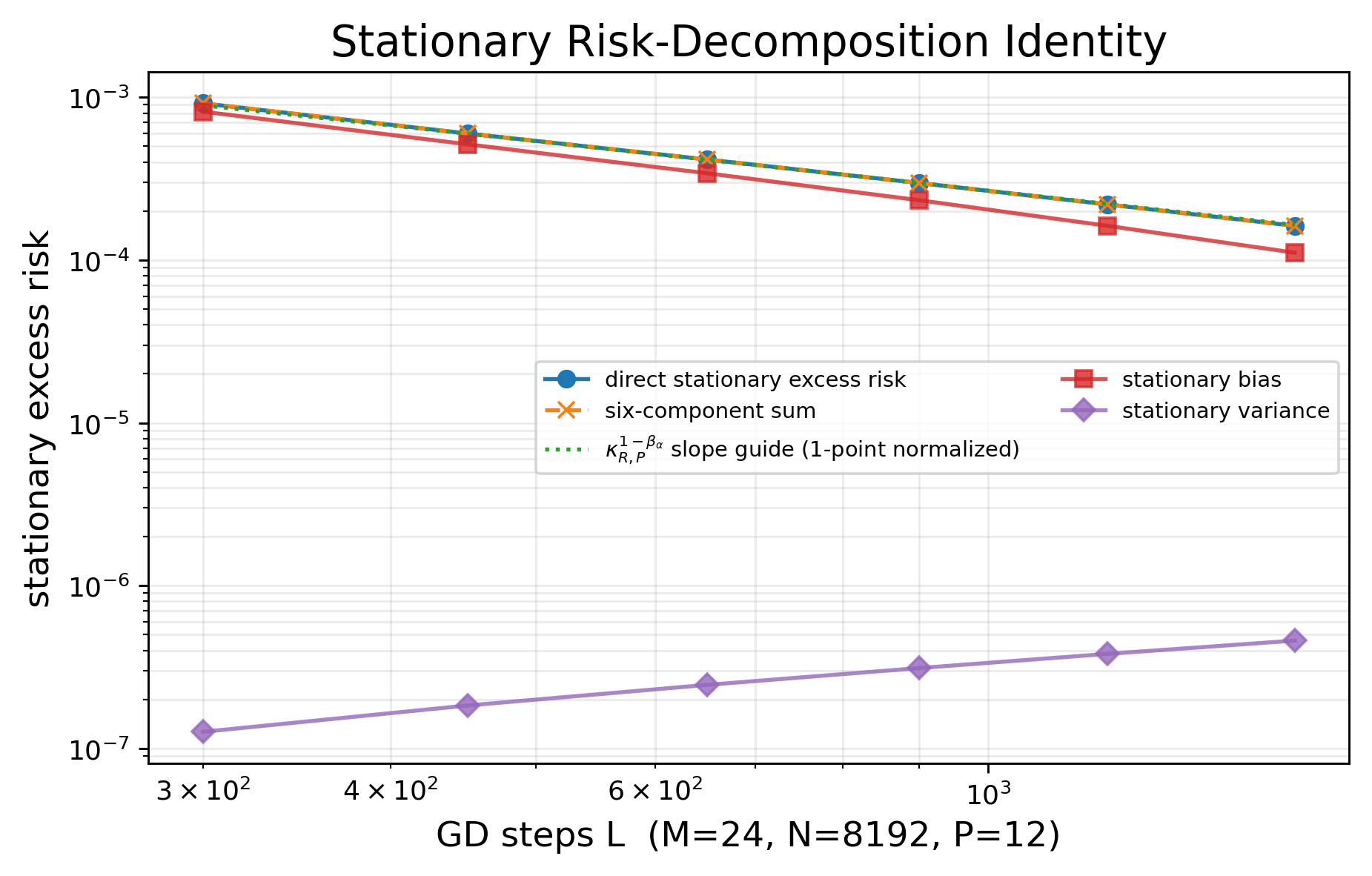}
    \caption{Experiment 08: the direct stationary excess risk and its
    six-component decomposition agree to numerical precision; the maximum
    absolute discrepancy is approximately \(3.21\times10^{-13}\). The plotted
    normalized \(\kappa_{R,P}^{1-\beta_\alpha}\) curve is only a slope guide;
    it does not independently verify an asymptotic total-risk exponent. The
    identity check validates the algebraic decomposition.}
    \label{fig:stationary-total-risk-check}
\end{figure}

\paragraph{Diagnostics and limitations.}
The checkpoints record \(\kappa_{R,P}/M\), \(NP/(RM)\), \(N/(RM)\), and
\[
    \left\|G_P^{-1/2}(\widehat G-G_P)G_P^{-1/2}\right\|_{\mathrm{op}},
\]
with strict-event threshold \(c_{\mathrm{rel}}R^{-1/2}\) and
\(c_{\mathrm{rel}}=0.08\). The strict-event pass rate is zero for the saved
Regime-II variance and fixed-token experiments; their safeguard clipping
rates are zero. This is a legacy full-spectrum diagnostic. The unified
theorem in both regimes instead uses
\(\|(G+R^{-1}I)^{-1/2}(\widehat G-G)(G+R^{-1}I)^{-1/2}\|_{\mathrm{op}}\),
which was not saved in those original checkpoints, so no claim about the new
regularized event is made from them. Some Regime-II variance checkpoints have \(\kappa_{R,P}/M\)
near one (up to about 0.99), so they approach the saturated range. These
experiments give finite-sample support for scaling mechanisms but do not
verify every sufficient high-probability condition. The Gaussian-sketch
arguments also assume \(d\ge C_{\mathrm{amb}}M\) for a sufficiently large
fixed geometric constant. Some configurations, including the Regime-II
population-Bias sweep with \((d,M)=(2048,1024)\), have only \(d/M=2\); they
need not meet that sufficient constant and should not be interpreted as
rigorous numerical verification of the full sketch theorem.

\paragraph{Empirical bias and regularized covariance validation.}
The separate validation experiment uses the same two primitive parameter
settings, \((M,d)=(24,96)\), six independent Gaussian sketches, and three
independent trajectory samples per sketch. Regime I fixes \((N,P)=(1024,64)\)
and uses \(R\in\{400,600,800,1200,1600,2400\}\). Regime II fixes
\((N,R)=(1024,1200)\) and uses \(P\in\{4,6,8,12,16,24\}\). For each
trajectory sample, we compute \(\widehat G\) from the full trajectories and
compare \(\frac12\|B_P^*\psi_L(\widehat G)\|_{G_\infty}^2\) with
\(\frac12\|B_P^*\psi_L(G_P)\|_{G_\infty}^2\), on the same sketch and with
the same realized WSD schedule. No safeguard is activated in these samples.
Figure~\ref{fig:bias-validation} reports means and percentile 95\% intervals
from 2000 bootstrap resamples of the six sketch-level means; the small number
of independent clusters limits interval precision. At \(P=4\) in Regime II,
the empirical mean is \(4.27\times10^{-7}\), compared with the population
mean \(3.64\times10^{-7}\); at \(P=24\), the corresponding means are
\(1.171\times10^{-5}\) and \(1.147\times10^{-5}\).

For Figure~\ref{fig:regularized-covariance-validation}, we fix \(R=1200\),
take \(P=64\) in Regime I and \(P=16\) in Regime II, and use nested trajectory
prefixes with \(N\in\{128,256,512,1024\}\). We record
\[
    \Delta_R=\left\|(G_P+R^{-1}I)^{-1/2}(\widehat G-G_P)
    (G_P+R^{-1}I)^{-1/2}\right\|_{\mathrm{op}}.
\]
The mean decreases from \(0.081\) to \(0.029\) in Regime I and from
\(0.436\) to \(0.137\) in Regime II. At \(N=128\), 17 of the 18 Regime-II
samples satisfy \(\Delta_R\le 0.5\); all 18 satisfy it at the other displayed
checkpoints in both regimes. The dashed \(0.5\) line is illustrative: even
passing this threshold does not establish the theorem's additional sample and
smallness conditions. These measurements do not retroactively diagnose the
larger original experiments, whose empirical covariances were not saved.

\begin{figure}[!t]
    \centering
    \includegraphics[width=\linewidth]{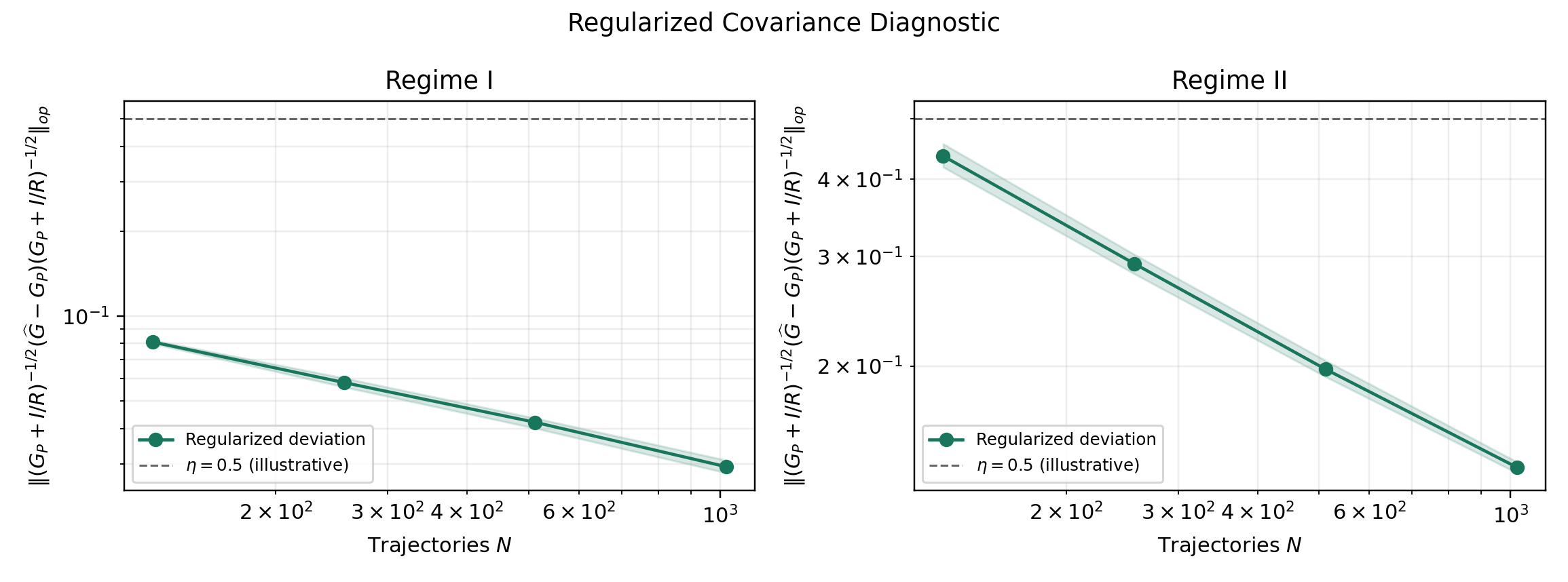}
    \caption{Regularized covariance deviation versus \(N\) at \(R=1200\),
    \((M,d)=(24,96)\), with \(P=64\) (Regime I) and \(P=16\) (Regime II).
    Shading is a 95\% percentile bootstrap interval over six sketch clusters,
    each with three trajectory repeats. The \(0.5\) threshold is illustrative
    and alone does not certify all theorem conditions.}
    \label{fig:regularized-covariance-validation}
\end{figure}

\subsection{Matched learning-rate schedules}
\label{app:general-lr-experiment}
We compare constant, WSD, and cosine learning rates on the same Gaussian
sketches and complete trajectory samples. For each cumulative learning-rate
mass \(\mathsf H\in\{100,200,400,800,1200\}\), the three deterministic
schedules have the same number of full-batch updates and the same nominal
mass. We use the primitive parameters of Section~\ref{sec:experiments},
\((M,d,N)=(24,96,1024)\), and \(P=64\) and \(P=8\) in Regimes I and II,
respectively. Each regime uses eight independent sketches and four independent
trajectory datasets per sketch; the same dataset is reused across schedules and
horizons. The safeguard is never activated, so nominal and realized masses
coincide.

Figure~\ref{fig:general-lr} reports stationary empirical Bias and Variance,
along with population-filter Bias evaluated using the same schedules. Relative
to constant steps, the schedule-dependent mean Bias and Variance differ by at
most \(0.014\%\) and \(0.006\%\), respectively, over these checkpoints.
This near-agreement is consistent with small-step filters at matched mass: the
largest realized \(\eta_t\|\widehat G\|_{\mathrm{op}}\) is below \(0.009\).
It does not establish equality of finite-step filters or validate arbitrary
admissible schedules. All observed regularized covariance deviations are below
\(0.226\); this does not verify the theorem's unspecified constants. The
primitive cutoff ratio \(k/M\) reaches approximately \(4.84\), and the
Regime-I bias curve differs visibly from the normalized scaling guide. These
finite-dimensional results therefore do not establish the asymptotic rates or
strict unsaturation.

\begin{figure}[t]
    \centering
    \includegraphics[width=\linewidth]{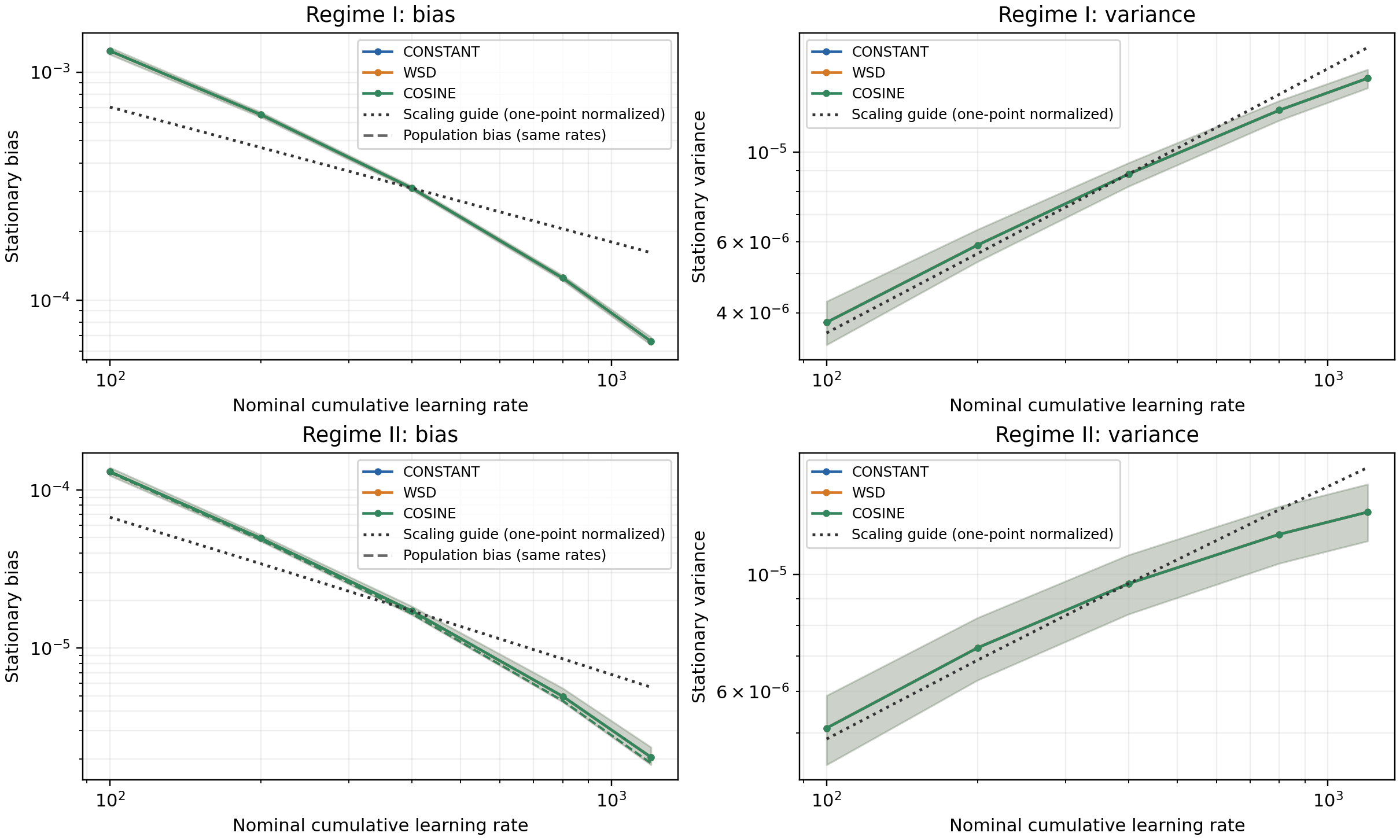}
    \caption{Matched learning-rate schedules in both regimes. Solid curves are
    empirical Bias or Variance; dashed curves are population-filter Bias.
    Shading gives percentile 95\% bootstrap intervals over eight sketch
    clusters. Dotted scaling guides are normalized at one WSD checkpoint and
    do not predict absolute constants.}
    \label{fig:general-lr}
\end{figure}

\paragraph{Slope summary.}
Table~\ref{tab:exp-slope-summary} reports the fitted log--log slopes. Regime-I
power-law entries use the stated asymptotic exponents; Regime-II Approximation
uses exact \(P\)-invariance, Bias uses the finite-\(\kappa\) guide fitted on
its saved grid, and Variance uses finite-filter guide slopes. The fixed-token
row uses the finite-filter guide; its bootstrap interval is shown in the table
caption. All plotted guides are one-point-normalized references.

\begin{table}[H]
    \centering
    \small
    \caption{Measured slopes and slope guides. The fixed-\(NP\) variance slope
    has sketch-cluster bootstrap 95\% interval \([+0.335,+0.499]\). Regime-II
    Bias uses its saved finite-\(\kappa\) guide (+0.927), and Regime-II
    Variance and fixed-token guides are finite-filter references.}
    \label{tab:exp-slope-summary}
    \begin{tabular}{llrr}
        \toprule
        Regime & Quantity & Measured slope & Guide slope \\
        \midrule
        I & Approximation vs. \(M\) & \(-0.962\) & \(-0.900\) \\
        I & Bias vs. \(R\) & \(-0.666\) & \(-0.600\) \\
        I & Variance vs. \(N\) & \(-0.972\) & \(-1.000\) \\
        I & Variance vs. \(P\) & \(-1.027\) & \(-1.000\) \\
        I & Variance vs. \(R\) & \(+0.482\) & \(+0.667\) \\
        \midrule
        II & Approximation vs. \(P\) & \(+0.000\) & \(+0.000\) \\
        II & Bias vs. \(P\) & \(+0.954\) & \(+0.927\) (finite \(\kappa\)) \\
        II & Variance vs. \(N\) & \(-0.970\) & \(-1.000\) \\
        II & Variance vs. \(P\) & \(-0.620\) & \(-0.614\) \\
        II & Variance vs. \(R\) & \(+0.413\) & \(+0.374\) \\
        II & Fixed-\(NP\) variance vs. \(P\) & \(+0.427\) & \(+0.347\) \\
        \bottomrule
    \end{tabular}
\end{table}

The checkpoints are saved under \texttt{exp/}: numbered CSVs retain the
per-experiment measurements and diagnostics, while
\texttt{exp/07\_regime2\_fixed\_token\_replicates.csv} and
\texttt{exp/07\_regime2\_fixed\_token\_bootstrap\_slopes.csv} retain the
fixed-token replicate and bootstrap data. The paired bias and regularized
covariance results are retained in \texttt{exp/09\_10\_validation\_samples.csv}
and \texttt{exp/09\_10\_validation\_summary.csv}.

\section{AI Use Statement.}
The authors used ChatGPT 5.6-Sol Thinking for understanding ideas during literature reviews, editing the paper and vibe-coding (to facilitate the empirical experiments). All technical claims, proofs, and final manuscript content were checked and edited by the authors.

\end{document}